\documentclass[twoside,11pt]{article}
\input epsf 

\usepackage{theapa}
\usepackage{jair}

\ShortHeadings{Conjunctive Query Answering for the DL \SHIQ}
{Glimm, Horrocks, Lutz, \& Sattler}
\firstpageno{157}
\jairheading{31}{2008}{157-204}{06/07}{01/08}

\usepackage{url}
\usepackage{amsthm,amsmath,amssymb}
\usepackage{graphicx}
\usepackage{color}
\usepackage{pifont}
\usepackage{xspace}
\usepackage{stmaryrd}
\usepackage{boxedminipage}
\usepackage{macros}

\begin{document} 
 
\title{Conjunctive Query Answering for the Description Logic 
$\SHIQ$}

\author{\name Birte Glimm \email birte.glimm@comlab.ox.ac.uk \\
       \name Ian Horrocks \email ian.horrocks@comlab.ox.ac.uk \\
       \addr Oxford University Computing Laboratory, UK
       \AND
       \name Carsten Lutz \email clu@tcs.inf.tu-dresden.de \\
       \addr Dresden University of Technology, Germany
       \AND
       \name Ulrike Sattler \email sattler@cs.man.ac.uk \\
       \addr The University of Manchester, UK}

\maketitle
       
\begin{abstract}
  Conjunctive queries play an important role as an expressive query 
  language for Description Logics (DLs). Although modern DLs usually 
  provide for transitive roles, conjunctive query answering over DL 
  knowledge bases is only poorly understood if transitive roles are 
  admitted in the query. In this paper, we consider unions of 
  conjunctive queries over knowledge bases formulated in the prominent 
  DL \SHIQ and allow transitive roles in both the query and the 
  knowledge base. We show decidability of query answering in this 
  setting and establish two tight complexity bounds: regarding combined 
  complexity, we prove that there is a deterministic algorithm for query 
  answering that needs time single exponential in the size of the KB and 
  double exponential in the size of the query, which is optimal. 
  Regarding data complexity, we prove containment in co-\NPclass.
\end{abstract}



\section{Introduction}\label{sect:intro}

Description Logics (DLs) are a family of logic based knowledge 
representation formalisms \cite{dlhb}. Most DLs are fragments of 
First-Order Logic restricted to unary and binary predicates, which are 
called concepts and roles in DLs. The constructors for building complex 
expressions are usually chosen such that the key inference problems, 
such as concept satisfiability, are decidable and preferably of low 
computational complexity. A DL knowledge base (KB) consists of a TBox, 
which contains intensional knowledge such as concept definitions and 
general background knowledge, and an ABox, which contains extensional 
knowledge and is used to describe individuals. Using a database 
metaphor, the TBox corresponds to the schema, and the ABox corresponds 
to the data. In contrast to databases, however, DL knowledge bases adopt 
an open world semantics, i.e., they represent information about the 
domain in an incomplete way. 

Standard DL reasoning services include testing concepts for 
satisfiability and retrieving certain instances of a given concept. The 
latter retrieves, for a knowledge base consisting of an ABox \AB and a 
TBox \TB, all (ABox) individuals that are instances of the given 
(possibly complex) concept expression $C$, i.e., all those individuals 
$a$ such that \TB and \AB entail that $a$ is an instance of $C$. The 
underlying reasoning problems are well-understood, and it is known that 
the combined complexity of these reasoning problems, i.e., the 
complexity measured in the size of the TBox, the ABox, and the query, is 
\ExpTime-complete for \SHIQ \cite{Tobi01a}. The data complexity of a 
reasoning problem is measured in the size of the ABox only. Whenever the 
TBox and the query are small compared to the ABox, as is often the 
case in practice, the data complexity gives a more useful performance 
estimate. For \SHIQ, instance retrieval is known to be data complete for 
co-\NPclass \cite{HuMS05a}. 

Despite the high worst case complexity of the standard reasoning 
problems for very expressive DLs such as \SHIQ, there are highly 
optimized implementations available, e.g., FaCT++ \cite{TsHo06a}, 
KAON2\footnote{\url{http://kaon2.semanticweb.org}}, Pellet 
\cite{SPGK06a}, and 
RacerPro\footnote{\url{http://www.racer-systems.com}}. These systems are 
used in a wide range of applications, e.g., configuration 
\cite{McWr98a}, bio informatics \cite{WBHL05a}, and information 
integration \cite{CDLN98a}. Most prominently, DLs are known for their 
use as a logical underpinning of ontology languages, e.g., OIL, 
DAML+OIL, and OWL \cite{HoPH03a}, which is a W3C recommendation 
\cite{BHHH04a}. 

In data-intensive applications, querying KBs plays a central role. 
Instance retrieval is, in some aspects, a rather weak form of querying: 
although possibly complex concept expressions are used as queries, we 
can only query for tree-like relational structures, i.e., a DL concept 
cannot express arbitrary cyclic structures. This property is known as 
the tree model property and is considered an important reason for the 
decidability of most Modal and Description Logics 
\cite{Grad01a,Vard97a}. Conjunctive queries (CQs) are well known in the 
database community and constitute an expressive query language with 
capabilities that go well beyond standard instance retrieval. For an 
example, consider a knowledge base that contains an ABox assertion 
$(\exists \dlf{hasSon}.(\exists \dlf{hasDaughter}.\top))(\dlf{Mary})$, 
which informally states that the individual (or constant in FOL terms) 
Mary has a son who has a daughter; hence, that Mary is a grandmother. 
Additionally, we assume that both roles $\dlf{hasSon}$ and 
$\dlf{hasDaughter}$ have a transitive super-role $\dlf{hasDescendant}$. 
This implies that $\dlf{Mary}$ is related via the role 
$\dlf{hasDescendant}$ to her (anonymous) grandchild. For this knowledge 
base, $\dlf{Mary}$ is clearly an answer to the conjunctive query 
$\dlf{hasSon}(x, y) \wedge \dlf{hasDaughter}(y, z) \wedge 
\dlf{hasDescendant}(x, z)$, when we assume that $x$ is a distinguished 
variable (also called answer or free variable) and $y, z$ are 
non-distinguished (existentially quantified) variables. 


If all variables in the query are non-distinguished, the query answer is 
just $\dlf{true}$ or $\dlf{false}$ and the query is called a Boolean 
query. Given a knowledge base \KB and a Boolean CQ \Q, the query 
entailment problem is deciding whether \Q is $\dlf{true}$ or 
$\dlf{false}$ w.r.t.\ \KB. If a CQ contains distinguished variables, the 
answers to the query are those tuples of individual names for which the 
knowledge base entails the query that is obtained by replacing the free 
variables with the individual names in the answer tuple. The problem of 
finding all answer tuples is known as query answering. Since query 
entailment is a decision problem and thus better suited for complexity 
analysis than query answering, we concentrate on query entailment. This 
is no restriction since query answering can easily be reduced to query 
entailment as we illustrate in more detail in 
Section~\ref{pre:sect:cqs}. 

Devising a decision procedure for conjunctive query entailment in 
expressive DLs such as SHIQ is a challenging problem, in particular when 
transitive roles are admitted in the query \cite{GlHS06a}. In the 
conference version of this paper, we presented the first decision 
procedure for conjunctive query entailment in \SHIQ. In this paper, we 
generalize this result to unions of conjunctive queries (UCQs) over 
\SHIQ knowledge bases. We achieve this by rewriting a conjunctive query 
into a set of conjunctive queries such that each resulting query is 
either tree-shaped (i.e., it can be expressed as a concept) or grounded 
(i.e., it contains only constants/individual names and no variables). 
The entailment of both types of queries can be reduced to standard 
reasoning problems \cite{HoTe00a,CaDL98a}. 

The paper is organized as follows: in Section~\ref{sect:preliminaries}, 
we give the necessary definitions, followed by a discussion of related 
work in Section~\ref{sect:rw}. In Section~\ref{sect:example}, we 
motivate the query rewriting steps by means of an example. In 
Section~\ref{sect:rewriting}, we give formal definitions for the 
rewriting procedure and show that a Boolean query is indeed entailed by 
a knowledge base \KB iff the disjunction of the rewritten queries is 
entailed by \KB. In Section~\ref{sect:decisionprocedure}, we present a 
deterministic algorithm for UCQ entailment in \SHIQ that runs in time 
single exponential in the size of the knowledge base and double 
exponential in the size of the query. Since the combined complexity of 
conjunctive query entailment is already 2\ExpTime-hard for the DL \ALCI 
\cite{Lutz07a}, it follows that this problem is 2\ExpTime-complete for 
\SHIQ. This shows that conjunctive query entailment for \SHIQ is 
strictly harder than instance checking, which is also the case for 
simpler DLs such as \EL \cite{Rosa07a}. We further show that (the 
decision problem corresponding to) conjunctive query answering in \SHIQ 
is co-\NPclass-complete regarding data complexity, and thus not harder 
than instance retrieval.

The presented decision procedure gives not only insight into query 
answering; it also has an immediate consequence on the field of 
extending DL knowledge bases with rules. From the work by 
\citeA[Thm.~11]{Rosa06a}, the consistency of a \SHIQ knowledge base 
extended with (weakly-safe) Datalog rules is decidable iff the 
entailment of unions of conjunctive queries in \SHIQ is decidable. 
Hence, we close this open problem as well. 

This paper is an extended version of the conference paper: Conjunctive 
Query Answering for the Description Logic \SHIQ. Proceedings of the 
Twentieth International Joint Conference on Artificial Intelligence 
(IJCAI'07), Jan 06 - 12, 2007.

\section{Preliminaries}\label{sect:preliminaries}

We introduce the basic terms and notations used throughout the paper. In 
particular, we introduce the DL \SHIQ \cite{HoST00a} and (unions of) 
conjunctive queries.

\subsection{Syntax and Semantics of \SHIQ}

Let \NC, \NR, and \NI be countably infinite sets of \emph{concept 
names}, \emph{role names}, and \emph{individual names}. We assume that 
the set of role names contains a subset $\NtR \subseteq \NR$ of 
\emph{transitive role names}. A \emph{role} is an element of $\NR \cup 
\{r^- \mid r \in \NR\}$, where roles of the form $r^-$ are called 
\emph{inverse roles}. A \emph{role inclusion} is of the form $r 
\sqsubseteq s$ with $r, s$ roles. A \emph{role hierarchy} \RB is a 
finite set of role inclusions.
  
An \emph{interpretation} \I = \inter consists of a non-empty set \dom, 
the \emph{domain} of \I, and a function \Int{$\cdot$}, which maps every 
concept name $A$ to a subset $\Int{A} \subseteq \dom$, every role name 
$r \in \NR$ to a binary relation $\Int{r} \subseteq \dom \times \dom$, 
every role name $r \in \NtR$ to a transitive binary relation $\Int{r} 
\subseteq \dom \times \dom$, and every individual name $a$ to an element 
$\Int{a} \in \dom$. An interpretation \I \emph{satisfies} a role 
inclusion $r \sqsubseteq s$ if $\Int{r} \subseteq \Int{s}$ and a role 
hierarchy \RB if it satisfies all role inclusions in \RB. 

We use the following standard notation:
\begin{enumerate} 
  \item 
    We define the function \mn{Inv} over roles as $\inv{r} := r^-$ if $r 
    \in \NR$ and $\inv{r} := s$ if $r = s^-$ for a role name~$s$.
  \item 
    For a role hierarchy \RB, we define \sssR as the reflexive 
    transitive closure of $\sqsubseteq$ over $\RB \cup \{\inv{r} 
    \sqsubseteq \inv{s} \mid r \sqsubseteq s \in \RB\}$. We use $r 
    \equiv_\RB s$ as an abbreviation for $r \sssR s$ and $s \sssR r$.  
  \item 
    For a role hierarchy \RB and a role $s$, we define the set $\transR$ 
    of transitive roles as $\{s \mid \mbox{there is a role } r \mbox{ 
    with } r \equiv_\RB s \mbox{ and } r \in \NtR \mbox{ or } \inv{r} 
    \in \NtR\}$. 
  \item 
    A role $r$ is called \emph{simple} w.r.t.\ a role hierarchy \RB if, 
    for each role $s$ such that \ $s \sssR r$, $s \notin \transR$. 
\end{enumerate}
The subscript \RB of \sssR and \transR is dropped if clear from the 
context. The set of $\SHIQ$\emph{-concepts} (or concepts for short) is 
the smallest set built inductively from \NC using the following grammar, 
where $A \in \NC$, $n \in \Nbbm$, $r$ is a role and $s$ is a simple 
role:
$$
  C ::= \top \mid \bot \mid A \mid \neg C \mid C_1 \sqcap C_2 \mid 
        C_1 \sqcup C_2 \mid \forall r.C \mid \exists r.C \mid 
        \leqslant n\; s.C \mid \geqslant n\; s.C.
$$

Given an interpretation \I, the semantics of $\SHIQ$-concepts is defined 
as follows:
$$
  {\setlength\arraycolsep{0.2em}
  \begin{array}{r l r l r l}
    \Int{\top} = & \dom\hspace*{.5cm}& 
    \Int{(C \sqcap D)} = & \Int{C} \cap \Int{D}\hspace*{.5cm}& 
    \Int{(\neg C)} = & \dom \setminus \Int{C}\hspace*{.5cm}\\
    \Int{\bot} = & \emptyset\hspace*{.5cm}& 
    \Int{(C \sqcup D)} = & \Int{C} \cup \Int{D}\hspace*{.5cm}& 
    \\ 
    \Int{(\forall r.C)} = & 
    \multicolumn{5}{l}{\{d \in \dom \mid \mbox{ if } (d, d') \in 
    \Int{r}, \mbox{ then } d' \in \Int{C}\}}\\
    \Int{(\exists r.C)} = & 
    \multicolumn{5}{l}{\{d \in \dom \mid \mbox{ there is a } (d, d') 
    \in \Int{r} \mbox{ with } d' \in \Int{C}\}}\\
    \Int{(\leqslant n\; s.C)} = & \multicolumn{5}{l}{\{d \in \dom 
      \mid \card{\Int{s}(d, C)} \leq n\}}\\
    \Int{(\geqslant n\; s.C)} = & \multicolumn{5}{l}{\{d \in \dom 
      \mid \card{\Int{s}(d, C)} \geq n\}}\\
  \end{array}
  }
$$
where $\card{M}$ denotes the cardinality of the set $M$ and $\Int{s}(d,
C)$ is defined as 
$$
  \{d' \in \dom \mid (d, d') \in \Int{s} \mbox{ and } d' \in \Int{C}\}.
$$

A \emph{general concept inclusion} (GCI) is an expression $C \sqsubseteq 
D$, where both $C$ and $D$ are concepts. A finite set of GCIs is called 
a \emph{TBox}. An interpretation \I \emph{satisfies} a GCI $C 
\sqsubseteq D$ if $\Int{C} \subseteq \Int{D}$, and a TBox \TB if it 
satisfies each GCI in \TB. 

An (ABox) \emph{assertion} is an expression of the form $C(a)$, $r(a, 
b)$, $\neg r(a,b)$, or $a \ndoteq b$, where $C$ is a concept, $r$ is a 
role, $a, b \in \NI$. An \emph{ABox} is a finite set of assertions. We 
use \indA to denote the set of individual names occurring in \AB. An 
interpretation \I \emph{satisfies} an assertion $C(a)$ if $\Int{a} \in 
\Int{C}$, $r(a, b)$ if $(\Int{a}, \Int{b}) \in \Int{r}$, $\neg r(a, b)$ 
if $(\Int{a}, \Int{b}) \notin \Int{r}$, and $a \ndoteq b$ if $\Int{a} 
\neq \Int{b}$. An interpretation \I \emph{satisfies} an ABox if it 
satisfies each assertion in \AB, which we denote with $\I \models \AB$.

A \emph{knowledge base} (KB) is a triple (\TB, \RB, \AB{}) with \TB a 
TBox, \RB a role hierarchy, and \AB an ABox. Let \KBDef be a KB and \I = 
\inter an interpretation. We say that \emph{\I satisfies $\KB$} if \I 
satisfies \TB, \RB, and \AB. In this case, we say that \I is a 
\emph{model} of \KB and write $\I \models \KB$.  We say that \emph{\KB 
is consistent} if \KB has a model.

\subsubsection{Extending \SHIQ to \SHIQR}

In the following section, we show how we can reduce a conjunctive query 
to a set of ground or tree-shaped conjunctive queries. During the 
reduction, we may introduce concepts that contain an intersection of 
roles under existential quantification. We define, therefore, the 
extension of \SHIQ with role conjunction/intersection, denoted as \SHIQR 
and, in the appendix, we show how to decide the consistency of \SHIQR 
knowledge bases. 

In addition to the constructors introduced for \SHIQ, \SHIQR allows for 
concepts of the form 
$$
  C ::= \forall R.C \mid \exists R.C \mid \leqslant n\; S.C \mid 
        \geqslant n\; S.C,
$$
where $R := r_1 \sqcap \ldots \sqcap r_n$, $S := s_1 \sqcap \ldots 
\sqcap s_n$, $r_1, \ldots, r_n$ are roles, and $s_1, \ldots, s_n$ are 
simple roles. The interpretation function is extended such that 
$\Int{(r_1 \sqcap \ldots \sqcap r_n)} = \Int{r_1} \cap \ldots \cap 
\Int{r_n}$.

\subsection{Conjunctive Queries and Unions of Conjunctive Queries}
\label{pre:sect:cqs}

We now introduce Boolean conjunctive queries since they are the basic 
form of queries we are concerned with. We later also define non-Boolean 
queries and show how they can be reduced to Boolean queries. Finally, 
unions of conjunctive queries are just a disjunction of conjunctive 
queries. 

For simplicity, we write a conjunctive query as a set instead of as a 
conjunction of atoms. For example, we write the introductory example 
from Section~\ref{sect:intro} as 
$$
\{\dlf{hasSon}(x, y), \dlf{hasDaughter}(y, z), \dlf{hasDescendant}(x, 
z)\}.
$$ 

For non-Boolean queries, i.e., when we consider the problem of query 
answering, the answer variables are often given in the head of the 
query, e.g., 
$$
(x_1, x_2, x_3) \leftarrow \{\dlf{hasSon}(x_1, x_2), 
\dlf{hasDaughter}(x_2, x_3), 
\dlf{hasDescendant}(x_1, x_3)\}
$$ 
indicates that the query answers are those tuples $(a_1, a_2, a_3)$ of 
individual names that, substituted for $x_1, x_2$, and $x_3$ 
respectively, result in a Boolean query that is entailed by the 
knowledge base. For simplicity and since we mainly focus on query 
entailment, we do not use a query head even in the case of a non-Boolean 
query. Instead, we explicitly say which variables are answer variables 
and which ones are existentially quantified. We now give a definition of 
Boolean conjunctive queries. 
 
\begin{definition}
  Let \NV be a countably infinite set of variables disjoint from \NC, 
  \NR, and \NI. A \emph{term} $t$ is an element from $\NV \cup \NI$. Let 
  $C$ be a concept, $r$ a role, and $t, t'$ terms. An \emph{atom} is an 
  expression $C(t)$, $r(t, t')$, or $t \approx t'$ and we refer to these 
  three different types of atoms as \emph{concept atoms, role atoms}, 
  and \emph{equality atoms} respectively. A \emph{Boolean conjunctive 
  query} \Q is a non-empty set of atoms. We use \vars{\Q} to denote 
  the set of (existentially quantified) variables occurring in \Q, 
  \inds{\Q} to denote the set of individual names occurring in \Q, and 
  \terms{\Q} for the set of terms in \Q, where $\terms{\Q} = \vars{\Q} 
  \cup \inds{\Q}$. If all terms in \Q are individual names, we say 
  that \Q is \emph{ground}. A \emph{sub-query} of \Q is simply a 
  subset of \Q (including \Q itself). As usual, we use $\card{\Q}$ to 
  denote the cardinality of \Q, which is simply the number of atoms in 
  \Q, and we use $|\Q|$ for the size of \Q, i.e., the number of 
  symbols necessary to write \Q. A \SHIQ conjunctive query is a 
  conjunctive query in which all concepts $C$ that occur in a concept 
  atom $C(t)$ are $\SHIQ$-concepts. 

  Since equality is reflexive, symmetric and transitive, we define 
  \sapprox as the transitive, reflexive, and symmetric closure of 
  $\approx$ over the terms in \Q. Hence, the relation \sapprox is an 
  equivalence relation over the terms in \Q and, for $t \in \terms{\Q}$, 
  we use $[t]$ to denote the equivalence class of $t$ by \sapprox. 
  
  Let \I = \inter be an interpretation. A total function $\pi \colon 
  \terms{\Q} \to  \dom$ is an \emph{evaluation} if (i) $\pi(a) = 
  \Int{a}$ for each individual name $a \in \inds{\Q}$ and (ii) $\pi(t) = 
  \pi(t')$ for all $t \sapprox t'$. We write 
  \begin{itemize}
    \item 
      $\I \models^\pi C(t)$ if $\pi(t) \in \Int{C}$;
    \item 
      $\I \models^\pi r(t, t')$ if $(\pi(t), \pi(t')) \in \Int{r}$;
    \item 
      $\I \models^\pi t \approx t'$ if $\pi(t) = \pi(t')$.
  \end{itemize}
  If, for an evaluation $\pi$, $\I \models^\pi at$ for all atoms $at \in 
  \Q$, we write $\I \models^\pi \Q$. We say that \I \emph{satisfies} \Q 
  and write $\I \models \Q$ if there exists an evaluation $\pi$ such 
  that $\I \models^\pi \Q$. We call such a $\pi$ a \emph{match} for \Q 
  in \I.
  
  Let \KB be a \SHIQ knowledge base and \Q a conjunctive query. If $\I 
  \models \KB$ implies $\I \models \Q$, we say that \KB \emph{entails} 
  \Q and write $\KB \models \Q$. 
\end{definition}

The \emph{query entailment problem} is defined as follows: given a 
knowledge base \KB and a query \Q, decide whether $\KB \models \Q$. 

For brevity and simplicity of notation, we define the relation \din over 
atoms in \Q as follows: $C(t) \din \Q$ if there is a term $t' \in 
\terms{\Q}$ such that $t \sapprox t'$ and $C(t') \in \Q$, and $r(t_1, 
t_2) \din \Q$ if there are terms $t_1', t_2' \in \terms{\Q}$ such that 
$t_1 \sapprox t_1'$, $t_2 \sapprox t_2'$, and $r(t_1', t_2') \in \Q$ or 
$\inv{r}(t_2', t_1') \in \Q$. This is clearly justified by definition of 
the semantics, in particular, because $\I \models r(t, t')$ implies that 
$\I \models \inv{r}(t', t)$.

When devising a decision procedure for CQ entailment, most complications 
arise from cyclic queries \cite{CaDL98a,CaRa97a}. In this context, when 
we say cyclic, we mean that the graph structure induced by the query is 
cyclic, i.e., the graph obtained from \Q such that each term is 
considered as a node and each role atom induces an edge. Since, in the 
presence of inverse roles, a query containing the role atom $r(t, t')$ 
is equivalent to the query obtained by replacing this atom with 
$\inv{r}(t', t)$, the direction of the edges is not important and we say 
that a query is cyclic if its underlying undirected graph structure is 
cyclic. Please note also that multiple role atoms for two terms are not 
considered as a cycle, e.g., the query $\{r(t, t'), s(t, t')\}$ is not a 
cyclic query. The following is a more formal definition of this 
property. 

\begin{definition}\label{def:cyclic}
  A query \Q is \emph{cyclic} if there exists a sequence of terms $t_1, 
  \dots, t_n$ with $n > 3$ such that 
  \begin{enumerate}
    \item 
      for each $i$ with $1 \leq i < n$, there exists a role atom 
      $r_i(t_i, t_{i+1}) \din \Q$,
    \item
       $t_1 = t_n$, and 
    \item\label{it:multiedge} 
      $t_i \neq t_j$ for $1 \leq i < j < n$.
  \end{enumerate}
  \mathdefend
\end{definition}

In the above definition, Item~\ref{it:multiedge} makes sure that we do 
not consider queries as cyclic just because they contain two terms $t, 
t'$ for which there are more than two role atoms using the two terms. 
Please note that we use the relation \din here, which implicitly uses 
the relation \sapprox and abstracts from the directedness of role atoms. 
 
In the following, if we write that we replace $r(t, t') \din \Q$ with 
$s(t_1, t_2), \ldots, s(t_{n-1}, t_n)$ for $t = t_1$ and $t' = t_n$, we 
mean that we first remove any occurrences of $r(\hat{t}, \hat{t'})$ and 
$\inv{r}(\hat{t'}, \hat{t}) $ such that $\hat{t} \sapprox t$ and 
$\hat{t'} \sapprox t'$ from \Q, and then add the atoms $s(t_1, t_2), 
\ldots, s(t_{n-1}, t_n)$ to \Q. 

W.l.o.g., we assume that queries are connected. More precisely, let \Q 
be a conjunctive query. We say that \Q is \emph{connected} if, for all 
$t, t' \in \terms{\Q}$, there exists a sequence $t_1, \dots, t_n$ such 
that $t_1 = t$, $t_n = t'$ and, for all $1 \leq i < n$, there exists a 
role $r$ such that $r(t_i, t_{i+1}) \din \Q$. A collection $\Q_1, 
\dots, \Q_n$ of queries is a \emph{partitioning} of \Q if $\Q = \Q_1 
\cup \ldots \cup \Q_n$, $\Q_i \cap \Q_j = \emptyset$ for $1 \leq i 
< j \leq n$, and each $\Q_i$ is connected. 

\begin{lemma}
  Let \KB be a knowledge base, \Q a conjunctive query, and $\Q_1, \dots, 
  \Q_n$ a partitioning of $\Q$. Then $\KB \models \Q$ iff $\KB \models 
  \Q_i$ for each $i$ with $1 \leq i \leq n$.
\end{lemma}

A proof is given by \citeA[7.3.2]{Tess01a} and, with this lemma, it is 
clear that the restriction to connected queries is indeed w.l.o.g.\ 
since entailment of \Q can be decided by checking entailment of each 
$\Q_i$ at a time. In what follows, we therefore assume queries to be 
connected without further notice.

\begin{definition}
  A \emph{union of Boolean conjunctive queries} is a formula $\Q_1 \vee 
  \ldots \vee \Q_n$, where each disjunct $\Q_i$ is a Boolean conjunctive 
  query. 
  
  A knowledge base \KB \emph{entails} a union of Boolean conjunctive 
  queries $\Q_1 \vee \ldots \vee \Q_n$, written as $\KB \models \Q_1 
  \vee \ldots \vee \Q_n$, if, for each interpretation \I such that $\I 
  \models \KB$, there is some $i$ such that $\I \models \Q_i$ and $1 
  \leq i \leq n$. 
\end{definition}

W.l.o.g.\ we assume that the variable names in each disjunct are 
different from the variable names in the other disjuncts. This can 
always be achieved by naming variables apart. We further assume that 
each disjunct is a connected conjunctive query. This is w.l.o.g. since a 
UCQ which contains unconnected disjuncts can always be transformed into 
conjunctive normal form; we can then decide entailment for each 
resulting conjunct separately and each conjunct is a union of connected 
conjunctive queries. We describe this transformation now in more detail 
and, for a more convenient notation, we write a conjunctive query 
$\{at_1, \ldots, at_k\}$ as $at_1 \wedge \ldots \wedge at_k$ in the 
following proof, instead of the usual set notation. 

\begin{lemma}\label{lem:cnf}
  Let \KB be a knowledge base, $\Q = \Q_1 \vee \ldots \vee \Q_n$ a union 
  of conjunctive queries such that, for $1 \leq i \leq n, \Q_i^1, 
  \ldots, \Q_i^{k_i}$ is a partitioning of the conjunctive query $\Q_i$. 
  Then $\KB \models \Q$ iff 
  $$
  \KB \models \bigwedge\limits_{(i_1, \ldots, i_n) \in \{1, \ldots, 
  k_1\} \times \ldots \times \{1, \ldots, k_n\}} (\Q_1^{i_1} \vee \ldots 
  \vee \Q_n^{i_n}). 
  $$
\end{lemma}

Again, a detailed proof is given by \citeA[7.3.3]{Tess01a}. Please note 
that, due to the transformation into conjunctive normal form, the 
resulting number of unions of connected conjunctive queries for which we 
have to test entailment can be exponential in the size of the original 
query. When analysing the complexity of the decision procedures 
presented in Section~\ref{sect:decisionprocedure}, we show that the 
assumption that each CQ in a UCQ is connected does not increase the 
complexity. 

We now make the connection between query entailment and query answering 
clearer. For query answering, let the variables of a conjunctive query 
be typed: each variable can either be existentially quantified (also 
called \emph{non-distinguished}) or free (also called 
\emph{distinguished} or \emph{answer variables}). Let \Q be a query in 
$n$ variables (i.e., $\card{\vars{\Q}} = n$), of which $v_1, \dots, v_m$ 
($m \leq n$) are answer variables. The \emph{answers} of \KBDef to 
\Q are those $m$-tuples $(a_1, \dots, a_m) \in \indA^m$ such that, for 
all models \I of \KB, $\I \models^\pi \Q$ for some $\pi$ that satisfies 
$\pi(v_i) = \Int{a_i}$ for all $i$ with $1 \leq i \leq m$. It 
is not hard to see that the answers of \KB to \Q can be computed by 
testing, for each $(a_1, \ldots, a_m) \in \indA^m$, whether the query 
$\Q_{[v_1, \dots, v_m/a_1, \ldots, a_m]}$ obtained from \Q by replacing 
each occurrence of $v_i$ with $a_i$ for $1 \leq i \leq m$ is 
entailed by \KB. The answer to \Q is then the set of all $m$-tuples 
$(a_1, \ldots, a_m)$ for which $\KB \models \Q_{[v_1, \dots, v_m/a_1, 
\ldots, a_m]}$. Let $k = \card{\indA}$ be the number of individual names 
used in the ABox \AB. Since \AB is finite, clearly $k$ is finite. Hence, 
deciding which tuples belong to the set of answers can be checked with 
at most $k^m$ entailment tests. This is clearly not very efficient, but 
optimizations can be used, e.g., to identify a (hopefully small) set of 
candidate tuples. 

The algorithm that we present in Section~\ref{sect:decisionprocedure} 
decides query entailment. The reasons for devising a decision procedure 
for query entailment instead of query answering are two-fold: first, 
query answering can be reduced to query entailment as shown above; 
second, in contrast to query answering, query entailment is a decision 
problem and can be studied in terms of complexity theory. 

In the remainder of this paper, if not stated otherwise, we use \Q 
(possibly with subscripts) for a connected Boolean conjunctive query, 
\KB for a \SHIQ knowledge base (\TB, \RB, \AB{}), \I for an 
interpretation \inter, and $\pi$ for an evaluation.

\section{Related Work}
\label{sect:rw}

Very recently, an automata-based decision procedure for positive 
existential path queries over \ALCQIbreg knowledge bases has been 
presented \cite{CaEO07a}. Positive existential path queries generalize 
unions of conjunctive queries and since a \SHIQ knowledge base can be 
polynomially reduced to an \ALCQIbreg knowledge base, the presented 
algorithm is a decision procedure for (union of) conjunctive query 
entailment in \SHIQ as well. The automata-based technique can be 
considered more elegant than our rewriting algorithm, but it does not 
give an \NPclass upper bound for the data complexity as our technique. 


Most existing algorithms for conjunctive query answering in expressive 
DLs assume, however, that role atoms in conjunctive queries use only 
roles that are not transitive. As a consequence, the example query from 
the introductory section cannot be answered. Under this restriction, 
decision procedures for various DLs around \SHIQ are known 
\cite{HoTe00a,OrCE06a}, and it is known that answering conjunctive 
queries in this setting is data complete for 
\hbox{co-\NPclass~\cite{OrCE06a}.} Another common restriction is that 
only individuals named in the ABox are considered for the assignments of 
variables. In this setting, the semantics of queries is no longer the 
standard First-Order one. With this restriction, the answer to the 
example query from the introduction would be $\dlf{false}$ since 
$\dlf{Mary}$ is the only named individual. It is not hard to see that 
conjunctive query answering with this restriction can be reduced to 
standard instance retrieval by replacing the variables with individual 
names from the ABox and then testing the entailment of each conjunct 
separately. Most of the implemented DL reasoners, e.g., KAON2, Pellet, 
and RacerPro, provide an interface for conjunctive query answering in 
this setting and employ several optimizations to improve the performance 
\cite{SiPa06a,MoSS04a,WeMo05a}. Pellet appears to be the only reasoner 
that also supports the standard First-Order semantics for \SHIQ 
conjunctive queries under the restriction that the queries are acyclic. 

To the best of our knowledge, it is still an open problem whether 
conjunctive query entailment is decidable in \SHOIQ. Regarding 
undecidability results, it is known that conjunctive query entailment 
in the two variable fragment of First-Order Logic $\Lmc_2$ is 
undecidable \cite{Rosa07b} and Rosati identifies a relatively small 
set of constructors that causes the undecidability. 

Query entailment and answering have also been studied in the context of 
databases with incomplete information \cite{Rosa06b,Meyd98a,Grah91a}. In 
this setting, DLs can be used as schema languages, but the expressivity 
of the considered DLs is much lower than the expressivity of \SHIQ. For 
example, the constructors provided by logics of the DL-Lite family 
\cite{CDLL07a} are chosen such that the standard reasoning tasks are in 
\PTime and query entailment is in \LogSpace with respect to data 
complexity. Furthermore, TBox reasoning can be done independently of the 
ABox and the ABox can be stored and accessed using a standard database 
SQL engine. Since the considered DLs are considerable less expressive 
than \SHIQ, the techniques used in databases with incomplete information 
cannot be applied in our setting.


Regarding the query language, it is well known that an extension of 
conjunctive queries with inequalities is undecidable \cite{CaDL98a}. 
Recently, it has further been shown that even for DLs with 
low expressivity, an extension of conjunctive queries with inequalities 
or safe role negation leads to undecidability \cite{Rosa07b}. 

A related reasoning problem is \emph{query containment}. Given a schema 
(or TBox) \Smc and two queries $q$ and $q'$, we have that $q$ is 
contained in $q'$ w.r.t.\ \Smc iff every interpretation \I that 
satisfies \Smc and $q$ also satisfies $q'$. It is well known that query 
containment w.r.t.\ a TBox can be reduced to deciding query entailment 
for (unions of) conjunctive queries w.r.t.\ a knowledge base 
\cite{CaDL98a}. Hence a decision procedure for (unions of) conjunctive 
queries in \SHIQ can also be used for deciding query containment w.r.t.\ 
to a \SHIQ TBox.

Entailment of unions of conjunctive queries is also closely related to 
the problem of adding rules to a DL knowledge base, e.g., in the form of 
Datalog rules. Augmenting a DL KB with an arbitrary Datalog program 
easily leads to undecidability \cite{LeRo98a}. In order to ensure 
decidability, the interaction between the Datalog rules and the DL 
knowledge base is usually restricted by imposing a safeness condition. 
The \DLlog framework \cite{Rosa06a} provides the least restrictive 
integration proposed so far. Rosati presents an algorithm that decides 
the consistency of a \DLlog knowledge base by reducing the problem to 
entailment of unions of conjunctive queries, and he proves that 
decidability of UCQs in \SHIQ implies the decidability of consistency 
for $\SHIQ$+$log$ knowledge bases.


\section{Query Rewriting by Example}\label{sect:example}

In this section, we motivate the ideas behind our query rewriting 
technique by means of examples. In the following section, we give 
precise definitions for all rewriting steps.

\subsection{Forest Bases and Canonical Interpretations}
\label{sect:forestbase}

The main idea is that we can focus on models of the knowledge base that 
have a kind of tree or forest shape. 
%
%
It is well known that one reason for Description and Modal Logics being 
so robustly decidable is that they enjoy some form of tree model 
property, i.e., every satisfiable concept has a model that is 
tree-shaped \cite{Vard97a,Grad01a}. When going from concept 
satisfiability to knowledge base consistency, we need to replace the 
tree model property with a form of forest model property, i.e., every 
consistent KB has a model that consists of a set of ``trees'', where 
each root corresponds to a named individual in the ABox. The roots can 
be connected via arbitrary relational structures, induced by the role 
assertions given in the ABox. A forest model is, therefore, not a forest 
in the graph theoretic sense. Furthermore, transitive roles can 
introduce ``short-cut'' edges between elements within a tree or even 
between elements of different trees. Hence we talk of ``a form of'' 
forest model property.

We now define forest models and show that, for deciding query 
entailment, we can restrict our attention to forest models. The 
rewriting steps are then used to transform cyclic subparts of the query 
into tree-shaped ones such that there is a ``forest-shaped match'' for 
the rewritten query into the forest models. 

In order to make the forest model property even clearer, we also 
introduce \emph{forest bases}, which are interpretations that interpret 
transitive roles in an unrestricted way, i.e., not necessarily in a 
transitive way. For a forest base, we require in particular that 
\emph{all} relationships between elements of the domain that can be 
inferred by transitively closing a role are omitted. In the following, 
we assume that the ABox contains at least one individual name, i.e., 
\indA is non-empty. This is w.l.o.g.\ since we can always add an 
assertion $\top(a)$ to the ABox for a fresh individual name $a \in \NI$. 
For readers familiar with tableau algorithms, it is worth noting that 
forest bases can also be thought of as those tableaux generated from a 
complete and clash-free completion tree \cite{HoST00a}. 

\begin{definition}\label{def:tree}
  Let $\Nbbm$ denote the non-negative integers and $\Nbbm^*$ the set of 
  all (finite) words over the alphabet \Nbbm. A \emph{tree} $T$ is a 
  non-empty, prefix-closed subset of $\Nbbm^*$. For $w, w' \in T$, we 
  call $w'$ a \emph{successor} of $w$ if $w' = w \cdot c$ for some $c 
  \in \Nbbm$, where ``$\cdot$'' denotes concatenation. We call $w'$ a 
  \emph{neighbor} of $w$ if $w'$ is a successor of $w$ or vice versa. 
  The empty word $\varepsilon$ is called the \emph{root}.
  
  A \emph{forest base for} \KB is an interpretation \Jmc = \interJ that 
  interprets transitive roles in an unrestricted (i.e., not necessarily 
  transitive) way and, additionally, satisfies the following conditions:
  \begin{enumerate}
    \renewcommand{\theenumi}{\mn{T\arabic{enumi}}}
    \renewcommand\labelenumi{\theenumi}
    \item\label{it:tree} 
      $\domJ \subseteq \indA \times \Nbbm^*$ such that, for all $a 
      \in \indA$, the set $\{w \mid  (a, w) \in \domJ\}$ is a tree;
    \item\label{it:minimal} 
      if $((a, w), (a', w')) \in \IntJ{r}$, then either $w = w' = 
      \varepsilon$ or $a = a'$ and $w'$ is a neighbor of $w$;
    \item\label{it:roots} 
      for each $a \in \indA$, $\IntJ{a} = (a, \varepsilon)$; 
  \end{enumerate}
  An interpretation \I is \emph{canonical for} \KB if there exists a 
  forest base \Jmc for \KB such that \I is identical to \Jmc except 
  that, for all non-simple roles $r$, we have
  $$
     \Int{r} = \IntJ{r} \cup \bigcup_{s \ssssR r, \; s \in \transR} 
     (\IntJ{s})^+
  $$
  In this case, we say that \Jmc is a forest base \emph{for} \I and if 
  $\I \models \KB$ we say that \I is a \emph{canonical model for} \KB.
\end{definition}

For convenience, we extend the notion of successors and neighbors to 
elements in canonical models. Let \I be a canonical model with $(a, w), 
(a', w') \in \dom$. We call $(a', w')$ a \emph{successor} of $(a, w)$ if 
either $a = a'$ and $w' = w \cdot c$ for some $c \in \Nbbm$ or $w = w' = 
\varepsilon$. We call $(a', w')$ a \emph{neighbor} of $(a, w)$ if $(a', 
w')$ is a successor of $(a, w)$ or vice versa. 

Please note that the above definition implicitly relies on the unique 
name assumption (UNA) (cf.\ \ref{it:roots}). This is w.l.o.g.\ as we can 
guess an appropriate partition among the individual names and replace 
the individual names in each partition with one representative 
individual name from that partition. In 
Section~\ref{sect:decisionprocedure}, we show how the partitioning of 
individual names can be used to simulate the UNA, hence, our decision 
procedure does not rely on the UNA. We also show that this does not 
affect the complexity.

\begin{lemma}\label{lem:canonicalcountermodels}
  Let \KB be a \SHIQ knowledge base and $\Q = \Q_1 \vee \ldots \vee 
  \Q_n$ a union of conjunctive queries. Then $\KB \not\models \Q$ iff 
  there exists a canonical model \I of \KB such that $\I \not\models 
  \Q$.
\end{lemma}

A detailed proof is given in the appendix. Informally, for the only if 
direction, we can take an arbitrary counter-model for the query, which 
exists by assumption, and ``unravel'' all non-tree structures. Since, 
during the unraveling process, we only replace cycles in the model by 
infinite paths and leave the interpretation of concepts unchanged, the 
query is still not satisfied in the unravelled canonical model. The if 
direction of the proof is trivial.

\subsection{The Running Example}
\label{sect:rewritingexample}

We use the following Boolean query and knowledge base as a running 
example:
 
\begin{example}\label{ex:running}
  Let $\KB = (\TB, \RB, \AB)$ be a \SHIQ knowledge base with $r, t \in 
  \NtR, k \in \Nbbm$
  $$
  \begin{array}{l l l}
    \TB = & \{ & C_k \sqsubseteq \,\,\geqslant k\; p.\top,\\
          &    & C_3 \sqsubseteq \,\,\geqslant 3\; p.\top,\\
          &    & D_2 \sqsubseteq \exists s^-.\top \sqcap \exists 
                 t.\top \\
          & \}\\
    \RB = & \{ & t \sqsubseteq t^-, \\
          &    & s^- \sqsubseteq r\\
          & \}\\
    \AB = & \{ & r(a, b), \\
          &    & (\exists p.C_k \sqcap \exists p.C \sqcap \exists 
                 r^-.C_3)(a), \\
          &    & (\exists p.D_1 \sqcap \exists r.D_2)(b)\\
          & \}
  \end{array}
  $$ 
  and $\Q = \{r(u, x), r(x, y), t(y, y), s(z, y), r(u, z)\}$ with 
  $\inds{\Q} = \emptyset$ and $\vars{\Q} = \{u, x, y, z\}$.
\end{example}

For simplicity, we choose to use a CQ instead of a UCQ. In case of a 
UCQ, the rewriting steps are applied to each disjunct separately. 

\begin{figure}[htb]
  \centering
  \input{canonical.pstex_t}
  \caption{A representation of a canonical interpretation \I for \KB.}
  \label{fig:canonical}
\end{figure}

Figure~\ref{fig:canonical} shows a representation of a canonical model 
\I for the knowledge base \KB from Example~\ref{ex:running}. Each 
labeled node represents an element in the domain, e.g., the individual 
name $a$ is represented by the node labeled $(a, \varepsilon)$. The 
edges represent relationships between individuals. For example, we can 
read the $r$-labeled edge from $(a, \varepsilon)$ to $(b, \varepsilon)$ 
in both directions, i.e., $(\Int{a}, \Int{b}) = ((a, \varepsilon), (b, 
\varepsilon)) \in \Int{r}$ and $(\Int{b}, \Int{a}) = ((b, \varepsilon), 
(a, \varepsilon)) \in \Int{r^-}$. The ``short-cuts'' due to transitive 
roles are shown as dashed lines, while the relationship between the 
nodes that represent ABox individuals is shown in grey. Please note that 
we did not indicate the interpretations of all concepts in the figure. 

Since \I is a canonical model for \KB, the elements of the domain are 
pairs $(a, w)$, where $a$ indicates the individual name that corresponds 
to the root of the tree, i.e., $\Int{a} = (a, \varepsilon)$ and the 
elements in the second place form a tree according to our definition of 
trees. For each individual name $a$ in our ABox, we can, therefore, 
easily define the tree rooted in $a$ as $\{w \mid (a, w) \in \dom\}$. 

\begin{figure}[htb]
  \centering
  \input{forestbase.pstex_t}
  \caption{A forest base for the interpretation represented by 
  Figure~\ref{fig:canonical}.}
  \label{fig:forestbase}
\end{figure}

Figure~\ref{fig:forestbase} shows a representation of a forest base for 
the interpretation from Figure~\ref{fig:canonical} above. For 
simplicity, the interpretation of concepts is no longer shown. The two 
trees, rooted in $(a, \varepsilon)$ and $(b, \varepsilon)$ respectively, 
are now clear. 

\begin{figure}[htb]
  \centering
  \input{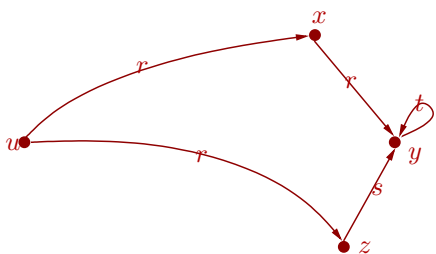}
  \caption{A graph representation of the query from 
  Example~\ref{ex:running}.}
  \label{fig:query}
\end{figure}

A graphical representation of the query \Q from Example~\ref{ex:running} 
is shown in Figure~\ref{fig:query}, where the meaning of the nodes and 
edges is analogous to the ones given for interpretations. We call this 
query a cyclic query since its underlying undirected graph is cyclic 
(cf.\ Definition~\ref{def:cyclic}). 

\begin{figure}[htb]
  \centering
  \input{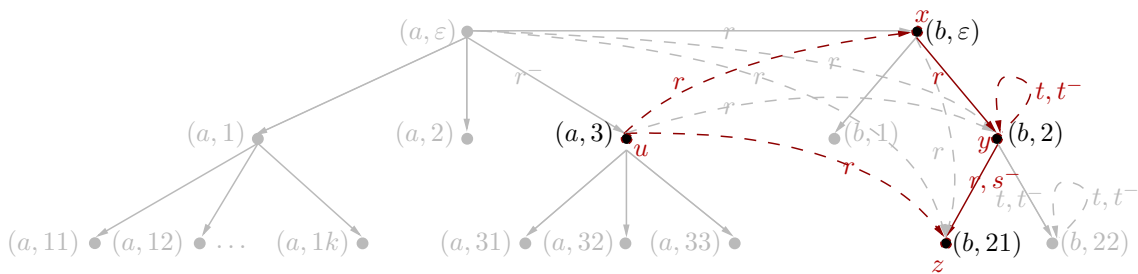}
  \caption{A match $\pi$ for the query \Q from Example~\ref{ex:running} 
  onto the model \I from Figure~\ref{fig:canonical}.}
  \label{fig:querymatch}
\end{figure}

Figure~\ref{fig:querymatch} shows a match $\pi$ for \Q and \I and, 
although we consider only one canonical model here, it is not hard to 
see that the query is true in each model of the knowledge base, i.e., 
$\KB \models \Q$. 

The forest model property is also exploited in the query rewriting 
process. We want to rewrite \Q into a set of queries $\Q_1, \ldots, 
\Q_n$ of ground or tree-shaped queries such that $\KB \models \Q$ iff 
$\KB \models \Q_1 \vee \ldots \vee \Q_n$. Since the resulting queries 
are ground or tree-shaped queries, we can explore the known techniques 
for deciding entailment of these queries. As a first step, we transform 
\Q into a set of forest-shaped queries. Intuitively, forest-shaped 
queries consist of a set of tree-shaped sub-queries, where the roots of 
these trees might be arbitrarily interconnected (by atoms of the form 
$r(t, t')$). A tree-shaped query is a special case of a forest-shaped 
query. We will call the arbitrarily interconnected terms of a 
forest-shaped query the root choice (or, for short, just roots). At the 
end of the rewriting process, we replace the roots with individual names 
from \indA and transform the tree parts into a concept by applying the 
so called rolling-up or tuple graph technique \cite{Tess01a,CaDL98a}. 

In the proof of the correctness of our procedure, we use the structure 
of the forest bases in order to explicate the transitive ``short-cuts'' 
used in the query match. By explicating we mean that we replace each 
role atom that is mapped to such a short-cut with a sequence of role 
atoms such that an extended match for the modified query uses only paths 
that are in the forest base.

\subsection{The Rewriting Steps}

The rewriting process for a query \Q is a six stage process. At the end 
of this process, the rewritten query may or may not be in a forest 
shape. As we show later, this ``don't know'' non-determinism does not 
compromise the correctness of the algorithm. In the first stage, we 
derive a \emph{collapsing} $\Q_{co}$ of \Q by adding (possibly several) 
equality atoms to \Q. Consider, for example, the cyclic query $\Q = 
\{r(x, y), r(x, y'), s(y, z), s(y', z)\}$ (see 
Figure~\ref{fig:collapsing}), which can be transformed into a 
tree-shaped one by adding the equality atom $y \approx y'$. 

\begin{figure}[htb]
  \centering  
  \input{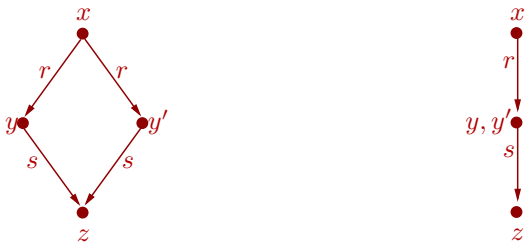}
  \caption{A representation of a cyclic query and of the tree-shaped 
           query obtained by adding the atom $y \approx y'$ to the 
           query depicted on the left hand side.}
  \label{fig:collapsing}
\end{figure}

A common property of the next three rewriting steps is that they allow 
for substituting the implicit short-cut edges with explicit paths that 
induce the short-cut. The three steps aim at different cases in which 
these short-cuts can occur and we describe their goals and application 
now in more detail:

The second stage is called \emph{split rewriting}. In a split rewriting 
we take care of all role atoms that are matched to transitive 
``short-cuts'' connecting elements of two different trees and by-passing 
one or both of their roots. We substitute these short-cuts with either 
one or two role atoms such that the roots are included. In our running 
example, $\pi$ maps $u$ to $(a, 3)$ and $x$ to $(b, \varepsilon)$. Hence 
$\I \models^\pi r(u, x)$, but the used $r$-edge is a transitive 
short-cut connecting the tree rooted in $a$ with the tree rooted in $b$, 
and by-passing $(a, \varepsilon)$. Similar arguments hold for the atom 
$r(u, z)$, where the path that implies this short-cut relationship goes 
via the two roots $(a, \varepsilon)$ and $(b, \varepsilon)$. It is clear 
that $r$ must be a non-simple role since, in the forest base \Jmc for 
\I, there is no ``direct'' connection between different trees other than 
between the roots of the trees. Hence, $(\pi(u), \pi(x)) \in \Int{r}$ 
holds only because there is a role $s \in \transR$ such that $s \sssR 
r$. In case of our example, $r$ itself is transitive. A split rewriting 
eliminates transitive short-cuts between different trees of a canonical 
model and adds the ``missing'' variables and role atoms matching the 
sequence of edges that induce the short-cut. 

\begin{figure}[htb]
  \centering
  \input{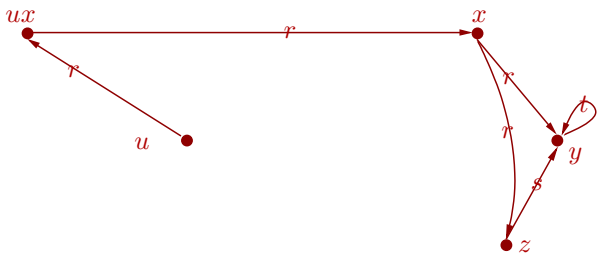}
  \caption{A split rewriting $\Q_{sr}$ for the query shown in 
           Figure~\ref{fig:query}.}
  \label{fig:splitrewriting}
\end{figure}

Figure~\ref{fig:splitrewriting} depicts the split rewriting 
$$ 
{
\setlength\arraycolsep{0.1em} 
\begin{array}{l l} 
  \Q_{sr} = \{&r(u, ux), r(ux, x), r(x, y), t(y, y), s(z, y), \\ 
  & r(u, ux), r(ux, x), r(x, z)\} 
\end{array} 
} 
$$ 
of \Q that is obtained from \Q by replacing (i) $r(u, x)$ with $r(u, 
ux)$ and $r(ux, x)$ and (ii) $r(u, z)$ with $r(u, ux), r(ux, x)$, and 
$r(x, z)$. Please note that we both introduced a new variable ($ux$) and 
re-used an existing variable ($x$). Figure~\ref{fig:splitmatch} shows a 
match for $\Q_{sr}$ and the canonical model \I of \KB in which the two 
trees are only connected via the roots. For the rewritten query, we 
also guess a set of roots, which contains the variables that are mapped 
to the roots in the canonical model. For our running example, we guess 
that the set of roots is $\{ux, x\}$. 

\begin{figure}[htb]
  \centering
  \input{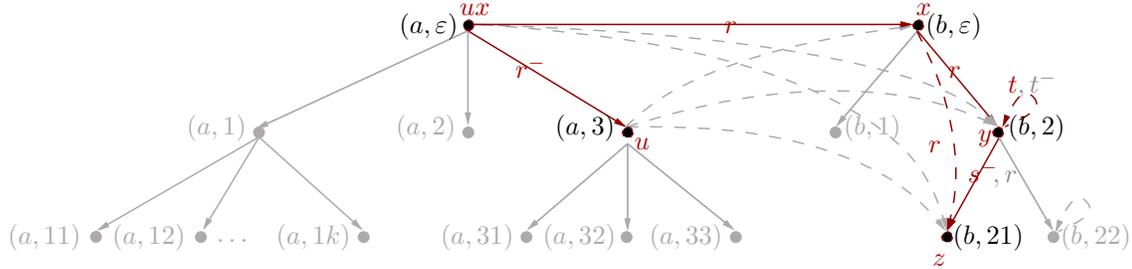}
  \caption{A split match $\pi_{sr}$ for the query $\Q_{sr}$ from 
           Figure~\ref{fig:splitrewriting} onto the canonical 
           interpretation from Figure~\ref{fig:canonical}.}
  \label{fig:splitmatch}
\end{figure}

In the third step, called \emph{loop rewriting}, we eliminate ``loops'' 
for variables $v$ that do not correspond to roots by replacing atoms 
$r(v, v)$ with two atom $r(v, v')$ and $r(v', v)$, where $v'$ can either 
be a new or an existing variable in \Q. In our running example, we 
eliminate the loop $t(y, y)$ as follows:  
$$
{\setlength\arraycolsep{0.1em}
\begin{array}{l l}
  \Q_{\ell r} = \{ & r(u, ux), r(ux, x), r(x, y), t(y, y'), t(y', y), 
                     s(z, y), \\
                   & r(u, ux), r(ux, x), r(x, z)\}
\end{array}
}
$$ 
is the query obtained from $\Q_{sr}$ (see 
Figure~\ref{fig:splitrewriting}) by replacing $t(y, y)$ with $t(y, y')$ 
and $t(y', y)$ for a new variable $y'$. Please note that, since $t$ is 
defined as transitive and symmetric, $t(y, y)$ is still implied, i.e., 
the loop is also a transitive short-cut. Figure~\ref{fig:looprewriting} 
shows the canonical interpretation \I from Figure~\ref{fig:canonical} 
with a match $\pi_{\ell r}$ for $\Q_{\ell r}$. The introduction of the 
new variable $y'$ is needed in this case since there is no variable that 
could be re-used and the individual $(b, 22)$ is not in the range of the 
match $\pi_{sr}$. 

\begin{figure}[htb]
  \centering
  \input{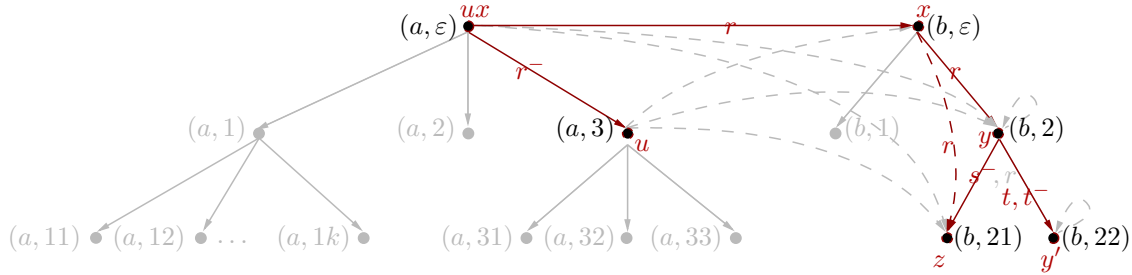}
  \caption{A loop rewriting $\Q_{\ell r}$ and a match for the canonical 
           interpretation from Figure~\ref{fig:canonical}.}
  \label{fig:looprewriting}
\end{figure}

The forth rewriting step, called \emph{forest rewriting}, allows again 
the replacement of role atoms with sets of role atoms. This allows the 
elimination of cycles that are within a single tree. A forest rewriting 
$\Q_{fr}$ for our example can be obtained from $\Q_{\ell r}$ by 
replacing the role atom $r(x, z)$ with $r(x, y)$ and $r(y, z)$, 
resulting in the query 
$$
{\setlength\arraycolsep{0.1em}
\begin{array}{l l}
  \Q_{fr} = \{ & r(u, ux), r(ux, x), r(x, y), t(y, y'), t(y', y), 
                 s(z, y), \\
               & r(u, ux), r(ux, x), r(x, y), r(y, z)\}.
\end{array}
}
$$
Clearly, this results in tree-shaped sub-queries, one rooted in $ux$ and 
one rooted in $x$. Hence $\Q_{fr}$ is forest-shaped w.r.t. the root 
terms $ux$ and $x$. Figure~\ref{fig:forestrewriting} shows the canonical 
interpretation \I from Figure~\ref{fig:canonical} with a match 
$\pi_{fr}$ for $\Q_{fr}$.

\begin{figure}[htb]
  \centering
  \input{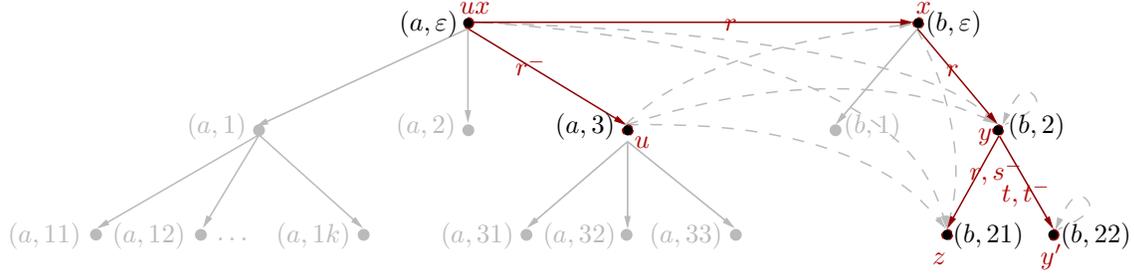}
  \caption{A forest rewriting $\Q_{fr}$ and a forest match $\pi_{fr}$ 
           for the canonical interpretation from 
           Figure~\ref{fig:canonical}.}
  \label{fig:forestrewriting}
\end{figure}

In the fifth step, we use the standard rolling-up 
technique \cite{HoTe00a,CaDL98a} and express the tree-shaped sub-queries 
as concepts. In order to do this, we traverse each tree in a bottom-up 
fashion and replace each leaf (labeled with a concept $C$, say) and its 
incoming edge (labeled with a role $r$, say) with the concept $\exists 
r.C$ added to its predecessor. For example, the tree rooted in $ux$ 
(i.e., the role atom $r(u, ux)$) can be replaced with the atom $(\exists 
r^-.\top)(ux)$. Similarly, the tree rooted in $x$ (i.e., the role atoms 
$r(x, y), r(y, z), s(z, y), t(y, y')$, and $t(y', y)$) can be replaced 
with the atom 
$$
(\exists r.((\exists (r \sqcap \inv{s}).\top) \sqcap (\exists (t \sqcap 
\inv{t}).\top))(x).
$$ 
Please note that we have to use role conjunctions in the resulting query 
in order to capture the semantics of multiple role atoms relating the 
same pair of variables.

Recall that, in the split rewriting, we have guessed that $x$ and $ux$ 
correspond to roots and, therefore, correspond to individual names in 
\indA. In the sixth and last rewriting step, we guess which variable 
corresponds to which individual name and replace the variables with the 
guessed names. A possible guess for our running example would be that 
$ux$ corresponds to $a$ and $x$ to $b$. This results in the (ground) 
query 
$$
\{(\exists r^-.\top)(a), r(a, b), (\exists r.((\exists (r \sqcap 
\inv{s}).\top) \sqcap (\exists (t \sqcap \inv{t}).\top)))(b)\}, 
$$
which is entailed by \KB. 

Please note that we focused in the running example on the most 
reasonable rewriting. There are several other possible rewritings, e.g., 
we obtain another rewriting from $\Q_{fr}$ by replacing $ux$ 
with $b$ and $x$ with $a$ in the last step. For a UCQ, we apply the 
rewriting steps to each of the disjuncts separately. 

At the end of the rewriting process, we have, for each disjunct, a set 
of ground queries and/or queries that were rolled-up into a single 
concept atom. The latter queries result from forest rewritings that are 
tree-shaped and have an empty set of roots. Such tree-shaped rewritings 
can match anywhere in a tree and can, thus, not be grounded. Finally, we 
check if our knowledge base entails the disjunction of all the rewritten 
queries. We show that there is a bound on the number of (forest-shaped) 
rewritings and hence on the number of queries produced in the rewriting 
process. 

Summing up, the rewriting process for a connected conjunctive query \Q 
involves the following steps:
\begin{enumerate}
  \item
    Build all collapsings of \Q.
  \item
    Build all split rewritings of each collapsing w.r.t.\ a subset 
    \roots of roots.
  \item
    Build all loop rewritings of the split rewritings.
  \item
    Build all (forest-shaped) forest rewritings of the loop rewritings.
  \item
    Roll up each tree-shaped sub-query in a forest-rewriting into a 
    concept atom and
  \item 
    replace the roots in \roots with individual names from the ABox in 
    all possible ways.
\end{enumerate}

Let $\Q_1, \ldots, \Q_n$ be the queries resulting from the rewriting 
process. In the next section, we define each rewriting step and prove 
that $\KB \models \Q$ iff $\KB \models \Q_1 \vee \cdots \vee \Q_n$. 
Checking entailment for the rewritten queries can easily be reduced to 
KB consistency and any decision procedure for \SHIQR KB consistency 
could be used in order to decide if $\KB \models \Q$. We present one 
such decision procedure in Section~\ref{sect:decisionprocedure}.

\section{Query Rewriting}\label{sect:rewriting}

In the previous section, we have used several terms, e.g., tree- or 
forest-shaped query, rather informally. In the following, we give 
definitions for the terms used in the query rewriting process. Once this 
is done, we formalize the query rewriting steps and prove the 
correctness of the procedure, i.e., we show that the forest-shaped 
queries obtained in the rewriting process can indeed be used for 
deciding whether a knowledge base entails the original query. We do not 
give the detailed proofs here, but rather some intuitions behind the 
proofs. Proofs in full detail are given in the appendix.  

\subsection{Tree- and Forest-Shaped Queries}\label{sect:queryshapes}

In order to define tree- or forest-shaped queries more precisely, we use 
mappings between queries and trees or forests. Instead of mapping 
equivalence classes of terms by \sapprox to nodes in a tree, we extend 
some well-known properties of functions as follows:

\begin{definition} 
  For a mapping $f \colon A \to B$, we use $\domain{f}$ and $\range{f}$ 
  to denote $f$'s \emph{domain} $A$ and \emph{range} $B$, respectively. 
  Given an equivalence relation \sapprox on $\domain{f}$, we say that 
  $f$ is \emph{injective modulo \sapprox} if, for all $a, a' \in 
  \domain{f}$, $f(a) = f(a')$ implies $a \sapprox a'$ and we say that 
  $f$ is \emph{bijective modulo $\sapprox$} if $f$ is injective modulo 
  \sapprox and surjective. 
%
%
  Let \Q be a query. A \emph{tree mapping} for \Q is a total function 
  $f$ from terms in \Q to a tree such that 
  \begin{enumerate}
    \item
      $f$ is bijective modulo \sapprox, 
    \item
      if $r(t, t') \din \Q$, then $f(t)$ is a neighbor of $f(t')$, 
      and, 
    \item
      if $a \in \inds{\Q}$, then $f(a) = \varepsilon$. 
  \end{enumerate}
  The query \Q is \emph{tree-shaped} if $\card{\inds{\Q}} \leq 1$ 
  and there is a tree mapping for \Q. 
  
  A \emph{root choice} \roots for \Q is a subset of \terms{\Q} such that 
  $\inds{\Q} \subseteq \roots$ and, if $t \in \roots$ and $t \sapprox 
  t'$, then $t' \in \roots$. For $t \in \roots$, we use $\reach{t}$ to 
  denote the set of terms $t' \in \terms{\Q}$ for which there exists a 
  sequence of terms $t_1, \dots, t_n \in \terms{\Q}$ such that 
  \begin{enumerate} 
    \item 
      $t_1 = t$ and $t_n = t'$, 
    \item 
      for all $1 \leq i < n$, there is a role $r$ such that $r(t_i, 
      t_{i+1}) \din \Q$, and,   
    \item\label{it:split} 
      for $1 < i \leq n$, if $t_i \in \roots$, then $t_i \sapprox t$.
  \end{enumerate}
  We call \roots a \emph{root splitting w.r.t.\ $\Q$} if either $\roots 
  = \emptyset$ or if, for $t_i, t_j \in \roots$, $t_i \nsapprox t_j$ 
  implies that $\reach{t_i} \cap \reach{t_j} = \emptyset$. Each term $t 
  \in \roots$ induces a sub-query 
  $$
    \begin{array}{l l} 
      \sq{\Q, t} := & \{at \din \Q \mid \mbox{ the terms in $at$ occur 
      in } \reach{t}\} \setminus \\ & 
      \{r(t, t) \mid r(t, t) \din \Q\}. 
    \end{array}
  $$ 
  A query \Q is \emph{forest-shaped w.r.t.\ a root splitting} \roots if 
  either $\roots = \emptyset$ and \Q is tree-shaped or each sub-query 
  $\sq{\Q, t}$ for $t \in \roots$ is tree-shaped. 
\end{definition}

For each term $t \in \roots$, we collect the terms that are reachable 
from $t$ in the set $\reach{t}$. By Condition~\ref{it:split}, we make 
sure that \roots and \sapprox are such that each $t' \in \reach{t}$ is 
either not in \roots or $t \sapprox t'$. Since queries are connected by 
assumption, we would otherwise collect all terms in $\reach{t}$ and not 
just those $t' \notin \roots$. For a root splitting, we require that the 
resulting sets are mutually disjoint for all terms $t, t' \in \roots$ 
that are not equivalent. This guarantees that all paths between the 
sub-queries go via the root nodes of their respective trees. 
Intuitively, a forest-shaped query is one that can potentially be mapped 
onto a canonical interpretation \I = \inter such that the terms in the 
root splitting \roots correspond to roots $(a, \varepsilon) \in \dom$. 
In the definition of $\sq{\Q, t}$, we exclude loops of the form $r(t, t) 
\din \Q$, as these parts of the query are grounded later in the query 
rewriting process and between ground terms, we allow arbitrary 
relationships. 

Consider, for example, the query $\Q_{sr}$ of our running example from 
the previous section (cf.\ Figure~\ref{fig:splitrewriting}). Let us 
again make the root choice $\roots := \{ux, x\}$ for \Q. The sets 
$\reach{ux}$ and $\reach{x}$ w.r.t.\ $\Q_{sr}$ and \roots are $\{ux, 
u\}$ and $\{x, y, z\}$ respectively. Since both sets are disjoint, 
\roots is a root splitting w.r.t.\ $\Q_{sr}$. If we choose, however, 
$\roots := \{x, y\}$, the set \roots is not a root splitting w.r.t.\ 
$\Q_{sr}$ since $\reach{x} = \{ux, u, z\}$ and $\reach{y} = \{z\}$ are 
not disjoint.

\subsection{From Graphs to Forests}\label{sect:rewritingsteps}

We are now ready to define the query rewriting steps. Given an arbitrary 
query, we exhaustively apply the rewriting steps and show that we can 
use the resulting queries that are forest-shaped for deciding entailment 
of the original query. Please note that the following definitions are 
for conjunctive queries and not for unions of conjunctive queries since 
we apply the rewriting steps for each disjunct separately. 

\begin{definition}\label{def:rewriting}
  Let \Q be a Boolean conjunctive query. A \emph{collapsing} $\Q_{co}$ 
  of \Q is obtained by adding zero or more equality atoms of the form 
  $t \approx t'$ for $t, t' \in \terms{\Q}$ to \Q. We use \co{\Q} to 
  denote the set of all queries that are a collapsing of \Q. 
  
  Let \KB be a \SHIQ knowledge base. A query $\Q_{sr}$ is called a 
  \emph{split rewriting} of \Q w.r.t.\ \KB if it is obtained from \Q by 
  choosing, for each atom $r(t, t') \din \Q$, to either:
  \begin{enumerate}
    \item 
      do nothing,
    \item 
      choose a role $s \in \transR$ such that $s \sssR r$ and replace 
      $r(t, t')$ with $s(t, u), s(u, t')$, or 
    \item 
      choose a role $s \in \transR$ such that $s \sssR r$ and replace 
      $r(t, t')$ with $s(t, u)$, $s(u, u')$, $s(u', t')$, 
  \end{enumerate}
  where $u, u' \in \NV$ are possibly fresh variables. We use \srK{\Q} to 
  denote the set of all pairs $(\Q_{sr}, \roots)$ for which there is a 
  query $\Q_{co} \in \co{\Q}$ such that $\Q_{sr}$ is a split rewriting 
  of $\Q_{co}$ and \roots is a root splitting w.r.t.\ $\Q_{sr}$.
  
  A query $\Q_{\ell r}$ is called a \emph{loop rewriting} of \Q w.r.t.\ 
  a root splitting \roots and \KB if it is obtained from \Q by choosing, 
  for all atoms of the form $r(t, t) \din \Q$ with $t \notin \roots$, a 
  role $s \in \transR$ such that $s \sssR r$ and by replacing $r(t, t)$ 
  with two atoms $s(t, t')$ and $s(t', t)$ for $t' \in \NV$ a possibly 
  fresh variable. We use \lrK{\Q} to denote the set of all pairs 
  $(\Q_{\ell r}, \roots)$ for which there is a tuple $(\Q_{sr}, \roots) 
  \in \srK{\Q}$ such that $\Q_{\ell r}$ is a loop rewriting of $\Q_{sr}$ 
  w.r.t.\ \roots and \KB.

  For a forest rewriting, fix a set $V \subseteq \NV$ of variables not 
  occurring in \Q such that $\card{V} \leq \card{\vars{\Q}}$. A 
  \emph{forest rewriting} $\Q_{fr}$ w.r.t.\ a root splitting \roots of 
  \Q and \KB is obtained from \Q by choosing, for each role atom $r(t, 
  t')$ such that either $\roots = \emptyset$ and $r(t, t') \din \Q$ or 
  there is some $t_r \in \roots$ and $r(t, t') \din \sq{\Q, t_r}$ to 
  either 
  \begin{enumerate}
    \item 
      do nothing, or
    \item
      choose a role $s \in \transR$ such that $s \sssR r$ and replace 
      $r(t, t')$ with $\ell \leq \card{\vars{\Q}}$ role atoms $s(t_1, 
      t_2), \ldots, s(t_\ell, t_{\ell + 1})$, where $t_1 = t$, $t_{\ell 
      + 1} = t'$, and $t_2, \ldots, t_\ell \in \vars{\Q} \cup V$.
  \end{enumerate}
  We use \frK{\Q} to denote the set of all pairs $(\Q_{fr}, \roots)$ for 
  which there is a tuple $(\Q_{\ell r}, \roots) \in \lrK{\Q}$ such that 
  $\Q_{fr}$ is a forest-shaped forest rewriting of $\Q_{\ell r}$ w.r.t.\ 
  \roots and \KB. 
\end{definition}

If \KB is clear from the context, we say that $\Q'$ is a split, loop, or 
forest rewriting of \Q instead of saying that $\Q'$ is a split, loop, or 
forest rewriting of \Q w.r.t.\ \KB. We assume that $\srK{\Q}, \lrK{\Q}$, 
and $\frK{\Q}$ contain no isomorphic queries, i.e., differences in 
(newly introduced) variable names only are neglected.

In the next section, we show how we can build a disjunction of 
conjunctive queries $\Q_1 \vee \cdots \vee \Q_\ell$ from the queries in 
\frK{\Q} such that each $\Q_i$ for $1 \leq i \leq \ell$ is 
either of the form $C(v)$ for a single variable $v \in \vars{\Q_i}$ or 
$\Q_i$ is ground, i.e., $\Q_i$ contains only constants and no variables. 
It then remains to show that $\KB \models \Q$ iff $\KB \models \Q_1 \vee 
\cdots \vee \Q_\ell$.

\subsection{From Trees to Concepts}
\label{sect:rollingup}

In order to transform a tree-shaped query into a single concept atom and 
a forest-shaped query into a ground query, we define a mapping $f$ from 
the terms in each tree-shaped sub-query to a tree. We then incrementally 
build a concept that corresponds to the tree-shaped query by traversing 
the tree in a bottom-up fashion, i.e., from the leaves upwards to the 
root. 

\begin{definition}\label{def:rollingUp}
  Let \Q be a tree-shaped query with at most one individual name. If $a 
  \in \inds{\Q}$, then let $t_r = a$ otherwise let $t_r = v$ for some 
  variable $v \in \vars{\Q}$. Let $f$ be a tree mapping such that 
  $f(t_r) = \varepsilon$. We now inductively assign, to each term $t \in 
  \terms{\Q}$, a concept $\con{\Q, t}$ as follows:
  \begin{itemize}
    \item 
      if $f(t)$ is a leaf of $\range{f}$, then $\con{\Q, t} := 
      \bigsqcap_{C(t) \din \Q} C$,
    \item 
      if $f(t)$ has successors $f(t_1), \dots, f(t_k)$, then  
      $$
      \begin{array}{l l}
        \con{\Q, t} := 
        & \bigsqcap_{C(t) \din \Q} C \, \sqcap \\[.5em]
        & \bigsqcap_{1 \leq i \leq k} \exists \big( 
          \bigsqcap_{r(t, t_i) \din \Q} r \big) . \con{\Q, t_i}.\\
      \end{array}
      $$
  \end{itemize}
  Finally, the \emph{query concept of \Q w.r.t.\ $t_r$} is 
  $\con{\Q, t_r}$.
\end{definition}

Please note that the above definition takes equality atoms into account. 
This is because the function $f$ is bijective modulo $\sapprox$ and, in 
case there are concept atoms $C(t)$ and $C(t')$ for $t \sapprox t'$, 
both concepts are conjoined in the query concept due to the use of the 
relation $\din$. Similar arguments can be applied to the role atoms. 

The following lemma shows that query concepts indeed capture the 
semantics of \Q.

\begin{lemma}\label{lem:conceptsandqueries}
  Let \Q be a tree-shaped query with $t_r \in \terms{\Q}$ as defined 
  above, $C_\Q = \con{\Q, t_r}$, and \I an interpretation. Then $\I 
  \models \Q$ iff there is a match $\pi$ and an element $d \in 
  \Int{C_\Q}$ such that $\pi(t_r) = d$.
\end{lemma} 

The proof given by \citeA{HSTT99a} easily transfers from \DLR to \SHIQ. 
By applying the result from the above lemma, we can now transform a 
forest-shaped query into a ground query as follows:

\begin{definition}\label{def:transformations}
  Let $(\Q_{fr}, \roots) \in \frK{\Q}$ for $\roots \neq \emptyset$, and 
  $\tau \colon \roots \to \indA$ a total function such that, for each $a 
  \in \inds{\Q}$, $\tau(a) = a$ and, for $t, t' \in \roots$, $\tau(t) = 
  \tau(t')$ iff $t \sapprox t'$. We call such a mapping $\tau$ a 
  \emph{ground mapping for \roots w.r.t.\ $\AB$}. We obtain a ground 
  query $\grounding{\Q_{fr}, \roots, \tau}$ of $\Q_{fr}$ w.r.t.\ the 
  root splitting \roots and ground mapping $\tau$ as follows:
  \begin{itemize}
    \item 
      replace each $t \in \roots$ with $\tau(t)$, and, 
    \item
      for each $a \in \range{\tau}$, replace the sub-query $\Q_a = 
      \sq{\Q_{fr}, a}$ with $\con{\Q_a, a}$.
  \end{itemize}
  We define the set \groundings{\Q} of \emph{ground queries for} \Q 
  w.r.t.\ \KB as follows: 
  $$
  \begin{array}{l l}
    \groundings{\Q} := \{\Q' \mid 
    & \mbox{there exists some }(\Q_{fr}, \roots) \in \frK{\Q} 
            \mbox{ with } \roots \neq \emptyset \\
    & \mbox{and some ground mapping $\tau$ w.r.t.\ \AB and \roots} \\
    & \mbox{such that } \Q' = \grounding{\Q_{fr}, \roots, \tau}\}
  \end{array}
  $$
  We define the set of \trees{\Q} of \emph{tree queries for} \Q as 
  follows:
  $$
  \begin{array}{l l}
    \trees{\Q} := \{\Q' \mid 
    & \mbox{there exists some }(\Q_{fr}, \emptyset) \in \frK{\Q} 
      \mbox{ and }\\
    &  v \in \vars{\Q_{fr}} \mbox{ such that } \Q' = (\con{\Q_{fr}, 
       v})(v)\}\\
  \end{array}
  $$
  \mathdefend
\end{definition}

Going back to our running example, we have already seen that $(\Q_{fr}, 
\{ux, x\})$ belongs to the set \frK{\Q} for 
$$ 
\Q_{fr} = \{r(u, ux), r(ux, x), r(x, y), t(y, y'), t(y', y), s(z, y), 
r(y, z)\}. 
$$ 
There are also several other queries in the set \frK{\Q}, e.g., $(\Q, 
\{u, x, y, z\})$, where \Q is the original query and the root splitting 
\roots is such that $\roots = \terms{\Q}$, i.e., all terms are in the 
root choice for \Q. In order to build the set $\groundings{\Q}$, we now 
build all possible ground mappings $\tau$ for the set \indA of 
individual names in our ABox and the root splittings for the queries in 
\frK{\Q}. The tuple $(\Q_{fr}, \{ux, x\}) \in \frK{\Q}$ contributes two 
ground queries for the set \groundings{\Q}: 
$$ 
{\setlength{\arraycolsep}{0em} 
\begin{array}{l} 
  \grounding{\Q_{fr}, \{ux, x\}, \{ux \mapsto a, x \mapsto b\}} = \\ 
  \{r(a, b), (\exists \inv{r}.\top)(a), (\exists r.((\exists (r \sqcap 
  \inv{s}).\top) \sqcap (\exists (t \sqcap \inv{t}).\top)))(b)\}, 
\end{array} 
} 
$$ 
where $\exists \inv{r}.\top$ is the query concept for the (tree-shaped) 
sub-query $\sq{\Q_{fr}, ux}$ and $\exists r.((\exists (r \sqcap 
\inv{s}).\top) \sqcap (\exists (t \sqcap \inv{t}).\top)$ is the query 
concept for $\sq{\Q_{fr}, x}$ and 
$$ 
{\setlength{\arraycolsep}{0.1em} 
\begin{array}{l} 
  \grounding{\Q_{fr}, \{ux, x\}, \{ux \mapsto b, x \mapsto a\}} = \\ 
  \{r(b, a), (\exists \inv{r}.\top)(b), (\exists r.((\exists (r \sqcap 
  \inv{s}).\top) \sqcap (\exists (t \sqcap \inv{t}).\top)))(a)\}. 
\end{array} 
} 
$$ 
The tuple $(\Q, \{u, x, y, z\}) \in \frK{\Q}$, however, does not 
contribute a ground query since, for a ground mapping, we require that 
$\tau(t) = \tau(t')$ iff $t \sapprox t'$ and there are only two 
individual names in \indA compared to four terms \Q that need a distinct 
value. Intuitively, this is not a restriction, since in the first 
rewriting step (collapsing) we produce all those queries in which the 
terms of \Q have been identified with each other in all possible ways. 
In our example, $\KB \models \Q$ and $\KB \models \Q_1 \vee \cdots \vee 
\Q_\ell$, where $\Q_1 \vee \cdots \vee \Q_\ell$ are the queries from 
\trees{\Q} and \groundings{\Q} since each model \I of \KB satisfies 
$\Q_i = \grounding{\Q_{fr}, \{ux, x\}, \{ux \mapsto a, x \mapsto b\}}$.

\subsection{Query Matches}
\label{sect:querymatches}

Even if a query is true in a canonical model, it does not necessarily 
mean that the query is tree- or forest-shaped. However, a match $\pi$ 
for a canonical interpretation can guide the process of rewriting a 
query. Similarly to the definition of tree- or forest-shaped queries, we 
define the shape of matches for a query. In particular, we introduce 
three different kinds of matches: \emph{split matches, forest matches}, 
and \emph{tree matches} such that every tree match is a forest match, 
and every forest match is a split match. The correspondence to the query 
shapes is as  follows: given a split match $\pi$, the set of all root 
nodes $(a, \varepsilon)$ in the range of the match define a root 
splitting for the query, if $\pi$ is additionally a forest match, the 
query is forest-shaped w.r.t.\ the root splitting induced by $\pi$, and 
if $\pi$ is additionally a tree match, then the whole query can be 
mapped to a single tree (i.e., the query is tree-shaped or forest-shaped 
w.r.t.\ an empty root splitting). Given an arbitrary query match into a 
canonical model, we can first obtain a split match and then a tree or 
forest match, by using the structure of the canonical model for guiding 
the application of the rewriting steps. 

\begin{definition}\label{def:matches}
  Let \KB be a \SHIQ knowledge base, \Q a query, \I = \inter a canonical 
  model of \KB, and $\pi \colon \terms{\Q} \to \dom$ an evaluation such 
  that $\I \models^\pi \Q$. We call $\pi$ a \emph{split match} if, for 
  all $r(t, t') \din \Q$, one of the following holds:
  \begin{enumerate}
    \item
      $\pi(t) = (a, \varepsilon)$ and $\pi(t') = (b, \varepsilon)$ for 
      some $a, b \in \indA$; or
    \item
      $\pi(t) = (a, w)$ and $\pi(t') = (a, w')$ for some $a \in \indA$ 
      and $w, w' \in \Nbbm^*$.
  \end{enumerate}
  We call $\pi$ a \emph{forest match} if, additionally, for each term 
  $t_r \in \terms{\Q}$ with $\pi(t_r) = (a, \varepsilon)$ and $a \in 
  \indA$, there is a total and bijective mapping $f$ from $\{(a, w) \mid 
  (a, w) \in \range{\pi}\}$ to a tree $T$ such that $r(t, t') \din 
  \sq{\Q, t_r}$ implies that $f(\pi(t))$ is a neighbor of $f(\pi(t'))$. 
%
  We call $\pi$ a \emph{tree match} if, additionally, there is an $a \in 
  \indA$ such that each element in $\range{\pi}$ is of the form $(a, 
  w)$.
  
  A split match $\pi$ for a canonical interpretation induces a (possibly 
  empty) root splitting \roots such that $t \in \roots$ iff $\pi(t) = 
  (a, \varepsilon)$ for some $a \in \indA$. We call \roots the 
  \emph{root splitting induced by $\pi$}. 
\end{definition}

For two elements $(a, w)$ and $(a, w')$ in a canonical model, the path 
from $(a, w)$ to $(a, w')$ is the sequence $(a, w_1), \ldots, (a, w_n)$ 
where $w = w_1, w' = w_n$, and, for $1 \leq i < n, w_{i+1}$ is a 
successor of $w_i$. The length of the path is $n$. Please note that, for 
a forest match, we do not require that $w$ is a neighbor of $w'$ or vice 
versa. This still allows to map role atoms to paths in the canonical 
model of length greater than two, but such paths must be between 
ancestors and not between elements in different branches of the tree. 
The mapping $f$ to a tree also makes sure that if \roots is the induced 
root splitting, then each sub-query \sq{\Q, t} for $t \in \roots$ is 
tree-shaped. For a tree match, the root splitting is either empty or $t 
\sapprox t'$ for each $t, t' \in \roots$, i.e., there is a single root 
modulo \sapprox, and the whole query is tree-shaped.

\subsection{Correctness of the Query Rewriting}
\label{sect:correctnessrewriting}

The following lemmas state the correctness of the rewriting step by step 
for each of the rewriting stages. Full proofs are given in the 
appendix. As motivated in the previous section, we can use a given 
canonical model to guide the rewriting process such that we obtain a 
forest-shaped query that also has a match into the model. 

\begin{lemma}\label{lem:collapsing}
  Let \I be a model for \KB. 
  \begin{enumerate}
    \item\label{it:querycollapsing} 
      If $\I \models \Q$, then there is a collapsing $\Q_{co}$ of \Q 
      such that $\I \models^{\pi_{co}} \Q_{co}$ for $\pi_{co}$ an 
      injection modulo \sapprox. 
    \item\label{it:collapsingquery} 
      If $\I \models^{\pi_{co}} \Q_{co}$ for a collapsing $\Q_{co}$ of 
      \Q, then $\I \models \Q$.
  \end{enumerate}
\end{lemma}

Given a model \I that satisfies \Q, we can simply add equality atoms for 
all pairs of terms that are mapped to the same element in \I. It is not 
hard to see that this results in a mapping that is injective modulo 
\sapprox. For the second part, it is easy to see that a model that 
satisfies a collapsing also satisfies the original query.

\begin{lemma}\label{lem:splitrewriting}
  Let \I be a model for \KB. 
  \begin{enumerate}
    \item\label{it:one} 
      If \I is canonical and $\I \models^\pi \Q$, then there is a pair 
      $(\Q_{sr}, \roots) \in \srK{\Q}$ and a split match $\pi_{sr}$ such 
      that $\I \models^{\pi_{sr}} \Q_{sr}$, \roots is the induced root 
      splitting of $\pi_{sr}$, and $\pi_{sr}$ is an injection modulo 
      \sapprox.
    \item\label{it:two}
      If $(\Q_{sr}, \roots) \in \srK{\Q}$ and $\I \models^{\pi_{sr}} 
      \Q_{sr}$ for some match $\pi_{sr}$, then $\I \models \Q$.
  \end{enumerate}
\end{lemma}

For the first part of the lemma, we proceed exactly as illustrated in 
the example section and use the canonical model \I and the match $\pi$ 
to guide the rewriting steps. We first build a collapsing $\Q_{co} \in 
\co{\Q}$ as described in the proof of Lemma~\ref{lem:collapsing} such 
that $\I \models^{\pi_{co}} \Q_{co}$ for $\pi_{co}$ an injection modulo 
\sapprox. Since \I is canonical, paths between different trees can only 
occur due to non-simple roles, and thus we can replace each role atom 
that uses such a short-cut with two or three role atoms such that these 
roots are explicitly included in the query (cf.\ the query and match in 
Figure~\ref{fig:querymatch} and the obtained split rewriting and with a 
split match in Figure~\ref{fig:splitmatch}). The second part of the 
lemma follows immediately from the fact that we use only transitive 
sub-roles in the replacement.

\begin{lemma}\label{lem:looprewriting}
  Let \I be a model of \KB. 
  \begin{enumerate}
    \item
      If \I is canonical and $\I \models \Q$, then there is a pair 
      $(\Q_{\ell r}, \roots) \in \lrK{\Q}$ and a mapping $\pi_{\ell r}$ 
      such that $\I \models^{\pi_{\ell r}} \Q_{\ell r}$, $\pi_{\ell r}$ 
      is an injection modulo \sapprox, \roots is the root splitting 
      induced by $\pi_{\ell r}$ and, for each $r(t, t) \din \Q_{\ell 
      r}$, $t \in \roots$.
    \item 
      If $(\Q_{\ell r}, \roots) \in \lrK{\Q}$ and $\I 
      \models^{\pi_{\ell r}} \Q_{\ell r}$ for some match $\pi_{\ell r}$, 
      then $\I \models \Q$.
  \end{enumerate}
\end{lemma}

The second part is again straightforward, given that we can only use 
transitive sub-roles in the loop rewriting. For the first part, we 
proceed again as described in the examples section and use the canonical 
model \I and the match $\pi$ to guide the rewriting process. We first 
build a split rewriting $\Q_{sr}$ and its root splitting \roots as 
described in the proof of Lemma~\ref{lem:splitrewriting} such that 
$(\Q_{sr}, \roots) \in \srK{\Q}$ and $\I \models^{\pi_{sr}} \Q_{sr}$ for 
a split match $\pi_{sr}$. Since \I is a canonical model, it has a forest 
base \Jmc. In a forest base, non-root nodes cannot be successors of 
themselves, so each such loop is a short-cut due to some transitive 
role. An element that is, say, $r$-related to itself has, therefore, a 
neighbor that is both an $r$- and $\inv{r}$-successor. Depending on 
whether this neighbor is already in the range of the match, we can 
either re-use an existing variable or introduce a new one, when making 
this path explicit (cf.\ the loop rewriting depicted in 
Figure~\ref{fig:looprewriting} obtained from the split rewriting shown 
in Figure~\ref{fig:splitmatch}). 

\begin{lemma}\label{lem:forestrewriting}
  Let \I be a model of \KB. 
  \begin{enumerate}
    \item
      If \I is canonical and $\I \models \Q$, then there is a pair 
      $(\Q_{fr}, \roots) \in \frK{\Q}$ such that $\I \models^{\pi_{fr}} 
      \Q_{fr}$ for a forest match $\pi_{fr}$, \roots is the induced root 
      splitting of $\pi_{fr}$, and $\pi_{fr}$ is an injection modulo 
      \sapprox.
    \item 
      If $(\Q_{fr}, \roots) \in \frK{\Q}$ and $\I \models^{\pi_{fr}} 
      \Q_{fr}$ for some match $\pi_{fr}$, then $\I \models \Q$.
  \end{enumerate}
\end{lemma}

The main challenge is again the proof of (1) and we just give a short 
idea of it here. At this point, we know from 
Lemma~\ref{lem:looprewriting} that we can use a query $\Q_{\ell r}$ for 
which there is a root splitting \roots and a split match $\pi_{\ell r}$. 
Since $\pi_{\ell r}$ is a split match, the match for each such sub-query 
is restricted to a tree and thus we can transform each sub-query of 
$\Q_{\ell r}$ induced by a term $t$ in the root choice separately. The 
following example is meant to illustrate why the given bound of 
$\card{\vars{\Q}}$ on the number of new variables and role atoms that 
can be introduced in a forest rewriting suffices. 
Figure~\ref{fig:treeification1} depicts the representation of a tree 
from a canonical model, where we use only the second part of the names 
for the elements, e.g., we use just $\varepsilon$ instead of $(a, 
\varepsilon)$. For simplicity, we also do not indicate the concepts and 
roles that label the nodes and edges, respectively. We use black color 
to indicate the nodes and edges that are used in the match for a query 
and dashed lines for short-cuts due to transitive roles. In the example, 
the grey edges are also those that belong to the forest base and the 
query match uses only short-cuts.

\begin{figure}[htb]
  \centering
  \input{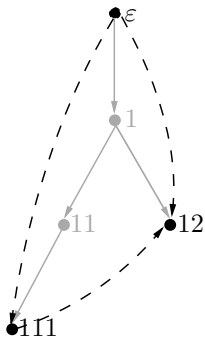}
  \caption{A part of a representation of a canonical model, where the 
  black nodes and edges are used in a match for a query and dashed edges 
  indicate short-cuts due to transitive roles.}
  \label{fig:treeification1}
\end{figure}

The forest rewriting aims at making the short-cuts more explicit by 
replacing them with as few edges as necessary to obtain a tree match. In 
order to do this, we need to include the ``common ancestors'' in the 
forest base between each two nodes used in the match. For $w, w' \in 
\Nbbm^*$, we therefore define the \emph{longest common prefix} (LCP) of 
$w$ and $w'$ as the longest $\hat{w} \in \Nbbm^*$ such that $\hat{w}$ is 
a prefix of both $w$ and $w'$. For a forest rewriting, we now determine 
the LCPs of any two nodes in the range of the match and add a variable 
for those LCPs that are not yet in the range of the match to the set $V$ 
of new variables used in the forest rewriting. In the example from 
Figure~\ref{fig:treeification1} the set $V$ contains a single variable 
$v_1$ for the node $1$. 

We now explicate the short-cuts as follows: for any edge used in the 
match, e.g., the edge from $\varepsilon$ to $111$ in the example, we 
define its \emph{path} as the sequence of elements on the path in the 
forest base, e.g., the path for the edge from $\varepsilon$ to $111$ is 
$\varepsilon, 1, 11, 111$. The \emph{relevant path} is obtained by 
dropping all elements from the path that are not in the range of the 
mapping or correspond to a variable in the set $V$, resulting in a 
relevant path of $\varepsilon, 1, 111$ for the example. We now replace 
the role atom that was matched to the edge from $\varepsilon$ to $111$ 
with two role atoms such that the match uses the edge from $\varepsilon$ 
to $1$ and from $1$ to $111$. An appropriate transitive sub-role exists 
since otherwise there could not be a short-cut. Similar arguments can be 
used to replace the role atom mapped to the edge from $111$ to $12$ and 
for the one that is mapped to the edge from $\varepsilon$ to $12$, 
resulting in a match as represented by Figure~\ref{fig:treeification2}. 
The given restriction on the cardinality of the set $V$ is no limitation 
since the number of LCPs in the set $V$ is maximal if there is no pair 
of nodes such that one is an ancestor of the other. We can see these 
nodes as $n$ leaf nodes of a tree that is at least binarily branching. 
Since such a tree can have at most $n$ inner nodes, we need at most $n$ 
new variables for a query in $n$ variables. 

\begin{figure}[htb]
  \centering
  \input{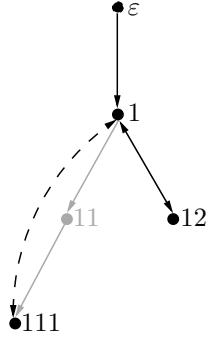}
  \caption{The match for a forest rewriting obtained from the example 
  given in Figure~\ref{fig:treeification1}.}
  \label{fig:treeification2}
\end{figure}

For the bound on the number of role atoms that can be used in the 
replacement of a single role atom, consider, for example, the cyclic 
query 
$$
\Q = \{r(x_1, x_2), r(x_2, x_3), r(x_3, x_4), t(x_1, x_4)\}, 
$$
for the knowledge base \KBDef with $\TB = \emptyset, \RB = \{r 
\sqsubseteq t\}$ with $t \in \transR$ and $\AB = \{(\exists r.(\exists 
r.(\exists r.\top)))(a)\}$. It is not hard to check that $\KB \models 
\Q$. Similarly to our running example from the previous section, there 
is also a single rewriting that is true in each canonical model of the 
KB, which is obtained by building only a forest rewriting and doing 
nothing in the other rewriting steps, except for choosing the empty set 
as root splitting in the split rewriting step. In the forest rewriting, 
we can explicate the short-cut used in the mapping for $t(x_1, x_4)$ by 
replacing $t(x_1, x_4)$ with $t(x_1, x_2), t(x_2, x_3), t(x_3, x_4)$. 

By using Lemmas~\ref{lem:collapsing} to \ref{lem:forestrewriting}, we 
get the following theorem, which shows that we can use the ground 
queries in \groundings{\Q} and the queries in \trees{\Q} in order to 
check whether \KB entails \Q, which is a well understood problem. 

\begin{theorem}\label{thm:union} 
  Let \KB be a \SHIQ knowledge base, \Q a Boolean conjunctive query, and 
  $\{\Q_1, \ldots, \Q_\ell\} = \trees{\Q} \cup \groundings{\Q}$. Then 
  $\KB \models \Q$ iff $\KB \models \Q_1 \vee \ldots \vee \Q_\ell$.
\end{theorem}

We now give upper bounds on the size and number of queries in \trees{\Q} 
and \groundings{\Q}. As before, we use $\card{S}$ to denote the 
\emph{cardinality} of a set $S$. The \emph{size} $|\KB|$ ($|\Q|$) of a 
knowledge base \KB (a query \Q) is simply the number of symbols needed 
to write it over the alphabet of constructors, concept names, and role 
names that occur in \KB (\Q), where numbers are encoded in binary. 
Obviously, the number of atoms in a query is bounded by its size, hence 
$\card{\Q} \leq |\Q|$ and, for simplicity, we use $n$ as the size 
and the cardinality of \Q in what follows. 

\begin{lemma}\label{lem:counting}
  Let \Q be a Boolean conjunctive query, \KBDef a \SHIQ knowledge base, 
  $|\Q| := n$ and $|\KB| := m$. Then there is a polynomial $p$ such that
  \begin{enumerate}
    \item\label{it:co} 
      $\card{\co{\Q}} \leq 2^{p(n)}$ and, for each $\Q' \in 
      \co{\Q}$, $|\Q'| \leq p(n)$, 
    \item\label{it:sr}
      $\card{\srK{\Q}} \leq 2^{p(n) \cdot \log p(m)}$, and, for 
      each $\Q' \in \srK{\Q}$, $|\Q'| \leq p(n)$, 
    \item\label{it:lr}
      $\card{\lrK{\Q}} \leq 2^{p(n) \cdot \log p(m)}$, and, for 
      each $\Q' \in \lrK{\Q}$, $|\Q'| \leq p(n)$, 
    \item\label{it:fr}
      $\card{\frK{\Q}} \leq 2^{p(n) \cdot \log p(m)}$, and, for 
      each $\Q' \in \frK{\Q}$, $|\Q'| \leq p(n)$, 
    \item\label{it:trees} 
      $\card{\trees{\Q}} \leq 2^{p(n) \cdot \log p(m)}$, and, for 
      each $\Q' \in \trees{\Q}$, $|\Q'| \leq p(n)$, and 
    \item\label{it:groundings} 
      $\card{\groundings{\Q}} \leq 2^{p(n) \cdot \log p(m)}$, and, 
      for each $\Q' \in \groundings{\Q}$, $|\Q'| \leq p(n)$.
  \end{enumerate}
\end{lemma}

As a consequence of the above lemma, there is a bound on the number of 
queries in \groundings{\Q} and \trees{\Q} and it is not hard to see that 
the two sets can be computed in time polynomial in $m$ and exponential 
in $n$. 

In the next section, we present an algorithm that decides entailment of 
unions of conjunctive queries, where each of the queries is either a 
ground query or consists of a single concept atom $C(x)$ for an 
existentially quantified variable $x$. By Theorem~\ref{thm:union} and 
Lemma~\ref{lem:counting}, such an algorithm is a decision procedure for 
arbitrary unions of conjunctive queries.

\subsection{Summary and Discussion}

In this section, we have presented the main technical foundations for 
answering (unions of) conjunctive queries. It is known that queries that 
contain non-simple roles in cycles among existentially quantified 
variables are difficult to handle. By applying the rewriting steps from 
Definition~\ref{def:rewriting}, we can rewrite such cyclic conjunctive 
queries into a set of acyclic and/or ground queries. Both types of 
queries are easier to handle and algorithms for both types exist. At 
this point, any reasoning algorithm for \SHIQR knowledge base 
consistency can be used for deciding query entailment. In order to 
obtain tight complexity results, we present in the following section a 
decision procedure that is based on an extension of the translation to 
looping tree automata given by \citeA{Tobi01a}. 

It is worth mentioning that, for queries with only simple roles, our 
algorithm behaves exactly as the existing rewriting algorithms (i.e., 
the rolling-up and tuple graph technique) since, in this case, only the 
collapsing step is applicable. The need for identifying variables was 
first pointed out in the work of \citeA{HSTT99a} and is also required 
(although not mentioned) for the algorithm proposed by \citeA{CaDL98a}. 

The new rewriting steps (split, loop, and forest rewriting) are only 
required for and applicable to non-simple roles and, when replacing a 
role atom, only transitive sub-roles of the replaced role can be used. 
Hence the number of resulting queries is in fact not determined by the 
size of the whole knowledge base, but by the number of transitive 
sub-roles for the non-simple roles in the query. Therefore, the number 
of resulting queries really depends on the number of transitive roles 
and the depth of the role hierarchy for the non-simple roles in the 
query, which can, usually, expected to be small.

\section{The Decision Procedure}
\label{sect:decisionprocedure}

We now devise a decision procedure for entailment of unions of Boolean 
conjunctive queries that uses, for each disjunct, the queries obtained 
in the rewriting process as defined in the previous section. Detailed 
proofs for the lemmas and theorems in this section can again be found in 
the appendix. For a knowledge base \KB and a union of Boolean 
conjunctive queries $\Q_1 \vee \ldots \vee \Q_\ell$, we show how we can 
use the queries in $\trees{\Q_i}$ and $\groundings{\Q_i}$ for $1 
\leq i \leq \ell$ in order to build a set of knowledge bases 
$\KB_1, \dots, \KB_n$ such that $\KB \models \Q_1 \vee \ldots \vee 
\Q_\ell$ iff all the $\KB_i$ are inconsistent. This gives rise to two 
decision procedures: a deterministic one in which we enumerate all 
$\KB_i$, and which we use to derive a tight upper bound for the combined 
complexity; and a non-deterministic one in which we guess a $\KB_i$, and 
which yields a tight upper bound for the data complexity. Recall that, 
for combined complexity, the knowledge base \KB and the queries $\Q_i$ 
both count as input, whereas for the data complexity only the ABox \AB 
counts as an input, and all other parts are assumed to be fixed.

\subsection{A Deterministic Decision Procedure for Query Entailment in 
\SHIQ}\label{sect:deterministic}

We first define the deterministic version of the decision procedure and 
give an upper bound for its combined complexity. The given algorithm 
takes as input a union of connected conjunctive queries and works under 
the unique name assumption (UNA). We show afterwards how it can be 
extended to an algorithm that does not make the UNA and that takes 
arbitrary UCQs as input, and that the complexity results carry over. 

We construct a set of knowledge bases that extend the original knowledge 
base \KB both w.r.t.\ the TBox and ABox. The extended knowledge bases 
are such that a given KB \KB entails a query \Q iff all the extended KBs 
are inconsistent. We handle the concepts obtained from the tree-shaped 
queries differently to the ground queries: the axioms we add to the TBox 
prevent matches for the tree-shaped queries, whereas the extended ABoxes 
contain assertions that prevent matches for the ground queries. 

\begin{definition}
  Let \KBDef be a \SHIQ knowledge base and $\Q = \Q_1 \vee \ldots \vee 
  \Q_\ell$ a union of Boolean conjunctive queries. We set  
  \begin{enumerate}
    \item
      $T := \trees{\Q_1} \cup \ldots \cup \trees{\Q_\ell}$,
    \item 
      $G := \groundings{\Q_1} \cup \ldots \cup \groundings{\Q_\ell}$, 
      and 
    \item\label{it:treeGCIs} 
      $\TB_\Q := \{\top \sqsubseteq \neg C \mid C(v) \in T\}$.
  \end{enumerate}
  An \emph{extended knowledge base} $\KB_\Q$ w.r.t.\ \KB and \Q is a 
  tuple $(\TB \cup \TB_\Q, \RB, \AB \cup \AB_\Q)$ such that $\AB_\Q$ 
  contains, for each $\Q' \in G$, at least one assertion $\neg at$ with 
  $at \in \Q'$. 
\end{definition}

Informally, the extended TBox $\TB \cup \TB_\Q$ ensures that there are 
no tree matches. Each extended ABox $\AB \cup \AB_\Q$ contains, for each 
ground query $\Q'$ obtained in the rewriting process, at least one 
assertion $\neg at$ with $at \in \Q'$ that ``spoils'' a match for $\Q'$. 
A model for such an extended ABox can, therefore, not satisfy any of the 
ground queries. If there is a model for any of the extended knowledge 
bases, we know that this is a counter-model for the original query.

We can now use the extended knowledge bases in order to define the 
deterministic version of our algorithm for deciding entailment of unions 
of Boolean conjunctive queries in \SHIQ. 

\begin{definition}\label{def:algodeterministic}
  Given a \SHIQ knowledge base \KBDef and a union of connected Boolean 
  conjunctive queries \Q as input, the algorithm answers ``\KB entails 
  \Q'' if each extended knowledge base w.r.t.\ \KB and \Q is 
  inconsistent and it answers ``\KB does not entail \Q'' otherwise.
\end{definition}

The following lemma shows that the above described algorithm is indeed 
correct. 

\begin{lemma}\label{lem:correctness}
  Let \KB be a \SHIQ knowledge base and \Q a union of connected Boolean 
  conjunctive queries. Given \KB and \Q as input, the algorithm from 
  Definition~\ref{def:algodeterministic} answers ``\KB entails \Q'' iff 
  $\KB \models \Q$ under the unique name assumption.
\end{lemma}

In the proof of the if direction for the above lemma, we can use a 
canonical model \I of \KB in order to guide the rewriting process. For 
the only if direction, we assume to the contrary of what is to be shown 
that there is no consistent extended knowledge base, but $\KB 
\not\models \Q$. We then use a model \I of \KB such that $\I 
\not\models \Q$, which exists by assumption, and show that \I is also a 
model of some extended knowledge base.

\subsubsection{Combined Complexity of Query Entailment in \SHIQ}
\label{sect:combinedcomplexity}

According to the above lemma, the algorithm given in 
Definition~\ref{def:algodeterministic} is correct. We now analyse its 
combined complexity and thereby prove that it is also terminating. 

For the complexity analysis, we assume, as usual 
\cite{HuMS05a,CGLL06a,OrCE06a}, that all concepts in concept atoms and 
ABox assertions are literals, i.e., concept names or negated concept 
names. If the input query or ABox contains non-literal atoms or 
assertions, we can easily transform these into literal ones in a truth 
preserving way: for each concept atom $C(t)$ in the query where $C$ is a 
non-literal concept, we introduce a new atomic concept $A_C \in \NC$, 
add the axiom $C \sqsubseteq A_C$ to the TBox, and replace $C(t)$ with 
$A_C(t)$; for each non-literal concept assertion $C(a)$ in the ABox, we 
introduce a new atomic concept $A_C \in \NC$, add an axiom $A_C 
\sqsubseteq C$ to the TBox, and replace $C(a)$ with $A_C(a)$. Such a 
transformation is obviously polynomial, so without loss of generality, 
it is safe to assume that the ABox and query contain only literal 
concepts. This has the advantage that the size of each atom and ABox 
assertion is constant. 

Since our algorithm involves checking the consistency of a \SHIQR 
knowledge base, we analyse the complexity of this reasoning 
service.~\citeA{Tobi01a} shows an \ExpTime upper bound for deciding the 
consistency of \SHIQ knowledge bases (even with binary coding of 
numbers) by translating a \SHIQ KB to an equisatisfiable \ALCQIb 
knowledge base. The $b$ stands for \emph{safe Boolean role expressions} 
built from \ALCQIb roles using the operator $\sqcap$ (role 
intersection), $\sqcup$ (role union), and $\neg$ (role 
negation/complement) such that, when transformed into disjunctive normal 
form, every disjunct contains at least one non-negated conjunct. Given a 
query \Q and a \SHIQ knowledge base $\KBDef$, we reduce query entailment 
to deciding knowledge base consistency of an extended \SHIQR knowledge 
base $\KB_\Q = (\TB \cup \TB_\Q, \RB, \AB \cup \AB_\Q)$. Recall that 
$\TB_\Q$ and $\AB_\Q$ are the only parts that contain role conjunctions 
and that we use role negation only in ABox assertions. We extend the 
translation given for \SHIQ so that it can be used for deciding the 
consistency of \SHIQR KBs. Although the translation works for all \SHIQR 
KBs, we assume the input KB to be of exactly the form of extended 
knowledge bases as described above. This is so because the translation 
for unrestricted \SHIQR is no longer polynomial, as in the case of 
\SHIQ, but exponential in the size of the longest role conjunction under 
a universal quantifier. Since role conjunctions occur only in the 
extended ABox and TBox, and since the size of each role conjunction is, 
by Lemma~\ref{lem:counting}, polynomial in the size of \Q, the 
translation is only exponential in the size of the query in the case of 
extended knowledge bases. 

We assume here, as usual, that all concepts are in \emph{negation normal 
form (NNF)}; any concept can be transformed in linear time into an 
equivalent one in NNF by pushing negation inwards, making use of de 
Morgan's laws and the duality between existential and universal 
restrictions, and between atmost and atleast number restrictions 
($\leqslant n\; r.C$ and $\geqslant n\; r.C$ respectively) 
\cite{HoST00a}. For a concept $C$, we use $\dot\neg C$ to denote the NNF 
of $\neg C$. 

We define the \emph{closure} $\cl{C, \RB}$ of a concept $C$ w.r.t.\ a 
role hierarchy \RB as the smallest set satisfying the following 
conditions:
\begin{itemize}
  \item 
    if $D$ is a sub-concept of $C$, then $D \in \cl{C, \RB}$,
  \item 
    if $D \in \cl{C, \RB}$, then $\dot\neg D \in \cl{C, \RB}$,
  \item 
    if $\forall r.D \in \cl{C, \RB}, s \sssR r$, and $s \in \transR$, 
    then $\forall s.D \in \cl{C, \RB}$.
\end{itemize}

We now show how we can extend the translation from \SHIQ to \ALCQIb 
given by Tobies. We first consider $\SHIQR$-concepts and then extend the 
translation to KBs. 

\begin{definition}
  For a role hierarchy \RB and roles $r, r_1, \ldots, r_n$, let 
  $$
    \up{r} = \bigsqcap \limits_{r \ssssR s} s \mbox{\hspace{.4cm} and 
    \hspace{.4cm}}\up{r_1 \sqcap \ldots \sqcap r_n} = \up{r_1} \sqcap 
    \ldots \sqcap \up{r_n}.
  $$
  \mathdefend
\end{definition}

Please note that, since $r \sssR r$, $r$ occurs in \up{r}.\hfill{}

\begin{lemma}\label{lem:uparrow}
  Let \RB be a role hierarchy, and $r_1, \ldots, r_n$ roles. For 
  every interpretation \I such that $\I \models \RB$, it holds that 
  $\Int{(\up{r_1 \sqcap \ldots \sqcap r_n})} = \Int{(r_1 \sqcap \ldots 
  \sqcap r_n)}$.
\end{lemma}

With the extended definition of $\uparrow$ on role conjunctions, we can 
now adapt the definition (Def.\ 6.22) that Tobies provides for 
translating $\SHIQ$-concepts into $\ALCQIb$-concepts.

\begin{definition}\label{def:translation}
  Let $C$ be a $\SHIQR$-concept in NNF and \RB a role hierarchy. For 
  every concept $\forall (r_1 \sqcap \ldots \sqcap r_n).D \in 
  \cl{C, \RB}$, let $X_{r_1 \sqcap \ldots \sqcap r_n, D} \in \NC$ be a 
  unique concept name that does not occur in $\cl{C, \RB}$. Given a role 
  hierarchy \RB, we define the function $\mn{tr}$ inductively on the 
  structure of concepts by setting
  $$
  \begin{array}{rcl}
    \tr{A} & = & A \mbox{ for all A } \in \NC\\
    \tr{\neg A} & = & \neg A \mbox{ for all A } \in \NC\\
    \tr{C_1 \sqcap C_2} & = & \tr{C_1} \sqcap \tr{C_2}\\
    \tr{C_1 \sqcup C_2} & = & \tr{C_1} \sqcup \tr{C_2}\\
    \tr{\bowtie n (r_1 \sqcap \ldots \sqcap r_n).D} & = & (\bowtie 
    n \up{r_1 \sqcap \ldots \sqcap r_n}.\tr{D})\\
    \tr{\forall (r_1 \sqcap \ldots \sqcap r_n).D} & = & X_{r_1 
    \sqcap \ldots \sqcap r_n, D}\\
    \tr{\exists (r_1 \sqcap \ldots \sqcap r_n).D} & = & \neg 
    (X_{r_1 \sqcap \ldots \sqcap r_n, \dot \neg D})\\
  \end{array}
  $$
  where $\bowtie$ stands for $\leqslant$ or $\geqslant$. Set 
  $\mn{tc}((r_1 \sqcap \ldots \sqcap r_n), \RB) := \{(t_1 \sqcap \ldots 
  \sqcap t_n) \mid t_i \sssR r_i \mbox{ and } t_i \in \transR$ for each 
  $i$ such that $1 \leq i \leq n\}$ and define an extended 
  TBox $\TB_{C, \RB}$ as 
  $$
  {\setlength{\arraycolsep}{0em}
  \begin{array}{rcll}
    \TB_{C, \RB} & = & \{ X_{r_1 \sqcap \ldots \sqcap r_n, D} \equiv 
    \forall \up{r_1 \sqcap \ldots \sqcap r_n}.\tr{D} & \mid \forall 
    (r_1 \sqcap \ldots \sqcap r_n).D \in \mn{cl}(C, \RB)\}\mbox{ }\cup\\
    & & \{ X_{r_1 \sqcap \ldots \sqcap r_n, D} \sqsubseteq \forall 
    \up{T}.X_{T, D} & \mid T \in \mn{tc}(r_1 \sqcap \ldots \sqcap r_n, 
    \RB)\}\\
  \end{array}
  }
  $$
  \mathdefend
\end{definition}

\begin{lemma}\label{lem:tr}
  Let $C$ be a $\SHIQR$-concept in NNF, \RB a role hierarchy, and 
  $\mn{tr}$ and $\TB_{C, \RB}$ as defined in 
  Definition~\ref{def:translation}. The concept $C$ is satisfiable 
  w.r.t.\ \RB iff the \ALCQIb-concept $\tr{C}$ is satisfiable w.r.t.\ 
  $\TB_{C, \RB}$.
\end{lemma}

Given Lemma~\ref{lem:uparrow}, the proof of Lemma~\ref{lem:tr} is a 
long, but straightforward extension of the proof given by 
\citeA[Lemma~6.23]{Tobi01a}.

We now analyse the complexity of the above described problem. Let $m := 
|\RB|$ and $r_1 \sqcap \ldots \sqcap r_n$ the longest role conjunction 
occurring in $C$, i.e., the maximal number of roles that occur in a role 
conjunction in $C$ is $n$. The TBox $\TB_{C, \RB}$ can contain 
exponentially many axioms in $n$ since the cardinality of the set 
$\mn{tc}((r_1 \sqcap \ldots \sqcap r_n), \RB)$ for the longest role 
conjunction can only be bounded by $m^n$ because each $r_i$ can have 
more than one transitive sub-role. It is not hard to check that the size 
of each axiom is polynomial in $|C|$. Since deciding whether an \ALCQIb 
concept $C$ is satisfiable w.r.t.\ an \ALCQIb TBox \TB is an 
\ExpTime-complete problem (even with binary coding of numbers) 
\cite[Thm.\ 4.42]{Tobi01a}, the satisfiability of a $\SHIQR$-concept 
$C$ can be checked in time $2^{p(m) 2^{p(n)}}$. 

We now extend the translation from concepts to knowledge bases. Tobies 
assumes that all role assertions in the ABox are of the form $r(a, b)$ 
with $r$ a role name or the inverse of a role name. Extended ABoxes 
contain, however, also negated roles in role assertions, which require a 
different translation. A positive role assertion such as $r(a, b)$ is 
translated in the standard way by closing the role upwards. The only 
difference of using $\uparrow$ directly is that we additionally split 
the conjunction $(\up{r})(a, b) = (r_1 \sqcap \ldots \sqcap r_n)(a, b)$ 
into $n$ different role assertions $r_1(a, b), \ldots, r_n(a, b)$, which 
is clearly justified by the semantics. For negated roles in a role 
assertion such as $\neg r(a, b)$, we close the role downwards instead of 
upwards and add a role atom $\neg s(a, b)$ for each sub-role $s$ of 
$r$. This is again justified by the semantics. Let $\KB = (\TB \cup 
\TB_\Q, \RB, \AB \cup \AB_\Q)$ be an extended knowledge base. More 
precisely, we set 
$$
  \tr{\TB \cup \TB_\Q} := \{\tr{C} \sqsubseteq \tr{D} \mid C \sqsubseteq 
  D \in \TB \cup \TB_\Q\},
$$ 
$$
  \begin{array}{l l}
    \tr{\AB \cup \AB_\Q} := & \{(\tr{C})(a) \mid C(a) \in \AB \cup 
          \AB_\Q\} \mbox{ } \cup \\
      & \{s(a, b) \mid r(a, b) \in \AB \cup \AB_\Q \mbox{ and } r \sssR 
          s\} \mbox{ } \cup \\
      & \{\neg s(a, b) \mid \neg r(a, b) \in \AB \cup \AB_\Q \mbox{ and  
          } s \sssR r\},
  \end{array}
$$
and we use \tr{\KB} to denote the \ALCQIb knowledge base (\tr{\TB \cup 
\TB_\Q}, \tr{\AB \cup \AB_\Q}).

For the complexity of deciding the consistency of a translated \SHIQR 
knowledge base, we can apply the same arguments as above for concept 
satisfiability, which gives the following result: 

\begin{lemma}\label{lem:shiqrconsistency} 
  Given a \SHIQR knowledge base \KBDef where $m := |\KB|$ and the size 
  of the longest role conjunction is $n$, we can decide consistency of 
  \KB in deterministic time $2^{p(m) 2^{p(n)}}$ with $p$ a polynomial.
\end{lemma}

We are now ready to show that the algorithm given in 
Definition~\ref{def:algodeterministic} runs in deterministic time single 
exponential in the size of the input KB and double exponential in the 
size of the input query.

\begin{lemma}\label{lem:combinedcomplexity}
  Let \KBDef be a \SHIQ knowledge base with $m = |\KB|$ and \Q a union 
  of connected Boolean conjunctive queries with $n = |\Q|$. Given \KB 
  and \Q as input, the algorithm given in 
  Definition~\ref{def:algodeterministic} decides whether $\KB \models 
  \Q$ under the unique name assumption in deterministic time in 
  $2^{p(m) 2^{p(n)}}$. 
\end{lemma}

In the proof of the above lemma, we show that there is some polynomial 
$p$ such that we have to check at most $2^{p(m) 2^{p(n)}}$ extended 
knowledge bases for consistency and that each consistency check can be 
done in this time bound as well. 

More precisely, let $\Q = \Q_1 \vee \ldots \vee \Q_\ell, T = 
\trees{\Q_1} \cup \ldots \cup \trees{\Q_\ell}$, and $G = 
\groundings{\Q_1} \cup \ldots \cup \groundings{\Q_\ell}$. Together with 
Lemma~\ref{lem:counting}, we get that $\card{T}$ and $\card{G}$ are 
bounded by $2^{p(n) \cdot \log p(m)}$ for some polynomial $p$ and that 
the size of each query in $G$ and $T$ is polynomial in $n$. Each of the 
$2^{p(n) \cdot \log p(m)}$ ground queries in $G$ contributes at most 
$p(n)$ negated assertion to an extended ABox $\AB_\Q$. Hence, there are 
at most $2^{p(m) 2^{p(n)}}$ extended ABoxes $\AB_\Q$ and, therefore, 
$2^{p(m) 2^{p(n)}}$ extended knowledge bases that have to be tested for 
consistency. 

Given the bounds on the cardinalities of $T$ and $G$ and the fact that 
the size of each query in $T$ and $G$ is polynomial in $n$, it is not 
hard to check that the size of each extended knowledge base $\KB_\Q = 
(\TB \cup \TB_\Q, \RB, \AB \cup \AB_\Q)$ is bounded by $2^{p(n) \cdot 
\log p(m)}$ and that each $\KB_\Q$ can be computed in this time bound as 
well. Since only the extended parts contain role conjunctions and the 
number of roles in a role conjunction is polynomial in $n$, there is a 
polynomial $p$ such that  
\begin{enumerate}
  \item
    $|\tr{\TB}| \leq p(m)$, 
  \item
    $|\tr{\TB_\Q}| \leq 2^{p(n) \cdot \log p(m)}$, 
  \item
    $|\tr{\AB}| \leq p(m)$, 
  \item
    $|\tr{\AB_\Q}| \leq 2^{p(n) \cdot \log p(m)}$, and, hence, 
  \item
    $|\tr{\KB_\Q}| \leq 2^{p(n) \cdot \log p(m)}$.
\end{enumerate}
By Lemma~\ref{lem:shiqrconsistency}, each consistency check can be done 
in time $2^{p(m) 2^{p(n)}}$ for some polynomial $p$. Since we have to 
check at most $2^{p(m) 2^{p(n)}}$ extended knowledge bases for 
consistency, and each check can be done in time $2^{p(m) 2^{p(n)}}$, we 
obtain the desired upper bound. 

We now show that this result carries over even when we do not restrict 
interpretations to the unique name assumption.
  
\begin{definition}
  Let \KBDef be a \SHIQ knowledge base and \Q a \SHIQ union of Boolean 
  conjunctive queries. For a partition \Pmc of \indA, a knowledge base 
  $\KB^\Pmc = (\TB, \RB, \AB^\Pmc)$ and a query $\Q^\Pmc$ are called an 
  \emph{$\AB$-partition w.r.t.\ \KB and \Q} if $\AB^\Pmc$ and $\Q^\Pmc$ 
  are obtained from \AB and \Q as follows:\\
  For each $P \in \Pmc$
  \begin{enumerate}
    \item 
      Choose one individual name $a \in P$.
    \item
      For each $b \in P$, replace each occurrence of $b$ in \AB and \Q 
      with $a$. 
  \end{enumerate}
\end{definition}

Please note that w.l.o.g.\ we assume that all constants that occur in 
the query occur in the ABox as well and that thus a partition of the 
individual names in the ABox also partitions the query.

\begin{lemma}\label{lem:UNA}
  Let \KBDef be a \SHIQ knowledge base and \Q a union of Boolean 
  conjunctive queries. $\KB \not\models \Q$ without making the unique 
  name assumption iff there is an \AB-partition $\KB^\Pmc = (\TB, \RB, 
  \AB^\Pmc)$ and $\Q^\Pmc$ w.r.t.\ \KB and \Q such that $\KB^\Pmc 
  \not\models \Q^\Pmc$ under the unique name assumption.
\end{lemma}

Let \KBDef be a knowledge base in a Description Logic $\mathcal{DL}$, 
\Cmc be the complexity class such that deciding whether $\KB \models \Q$ 
under the unique name assumption is in \Cmc, and let $n = 2^{|\AB|}$. 
Since the number of partitions for an ABox is at most exponential in the 
number of individual names that occur in the ABox, the following is a 
straightforward consequence of the above lemma: for a Boolean 
conjunctive $\mathcal{DL}$ query \Q, deciding whether $\KB \models \Q$ 
without making the unique name assumption can be reduced to deciding $n$ 
times a problem in \Cmc. 

In order to extend our algorithm to unions of possibly unconnected 
Boolean conjunctive queries, we first transform the input query \Q into 
conjunctive normal form (CNF). We then check entailment for each 
conjunct $\Q_i$, which is now a union of connected Boolean conjunctive 
queries. The algorithm returns ``\KB entails \Q'' if each entailment 
check succeeds and it answers ``\KB does not entail \Q'' otherwise. By 
Lemma~\ref{lem:cnf} and Lemma~\ref{lem:correctness}, the algorithm is 
correct. 

Let \KB be a knowledge base in a Description Logic $\mathcal{DL}$, \Q a 
union of connected Boolean conjunctive $\mathcal{DL}$ queries, and \Cmc 
the complexity class such that deciding whether $\KB \models \Q$ is in 
\Cmc. Let $\Q'$ be a union of possibly unconnected Boolean conjunctive 
queries and $\mn{cnf}(\Q')$ the CNF of $\Q'$. Since the number of 
conjuncts in $\mn{cnf}(\Q')$ is at most exponential in $|\Q'|$, deciding 
whether $\KB \models \Q'$ can be reduced to deciding $n$ times a problem 
in \Cmc, with $n = 2^{p(|\Q'|)}$ and $p$ a polynomial. 

The above observation together with the results from 
Lemma~\ref{lem:combinedcomplexity} gives the following general result:

\begin{theorem}\label{thm:combinedcomplexity}
  Let \KBDef be a \SHIQ knowledge base with $m = |\KB|$ and $\Q$ a union 
  of Boolean conjunctive queries with $n = |\Q|$. Deciding whether $\KB 
  \models \Q$ can be done in deterministic time in $2^{p(m) 2^{p(n)}}$. 
\end{theorem}

A corresponding lower bound follows from the work by \citeA{Lutz07a}. 
Hence the above result is tight. The result improves the known 
co-3\NExpTime upper bound for the setting where the roles in the query 
are restricted to simple ones \cite{OrCE06b}. 

\begin{corollary}
  Let \KB be a \SHIQ knowledge base with $m = |\KB|$ and \Q a union of 
  Boolean conjunctive queries with $n = |\Q|$. Deciding whether $\KB 
  \models \Q$ is a 2\ExpTime-complete problem. 
\end{corollary}

Regarding query answering, we refer back to the end of 
Section~\ref{pre:sect:cqs}, where we explain that deciding which 
tuples belong to the set of answers can be checked with at most 
$m_\AB^k$ entailment tests, where $k$ is the number of answer 
variables in the query and $m_\AB$ is the number of individual names in 
\indA. Hence, at least theoretically, this is absorbed by the combined 
complexity of query entailment in \SHIQ.

\subsection{A Non-Deterministic Decision Procedure for Query Entailment 
in \SHIQ}\label{sect:nondeterministic}

In order to study the data complexity of query entailment, we devise a 
non-deterministic decision procedure which provides a tight bound for 
the complexity of the problem. Actually, the devised algorithm decides 
non-entailment of queries: we guess an extended knowledge base 
$\KB_\Q$, check whether it is consistent, and return ``\KB does not 
entail \Q'' if the check succeeds and ``\KB entails \Q'' otherwise. 

\begin{definition}\label{def:algonondeterministic}
  Let \TB be a \SHIQ TBox, \RB a \SHIQ role hierarchy, and \Q a union of 
  Boolean conjunctive queries. Given a \SHIQ ABox \AB as input, the 
  algorithm guesses an \AB-partition $\KB^\Pmc = (\TB, \RB, \AB^\Pmc)$ 
  and $\Q^\Pmc$ w.r.t.\ \KB = (\TB, \RB, \AB) and \Q. The query 
  $\Q^\Pmc$ is then transformed into CNF and one of the resulting 
  conjuncts, say $\Q_i^\Pmc$, is chosen. The algorithm then guesses an 
  extended knowledge base $\KB_{\Q_i}^\Pmc = (\TB \cup \TB_{\Q_i}, \RB, 
  \AB^\Pmc \cup \AB^\Pmc_{\Q_i})$ w.r.t.\ $\KB^\Pmc$ and $\Q_i^\Pmc$ and 
  returns ``\KB does not entail \Q'' if $\KB_{\Q_i}^\Pmc$ is consistent 
  and it returns ``\KB entails \Q'' otherwise.
\end{definition}

Compared to the deterministic version of the algorithm given in 
Definition~\ref{def:algodeterministic}, we do not make the UNA but guess 
a partition of the individual names. We also non-deterministically 
choose one of the conjuncts that result from the transformation into 
CNF. For this conjunct, we guess an extended ABox and check whether the 
extended knowledge base for the guessed ABox is consistent and, 
therefore, a counter-model for the query entailment. 

In its (equivalent) negated form, Lemma~\ref{lem:correctness} says that 
$\KB \not\models \Q$ iff there is an extended knowledge base $\KB_\Q$ 
w.r.t.\ \KB and $\Q$ such that $\KB_\Q$ is consistent. Together with 
Lemma~\ref{lem:UNA} it follows, therefore, that the algorithm from 
Definition~\ref{def:algonondeterministic} is correct.

\subsubsection{Data Complexity of Query Entailment in \SHIQ}
\label{sect:datacomplexity}

We now analyze the data complexity of the algorithm given in 
Definition~\ref{def:algonondeterministic} and show that deciding UCQ 
entailment in \SHIQ is indeed in co-\NPclass for data complexity. 

\begin{theorem}\label{thm:datacomplexity}
  Let \TB be a \SHIQ TBox, \RB a \SHIQ role hierarchy, and \Q a union of 
  Boolean conjunctive queries. Given a \SHIQ ABox \AB with $m_a = 
  |\AB|$, the algorithm from Definition~\ref{def:algonondeterministic} 
  decides in non-deterministic polynomial time in $m_a$ whether $\KB 
  \not\models \Q$ for \KB = (\TB, \RB, \AB). 
\end{theorem}

Clearly, the size of an ABox $\AB^\Pmc$ in an \AB-partition is bounded 
by $m_a$. Since the query is no longer an input, its size is constant 
and the transformation to CNF can be done in constant time. We then 
non-deterministically choose one of the resulting conjuncts. Let this 
conjunct be $\Q_i = \Q_{(i, 1)} \vee \ldots \vee \Q_{(i, \ell)}$. As 
established in Lemma~\ref{thm:combinedcomplexity}, the maximal size of 
an extended ABox $\AB^\Pmc_{\Q_i}$ is polynomial in $m_a$. Hence, 
$|\AB^\Pmc \cup \AB^\Pmc_{\Q_i}| \leq p(m_a)$ for some polynomial $p$. 
Due to Lemma~\ref{lem:counting} and since the size of \Q, \TB, and \RB 
is fixed by assumption, the sets $\trees[\KB^\Pmc]{\Q_{(i, j)}}$ and 
$\groundings[\KB^\Pmc]{\Q_{(i, j)}}$ for each $j$ such that $1 \leq j 
\leq \ell$ can be computed in time polynomial in $m_a$. From 
Lemma~\ref{lem:combinedcomplexity}, we know that the translation of an 
extended knowledge base into an \ALCQIb knowledge base is polynomial in 
$m_a$ and a close inspection of the algorithm by \citeA{Tobi01a} for 
deciding consistency of an \ALCQIb knowledge base shows that its runtime 
is also polynomial in $m_a$.

The bound given in Theorem~\ref{thm:datacomplexity} is tight since the 
data complexity of conjunctive query entailment is already 
co-\NPclass-hard for the \ALE fragment of \SHIQ \cite{Scha93a}.

\begin{corollary}
  Conjunctive query entailment in \SHIQ is data complete for 
  co-\NPclass.
\end{corollary}

Due to the correspondence between query containment and query answering 
\cite{CaDL98a}, the algorithm can also be used to decide containment 
of two unions of conjunctive queries over a \SHIQ knowledge base, which 
gives the following result:

\begin{corollary}
  Given a \SHIQ knowledge base \KB and two unions of conjunctive queries 
  $\Q$ and $\Q'$, the problem whether $\KB \models \Q \subseteq \Q'$ is 
  decidable.
\end{corollary}

By using the result of \citeA[Thm.~11]{Rosa06a}, we further show that 
the consistency of a \SHIQ knowledge base extended with (weakly-safe) 
Datalog rules is decidable.

\begin{corollary}
  The consistency of $\SHIQ$+$log$-KBs (both under FOL semantics and 
  under NM semantics) is decidable.
\end{corollary}

\section{Conclusions}\label{sect:conclusions}

With the decision procedure presented for entailment of unions of 
conjunctive queries in \SHIQ, we close a long standing open problem. The 
solution has immediate consequences on related areas, as it shows that 
several other open problems such as query answering, query containment 
and the extension of a knowledge base with weakly safe Datalog rules for 
\SHIQ are decidable as well. Regarding combined complexity, we present a 
deterministic algorithm that needs time single exponential in the size 
of the KB and double exponential in the size of the query, which gives a 
tight upper bound for the problem. This result shows that deciding 
conjunctive query entailment is strictly harder than instance checking 
for \SHIQ. We further prove co-\NPclass-completeness for data 
complexity. Interestingly, this shows that regarding data complexity 
deciding UCQ entailment is (at least theoretically) not harder than 
instance checking for \SHIQ, which was also a previously open question. 

It will be part of our future work to extend this procedure to \SHOIQ, 
which is the DL underlying OWL DL. We will also attempt to find more 
implementable algorithms for query answering in \SHIQ. Carrying out the 
query rewriting steps in a more goal directed way will be crucial to 
achieving this.

\section*{Acknowledgments}
This work was supported by the EU funded IST-2005-7603 FET Project 
Thinking Ontologies (TONES). Birte Glimm was supported by an EPSRC 
studentship.

\appendix

\clearpage 

\section{Complete Proofs}\label{sect:proofs}

\begin{lemmaapp}[\ref{lem:canonicalcountermodels}]
  Let \KB be a \SHIQ knowledge base and $\Q = \Q_1 \vee \ldots \vee 
  \Q_n$ a union of conjunctive queries, then $\KB \not\models \Q$ iff 
  there exists a canonical model \I of \KB such that $\I \not\models 
  \Q$.
\end{lemmaapp}
  
\begin{proof}[Proof of Lemma~\ref{lem:canonicalcountermodels}]
  The ``if'' direction is trivial.
  
  For the ``only if'' direction, 
  since an inconsistent knowledge base entails every query, we can 
  assume that \KB is consistent. Hence, there is an interpretation $\I' 
  = \inter[\I']$ such that $\I' \models \KB$ and $\I' \not\models \Q$. 
  From $\I'$, we construct a canonical model \I for \KB and its forest 
  base \Jmc as follows: we define the set $P \subseteq 
  {(\Delta^{\I'})^*}$ of \emph{paths} to be the smallest set such that
  \begin{itemize}
    \item 
      for all $a \in \indA, a^{\I'}$ is a path;
    \item 
      $d_1 \cdots d_n \cdot d$ is a path, if  
      \begin{itemize}
        \item $d_1 \cdots d_n$ is a path, 
        \item $(d_n, d) \in r^{\I'}$ for some role $r$, 
        \item if there is an $a \in \indA$ such that $d = a^{\I'}$, then 
              $n > 2$. 
      \end{itemize} 
  \end{itemize}
  For a path $p = d_1 \cdots d_n$, the \emph{length} $\len{p}$ of $p$ is 
  $n$. Now fix a set $S \subseteq \indA \times \Nbbm^*$ and a bijection 
  $f \colon S \to P$ such that  
  \begin{enumerate}
    \renewcommand\theenumi{(\roman{enumi})}
    \renewcommand\labelenumi{\theenumi}
    \item \label{it:root}
      $\indA \times \{\varepsilon\} \subseteq S$, 
    \item 
      for each $a \in \indA, \{w \mid (a, w) \in S\}$ is a tree,
    \item 
      $f((a, \varepsilon)) = a^{\I'}$,
    \item\label{it:nonroot}
      if $(a, w), (a, w') \in S$ with $w'$ a successor of $w$, then 
      $f((a, w')) = f((a, w)) \cdot d$ for some $d \in 
      \Delta^{\I'}$.
  \end{enumerate}
  For all $(a, w) \in S$, set $\tail{(a, w)} := d_n$ if $f((a, w)) = d_1 
  \cdots d_n$. Now, define a forest base \Jmc = \inter[\Jmc] for \KB as 
  follows:
  \begin{enumerate}
    \renewcommand\theenumi{(\alph{enumi})}
    \renewcommand\labelenumi{\theenumi}
    \renewcommand\theenumii{(\Roman{enumii})}
    \renewcommand\labelenumii{\theenumii}
    \item \label{it:dom} 
      $\Delta^\Jmc := S$;
    \item \label{it:a} 
      for each $a \in \indA$, $a^\Jmc := (a, \varepsilon) \in S$;
    \item \label{it:b}
      for each $b \in \NI \setminus \indA$, $b^\Jmc = a^\Jmc$ for 
      some fixed $a \in \indA$;
    \item \label{it:C}
      for each $C \in \NC, (a, w) \in C^\Jmc$ if $(a, w) \in S$ and 
      $\tail{(a, w)} \in C^{\I'}$;
    \item \label{it:r}
      For all roles $r$, $((a, w), (b, w')) \in r^\Jmc$ if either 
      \begin{enumerate}
        \item
          $w = w' = \varepsilon$ and $(a^{\I'}, b^{\I'}) \in r^{\I'}$ or 
        \item
          $a = b$, $w'$ is a neighbor of $w$ and $(\tail{(a, w)}, 
          \tail{(b, w')}) \in r^{\I'}$.
      \end{enumerate}
  \end{enumerate}
  
  It is clear that \Jmc is a forest base for \KB due to the definition 
  of $S$ and the construction of \Jmc from $S$.
  
  Let \I = \inter be an interpretation that is identical to \Jmc except 
  that, for all non-simple roles $r$, we set
  $$
     \Int{r} = r^\Jmc \cup \bigcup_{s \ssssR r, \; s \in \transR} 
     (s^\Jmc)^+
  $$
  It is tedious but not too hard to verify that $\I \models \KB$ and 
  that \Jmc is a forest base for \I. Hence \I is a canonical model for 
  \KB. 
  
  Therefore, we only have to show that $\I \not\models \Q$. Assume to 
  the contrary that $\I \models \Q$. Then there is some $\pi$ and $i$ 
  with $1 \leq i \leq n$ such that $\I \models^\pi \Q_i$. We now define 
  a mapping $\pi' \colon \terms{\Q_i} \to \Delta^{\I'}$ by setting 
  $\pi'(t) := \tail{\pi(t)}$ for all $t \in \terms{\Q_i}$. It is not 
  difficult to check that $\I' \models^{\pi'} \Q_i$ and hence $\I' 
  \models^{\pi'} \Q$, which is a contradiction.
\end{proof}

\begin{lemmaapp}[\ref{lem:collapsing}]
  Let \I be a model for \KB. 
  \begin{enumerate}
    \item
      If $\I \models \Q$, then there is a collapsing $\Q_{co}$ of \Q 
      such that $\I \models^{\pi_{co}} \Q_{co}$ for $\pi_{co}$ an 
      injection modulo \sapprox. 
    \item
      If $\I \models^{\pi_{co}} \Q_{co}$ for a collapsing $\Q_{co}$ of 
      \Q, then $\I \models \Q$.
  \end{enumerate}
\end{lemmaapp}

\begin{proof}[Proof of Lemma~\ref{lem:collapsing}]
  For (\ref{it:querycollapsing}), let $\pi$ be such that $\I \models^\pi 
  \Q$, let $\Q_{co}$ be the collapsing of \Q that is obtained by adding 
  an atom $t \approx t'$ for all terms $t, t' \in \terms{\Q}$ for which 
  $\pi(t) = \pi(t')$. By definition of the semantics, $\I \models^\pi 
  \Q_{co}$ and $\pi$ is an injection modulo \sapprox. 
  
  Condition (\ref{it:collapsingquery}) trivially holds since $\Q 
  \subseteq \Q_{co}$ and hence $\I \models^{\pi_{co}} \Q$.
\end{proof}

\begin{lemmaapp}[\ref{lem:splitrewriting}]
  Let \I be a model for \KB. 
  \begin{enumerate}
    \item
      If \I is canonical and $\I \models^\pi \Q$, then there is a pair 
      $(\Q_{sr}, \roots) \in \srK{\Q}$ and a split match $\pi_{sr}$ such 
      that $\I \models^{\pi_{sr}} \Q_{sr}$, \roots is the induced root 
      splitting of $\pi_{sr}$, and $\pi_{sr}$ is an injection modulo 
      $\sapprox$.
    \item
      If $(\Q_{sr}, \roots) \in \srK{\Q}$ and $\I \models^{\pi_{sr}} 
      \Q_{sr}$ for some match $\pi_{sr}$, then $\I \models \Q$.
  \end{enumerate}
\end{lemmaapp}

\begin{proof}[Proof of Lemma~\ref{lem:splitrewriting}]
  The proof of the second claim is relatively straightforward: since 
  $(\Q_{sr}, \roots) \in \srK{\Q}$, there is a collapsing $\Q_{co}$ of 
  \Q such that $\Q_{sr}$ is a split rewriting of $\Q_{co}$. Since all 
  roles replaced in a split rewriting are non-simple and $\I \models 
  \Q_{sr}$ by assumption, we have that $\I \models \Q_{co}$. By 
  Lemma~\ref{lem:collapsing} (2), we then have that $\I \models \Q$ as 
  required.
  
  We go through the proof of the first claim in more detail: let 
  $\Q_{co}$ be in $\co{\Q}$ such that $\I \models^{\pi_{co}} \Q_{co}$ 
  for a match $\pi_{co}$ that is injective modulo \sapprox. Such a 
  collapsing $\Q_{co}$ and match $\pi_{co}$ exist due to 
  Lemma~\ref{lem:collapsing}. If $\pi_{co}$ is a split match w.r.t.\ \Q 
  and \I already, we are done, since a split match induces a root 
  splitting \roots and $(\Q_{co}, \roots)$ is trivially in $\srK{\Q}$. 
  If $\pi_{co}$ is not a split match, there are at least two terms $t, 
  t'$ with $r(t, t') \din \Q_{co}$ such that $\pi_{co}(t) = (a, w), 
  \pi_{co}(t') = (a', w'), a \neq a'$, and $w \neq \varepsilon$ or $w' 
  \neq \varepsilon$. We distinguish two cases:
  \begin{enumerate}
    \item 
      Both $t$ and $t'$ are not mapped to roots, i.e., $w \neq 
      \varepsilon$ and $w' \neq \varepsilon$. Since $\I 
      \models^{\pi_{co}} r(t, t')$, we have that $(\pi_{co}(t), 
      \pi_{co}(t')) \in \Int{r}$. Since \I is a canonical model for \KB, 
      there must be a role $s$ with $s \sssR r$ and $s \in \transR$ such 
      that 
      $$
      \{(\pi_{co}(t), (a, \varepsilon)), \; ((a, \varepsilon), (a', 
      \varepsilon)), \; ((a', \varepsilon), \pi_{co}(t'))\} \subseteq 
      \Int{s}.
      $$ 
      If there is some $\hat{t} \in \terms{\Q_{co}}$ such that 
      $\pi_{co}(\hat{t}) = (a, \varepsilon)$, then let $u = \hat{t}$, 
      otherwise let $u$ be a fresh variable. Similarly, if there is some 
      $\hat{t'} \in \terms{\Q_{co}}$ such that $\pi_{co}(\hat{t'}) = 
      (a', \varepsilon)$, then let $u' = \hat{t'}$, otherwise let $u'$ 
      be a fresh variable. Hence, we can define a split rewriting 
      $\Q_{sr}$ of $\Q_{co}$ by replacing $r(t, t')$ with $s(t, 
      u)$, $s(u, u')$, and $s(u', t')$. We then define a new 
      mapping $\pi_{sr}$ that agrees with $\pi_{co}$ on all terms 
      that occur in $\Q_{co}$ and that maps $u$ to $(a, 
      \varepsilon)$ and $u'$ to $(a', \varepsilon)$. 
    \item
      Either $t$ or $t'$ is mapped to a root. W.l.o.g., let this be $t$, 
      i.e., $\pi(t) = (a, \varepsilon)$. We can use the same arguments 
      as above: since $\I \models^{\pi_{co}} r(t, t')$, we have that 
      $(\pi(t), \pi(t')) \in \Int{r}$ and, since \I is a canonical model 
      for \KB, there must be a role $s$ with $s \sssR r$ and $s \in 
      \transR$ such that $\{(\pi(t), (a', \varepsilon)), \; ((a', 
      \varepsilon), \pi(t'))\} \subseteq \Int{s}$. If there is some 
      $\hat{t} \in \terms{\Q_{co}}$ such that $\pi_{co}(\hat{t}) = (a', 
      \varepsilon)$, then let $u = \hat{t}$, otherwise let $u$ be a 
      fresh variable. We then define a split rewriting $\Q_{sr}$ of 
      $\Q_{co}$ by replacing $r(t, t')$ with $s(t, u), s(u, t')$and a 
      mapping $\pi_{sr}$ that agrees with $\pi_{co}$ on all terms that 
      occur in $\Q_{co}$ and that maps $u$ to $(a', \varepsilon)$. 
  \end{enumerate}
  It immediately follows that $\I \models^{\pi_{sr}} \Q_{sr}$. 
  We can proceed as described above for each role atom $r(t, t')$ for 
  which $\pi(t) = (a, w)$ and $\pi(t') = (a', w')$ with $a \neq a'$ and 
  $w \neq \varepsilon$ or $w' \neq \varepsilon$. This will result in a 
  split rewriting $q_{sr}$ and a split match $\pi_{sr}$ such that $\I 
  \models^{\pi_{sr}} \Q_{sr}$. Furthermore, $\pi_{sr}$ is injective 
  modulo \sapprox since we only introduce new variables, when the 
  variable is mapped to an element that is not yet in the range of the 
  match. Since $\pi_{sr}$ is a split match, it induces a root splitting 
  \roots and, hence, $(\Q_{sr}, \roots) \in \srK{\Q}$ as required.
\end{proof}

\begin{lemmaapp}[\ref{lem:looprewriting}]
  Let \I be a model of \KB. 
  \begin{enumerate}
    \item
      If \I is canonical and $\I \models \Q$, then there is a pair 
      $(\Q_{\ell r}, \roots) \in \lrK{\Q}$ and a mapping $\pi_{\ell r}$ 
      such that $\I \models^{\pi_{\ell r}} \Q_{\ell r}$, $\pi_{\ell r}$ 
      is an injection modulo $\sapprox$, \roots is the root splitting 
      induced by $\pi_{\ell r}$ and, for each $r(t, t) \din \Q_{\ell 
      r}$, $t \in \roots$.
    \item 
      If $(\Q_{\ell r}, \roots) \in \lrK{\Q}$ and $\I 
      \models^{\pi_{\ell r}} \Q_{\ell r}$ for some match $\pi_{\ell r}$, 
      then $\I \models \Q$.
  \end{enumerate}
\end{lemmaapp}

\begin{proof}[Proof of Lemma~\ref{lem:looprewriting}]
  The proof of (2) is analogous to the one given in
  Lemma~\ref{lem:splitrewriting} since, by definition of loop 
  rewritings, all roles replaced in a loop rewriting are again 
  non-simple. 
  
  For (1), let $(\Q_{sr}, \roots) \in \srK{\Q}$ be such that $\I 
  \models^{\pi_{sr}} \Q_{sr}$, $\pi_{sr}$ is a split match, and \roots 
  is the root splitting induced by $\pi_{sr}$. Such a split rewriting 
  $\Q_{sr}$ and match $\pi_{sr}$ exist due to 
  Lemma~\ref{lem:splitrewriting} and the canonicity of \I. 
  
  Let $r(t, t) \din \Q_{sr}$ for $t \notin \roots$. Since \roots is the 
  root splitting induced by $\pi_{sr}$ and since $t \notin \roots$, 
  $\pi_{sr}(t) = (a, w)$ for some $a \in \indA$ and $w \neq 
  \varepsilon$. Now, let \Jmc be a forest base for \I. We show that 
  there exists a neighbor $d$ of $\pi_{sr}(t)$ and a role $s \in 
  \transR$ such that $s \sssR r$ and $(\pi_{sr}(t), d) \in   \Int{s} 
  \cap \Int{\inv{s}}$. Since $\I \models^{\pi_{sr}} \Q_{sr}$, we have 
  $(\pi_{sr}(t), \pi_{sr}(t)) \in \Int{r}$. Since \Jmc is a forest base 
  and since $w \neq \varepsilon$, we have $(\pi_{sr}(t), \pi_{sr}(t)) 
  \notin r^\Jmc$. It follows that there is a sequence $d_1, \dots, d_n 
  \in \dom$ and a role $s \in \transR$ such that $s \sssR r$, $d_1 = 
  \pi_{sr}(t) = d_n$, and $(d_i, d_{i+1}) \in s^\Jmc$ for $1 \leq i < n$ 
  and $d_i \neq d_1$ for each $i$ with $1 < i < n$. Then it is not hard 
  to see that, because $\{w' \mid (a, w') \in \dom\}$ is a tree and $w 
  \neq \varepsilon$, we have $d_2 = d_{n-1}$. Since $(d_1, d_2) \in 
  s^\Jmc$ and $(d_{n-1}, d_n) \in s^\Jmc$ with $d_{n-1} = d_2$ and $d_n 
  = d_1$, the role $s$ and the element $d = d_2$ is as required. For 
  each $r(t, t) \din \Q_{sr}$ with $t \notin \roots$, select an element 
  $d_{r, t}$ and a role $s_{r, t}$ as described above. Now let $\Q_{\ell 
  r}$ be obtained from $\Q_{sr}$ by doing the following for each $r(t, 
  t) \din \Q_{sr}$ with $t \notin \roots$:
  \begin{itemize}
    \item 
      if $d_{r, t} = \pi_{sr}(t')$ for some $t' \in \terms{\Q_{sr}}$, 
      then replace $r(t, t)$ with $s_{r, t}(t, t')$ and $s_{r, t}(t', 
      t)$;
    \item 
      otherwise, introduce a new variable $v_{r, t} \in \NV$ and replace 
      $r(t, t)$ with $s_{r, t}(t, v_{r, t})$ and $s_{r, t}(v_{r, t}, 
      t)$.
  \end{itemize}
  Let $\pi_{\ell r}$ be obtained from $\pi_{sr}$ by extending it with 
  $\pi_{\ell r}(v_{r, t}) = d_{r, t}$ for each newly introduced variable 
  $v_{r, t}$. By definition of $q_{\ell r}$ and $\pi_{\ell r}$, 
  $q_{\ell r}$ is connected, $\pi_{\ell r}$ is injective modulo 
  \sapprox, and $\I \models^{\pi_{\ell r}} \Q_{\ell r}$. 
\end{proof}

\begin{lemmaapp}[\ref{lem:forestrewriting}]
  Let \I be a model of \KB. 
  \begin{enumerate}
    \item
      If \I is canonical and $\I \models \Q$, then there is a pair 
      $(\Q_{fr}, \roots) \in \frK{\Q}$ such that $\I \models^{\pi_{fr}} 
      \Q_{fr}$ for a forest match $\pi_{fr}$, \roots is the induced root 
      splitting of $\pi_{fr}$, and $\pi_{fr}$ is an injection modulo 
      $\sapprox$.
    \item 
      If $(\Q_{fr}, \roots) \in \frK{\Q}$ and $\I \models^{\pi_{fr}} 
      \Q_{fr}$ for some match $\pi_{fr}$, then $\I \models \Q$.
  \end{enumerate}
\end{lemmaapp}

\begin{proof}[Proof of Lemma~\ref{lem:forestrewriting}]
  The proof of (2) is again analogous to the one given in 
  Lemma~\ref{lem:splitrewriting}. For (1), let $(\Q_{\ell r}, \roots) 
  \in \lrK{\Q}$ be such that $\I \models^{\pi_{\ell r}} \Q_{\ell r}$, 
  \roots is the root splitting induced by $\pi_{\ell r}$, $\pi_{\ell r}$ 
  is injective modulo $\sapprox$ and, for each $r(t, t) \din \Q_{\ell 
  r}$,  $t \in \roots$. Such a loop rewriting and match $\pi_{\ell r}$ 
  exist due to Lemma~\ref{lem:looprewriting} and the canonicity of \I. 
  By definition, \roots is a root splitting w.r.t.\ $\Q_{\ell r}$ and 
  \KB. 
  
  For $w, w' \in \Nbbm^*$, the \emph{longest common prefix} (LCP) of $w, 
  w'$ is the longest $w^* \in \Nbbm^*$ such that $w^*$ is prefix of both 
  $w$ and $w'$. For the match $\pi_{\ell r}$ we now define the set $D$ 
  as follows:
  $$
  \begin{array}{l l}
    D := \range{\pi_{\ell r}} \cup \{(a, w) \in \dom \mid & w \mbox{ is 
          the LCP of some } w, w' \\
    & \mbox{with } (a, w'), (a, w'') \in \range{\pi_{\ell r}}\}.
  \end{array}
  $$
  Let $V \subseteq \NV \setminus \vars{\Q_{\ell r}}$ be such that, for 
  each $d \in D \setminus \range{\pi_{\ell r}}$, there is a unique $v_d 
  \in V$. We now define a mapping $\pi_{fr}$ as $\pi_{\ell r} \cup \{v_d 
  \in V \mapsto d\}$. By definition of $V$ and $v_d$, $\pi_{fr}$ is a 
  split match as well. The set $V \cup \vars{\Q_{\ell r}}$ will be the 
  set of variables for the new query $\Q_{fr}$. Note that 
  $\range{\pi_{fr}} = D$.
  \begin{description}
    \item[Fact (a)]\label{it:facta} 
      if $(a, w), (a, w') \in \range{\pi_{fr}}$, then $(a, w'') \in 
      \range{\pi_{fr}}$, where $w''$ is the LCP of $w$ and $w'$;
    \item[Fact (b)]\label{it:factb} 
      $\card{V} \leq \card{\vars{\Q_{\ell r}}}$ (Because, in the worst 
      case, all $(a, w)$ in $\range{\pi_{\ell r}}$ are ``incomparable'' 
      and can thus be seen as leaves of a binarily branching tree. Now, 
      a tree that has $n$ leaves and is at least binarily branching at 
      every non-leaf has at most $n$ inner nodes, and thus $\card{V} 
      \leq \card{\vars{q_{\ell r}}}$.
  \end{description}  
  For a pair of individuals $d, d' \in \dom$, the \emph{path} from $d$ 
  to $d'$ is the (unique) shortest sequence of elements $d_1, \dots, d_n 
  \in \dom$ such that $d_1 = d$, $d_n = d'$, and $d_{i+1}$ is a neighbor 
  of $d_i$ for all $1 \leq i < n$. The \emph{length of a path} is 
  the number of elements in it, i.e., the path $d_1, \ldots, d_n$ is of 
  length $n$. The \emph{relevant path} $d_1', \dots, d_\ell'$ from $d$ 
  to $d'$ is the sub-sequence of $d_1, \dots, d_n$ that is obtained by 
  dropping all elements $d_i \notin D$. 
  \\[2mm]
  {\bf Claim 1}. Let $r(t, t') \din \sq{\Q_{\ell r}, t_r}$ for some $t_r 
  \in \roots$ and let $d_1', \ldots, d_\ell'$ be the relevant path from 
  $d = d_1' = \pi_{\ell r}(t)$ to $d' = d_\ell' = \pi_{\ell r}(t')$. If 
  $\ell > 2$, there is a role $s \in \transR$ such that $s \sssR r$ and 
  $(d_i', d_{i + 1}') \in \Int{s}$ for all $1 \leq i < \ell$. 
  \\[2mm]
  \emph{Proof}. Let $d_1, \dots, d_n$ be the path and $d_1', \ldots, 
  d_\ell'$ the relevant path from $\pi_{\ell r}(t)$ to $\pi_{\ell 
  r}(t')$. Then $\ell > 2$ implies $n > 2$. We have to show that there 
  is a role $s$ as in the claim. Let \Jmc be a forest base for \I. Since 
  $\I \models^{\pi_{\ell r}} \Q_{\ell r}$, $n > 2$ implies $(\pi_{\ell 
  r}(t), \pi_{\ell r}(t')) \in \Int{r} \setminus r^\Jmc$. Since \I is 
  based on \Jmc, it follows that there is an $s \in \transR$ such that 
  $s \sssR r$, and $(d_i, d_{i + 1}) \in s^\Jmc$ for all $1 \leq i < n$. 
  By construction of \I from \Jmc, it follows that $(d_i', d_{i + 1}') 
  \in \Int{s}$ for all $1 \leq i < \ell$, which finishes the proof of 
  the claim.
  
  Now let $\Q_{fr}$ be obtained from $\Q_{\ell r}$ as follows: for each 
  role atom $r(t, t) \din \sq{\Q_{\ell r}, t_r}$ with $t_r \in \roots$, 
  if the length of the relevant path $d_1', \ldots, d_\ell'$ from $d = 
  d_1' = \pi_{\ell r}(t)$ to $d' = d_\ell' = \pi_{\ell r}(t')$ is 
  greater than $2$, then select a role $s$ and variables $t_j \in D$ 
  such that $\pi_{fr}(t_j) = d_j'$ as in Claim~1 and replace the 
  atom $r(t, t')$ with $s(t_1, t_2), \ldots, s(t_{\ell - 1}, t_\ell)$, 
  where $t = t_1$, $t' = t_\ell$. Please note that these $t_j$ can be 
  chosen in a ``don't care'' non-deterministic way since $\pi_{fr}$ is 
  injective modulo \sapprox, i.e., if $\pi_{fr}(t_j) = d_j = 
  \pi_{fr}(t_j')$, then $t_j \sapprox t_j'$ and we can pick any of 
  these. 
  
  We now have to show that 
  \begin{enumerate}
    \renewcommand{\theenumi}{(\roman{enumi})}
    \renewcommand\labelenumi{\theenumi}
    \item\label{it:models}
      $\I \models^{\pi_{fr}} \Q_{fr}$, and
    \item\label{it:forestmatch}
      $\pi_{fr}$ is a forest match.
  \end{enumerate}
  For \ref{it:models}, let $r(t, t') \din \Q_{\ell r} \setminus \Q_{fr}$ 
  and let $s(t_1, t_2), \ldots, s(t_{\ell - 1}, t_\ell)$ be the atoms 
  that replaced $r(t, t')$. Since $\I \models^{\pi_{\ell r}} \Q_{\ell 
  r}$, $\I \models^{\pi_{\ell r}} r(t, t')$ and $(\pi_{\ell r}(t), 
  \pi_{\ell r}(t')) \in \Int{r}$. Since $r(t, t')$ was replaced in 
  $\Q_{fr}$, the length of the relevant path from $\pi_{\ell r}(t)$ to 
  $\pi_{\ell r}(t')$ is greater than $2$. Hence, it must be the case 
  that $(\pi_{\ell r}(t), \pi_{\ell r}(t')) \in \Int{r} \setminus 
  r^\Jmc$. Let $d_1, \ldots, d_n$ with $d_1 = \pi_{\ell r}(t)$ and $d_n 
  = \pi_{\ell r}(t')$ be the path from $\pi_{\ell r}(t)$ to $\pi_{\ell 
  r}(t')$ and $d_1', \dots, d_\ell'$ the relevant path from $\pi_{\ell 
  r}(t)$ to $\pi_{\ell r}(t')$. By construction of \I from \Jmc, this 
  means that there is a role $s \in \transR$ such that $s \sssR r$ and 
  $(d_i, d_{i+1}) \in s^\Jmc$ for all $1 \leq i < n$. Again by 
  construction of \I, this means $(d_i', d_{i+1}') \in \Int{s}$ for $1 
  \leq i < \ell$ as required. Hence $\I \models^{\pi_{fr}} s(t_i, 
  t_{i+1})$ for each $i$ with $\ \leq i < \ell$ by definition of 
  $\pi_{fr}$. 
  
  For \ref{it:forestmatch}: the mapping $\pi_{fr}$ differs from 
  $\pi_{\ell r}$ only for the newly introduced variables. Furthermore, 
  we only introduced new role atoms within a sub-query $\sq{\Q_{\ell r}, 
  t_r}$ and $\pi_{\ell r}$ is a split match by assumption. Hence, 
  $\pi_{fr}$ is trivially a split match and we only have to show that 
  $\pi_{fr}$ is a forest match. Since $\pi_{fr}$ is a split match, we 
  can do this ``tree by tree''. 
  
  For each $a \in \indA$, let $T_a := \{w \mid (a, w) \in 
  \range{\pi_{fr}}\}$. We need to construct a mapping $f$ as specified 
  in Definition~\ref{def:matches}, and we start with its root $t_r$. If 
  $T_a \neq \emptyset$, let $t_r \in \terms{\Q}$ be the unique term such 
  that $\pi_{fr}(t_r) = (a, w_r)$ and there is no $t \in \terms{\Q}$ 
  such that $\pi_{fr}(t) = (a, w)$ and $w$ is a proper prefix of $w_r$. 
  Such a term exists since $\pi_{fr}$ is a split match and it is unique 
  due to Fact (a) above. Define a \emph{trace} to be a sequence $\bar{w} 
  = w_1 \cdots w_n \in T_a^+$ such that
  \begin{itemize}
    \item 
      $w_1 = w_r$;
    \item 
      for all $1 \leq i < n$, $w_i$ is the longest proper prefix of 
      $w_{i + 1}$.
  \end{itemize}
  Since \I is canonical, each $w_i \in T_a$ is in $\Nbbm$. It is not 
  hard to see that $T = \{\bar{w} \mid \bar{w} \mbox{ is a trace}\} \cup 
  \{\varepsilon\}$ is a tree. For a trace $\bar{w} = w_1 \cdots w_n$, 
  let $\tail{\bar{w}} = w_n$. Define a mapping $f$ that maps each term 
  $t$ with $\pi_{fr}(t) = (a, w) \in T_a$ to the unique trace 
  $\bar{w}_{t}$ such that $w = \tail{\bar{w}_{t}}$. Let $r(t, t') \in 
  \Q_{fr}$ such that $\pi_{fr}(t), \pi_{fr}(t') \in T_a$. By 
  construction of $\Q_{fr}$, this implies that the length of the 
  relevant path from $\pi_{fr}(t)$ to $\pi_{fr}(t')$ is exactly $2$. 
  Thus, $f(t)$ and $f(t')$ are neighbors in $T$ and, hence, $\pi_{fr}$ 
  is a forest match as required. 
\end{proof}

\begin{theoremapp}[\ref{thm:union}]
  Let \KB be a \SHIQ knowledge base, \Q a Boolean conjunctive query, and 
  $\{\Q_1, \ldots, \Q_\ell\} = \trees{\Q} \cup \groundings{\Q}$. Then 
  $\KB \models \Q$ iff $\KB \models \Q_1 \vee \ldots \vee \Q_\ell$.
\end{theoremapp}

\begin{proof}[Proof of Theorem~\ref{thm:union}]
  For the ``if'' direction: let us assume that $\KB \models \Q_1 \vee 
  \ldots \vee \Q_\ell$. Hence, for each model \I of \KB, there is a 
  query $\Q_i$ with $1 \leq i \leq \ell$ such that $\I \models \Q_i$. We 
  distinguish two cases: (i) $\Q_i \in \trees{\Q}$ and (ii) $\Q_i \in 
  \groundings{\Q}$. 
  
  For (i): $\Q_i$ is of the form $C(v)$ where $C$ is the query concept 
  for some query $\Q_{fr}$ w.r.t.\ $v \in \vars{\Q_{fr}}$ and $(\Q_{fr}, 
  \emptyset) \in \frK{\Q}$. Hence $\I \models^\pi \Q_i$ for some match 
  $\pi$, and thus $\I \models^\pi C(v)$. Let $d \in \dom$ with $d = 
  \pi(v) \in \Int{C}$. By Lemma~\ref{lem:conceptsandqueries}, we then 
  have that $\I \models \Q_{fr}$ and, by 
  Lemma~\ref{lem:forestrewriting}, we then have that $\I \models \Q$ as 
  required. 
  
  For (ii): since $\Q_i \in \groundings{\Q}$, there is some pair 
  $(\Q_{fr}, \roots) \in \frK{\Q}$ such that $\Q_i = \grounding{\Q_{fr}, 
  \roots,  \tau}$. We show that $\I \models^{\pi_{fr}} \Q_{fr}$ for some 
  match $\pi_{fr}$. Since $\I \models \Q_1$, there is a match $\pi_i$ 
  such that $\I \models^{\pi_i} \Q_i$. We now construct the match 
  $\pi_{fr}$. For each $t \in \roots$, $\Q_i$ contains a concept atom 
  $C(\tau(t))$ where $C = \con{\sq{\Q_{fr}, t}, t}$ is the query concept 
  of \sq{\Q_{fr}, t} w.r.t.\ $t$. Since $\I \models^{\pi_i} C(\tau(t))$ 
  and by Lemma~\ref{lem:conceptsandqueries}, there is a match $\pi_t$ 
  such that $\I \models^{\pi_t} \sq{\Q_{fr}, t}$. We now define 
  $\pi_{fr}$ as the union of $\pi_t$, for each $t \in \roots$. Please 
  note that $\pi_{fr}(t) = \pi_i(\tau(t))$. Since $\inds{\Q_{fr}} 
  \subseteq \roots$ and $\tau$ is such that, for each $a \in 
  \inds{\Q_{fr}}$, $\tau(a) = a$ and $\tau(t) = \tau(t')$ iff $t 
  \sapprox t'$, it follows that $\I \models^{\pi_{fr}} at$ for each atom 
  $at \din \Q_{fr}$ such that $at$ contains only terms from the root 
  choice \roots and hence $\I \models^{\pi_{fr}} \Q_{fr}$ as required. 
  
  For the ``only if'' direction we have to show that, if $\KB \models 
  \Q$, then $\KB \models \Q_1 \vee \ldots \vee \Q_\ell$, so let us 
  assume that $\KB \models \Q$. By 
  Lemma~\ref{lem:canonicalcountermodels} in its negated form we have 
  that $\KB \models \Q$ iff all canonical models \I of \KB are such that 
  $\I \models \Q$. Hence, we can restrict our attention to the canonical 
  models of \KB. By Lemma~\ref{lem:forestrewriting}, $\I \models \KB$ 
  and $\I \models \Q$ implies that there is a pair $(\Q_{fr}, \roots) 
  \in \frK{\Q}$ such that $\I \models^{\pi_{fr}} \Q_{fr}$ for a forest 
  match $\pi_{fr}$, \roots is the induced root splitting of $\pi_{fr}$, 
  and $\pi_{fr}$ is an injection modulo \sapprox. We again distinguish 
  two cases: 
  \begin{description}
    \item[(i)] 
      $\roots = \emptyset$, i.e., the root splitting is empty and 
      $\pi_{fr}$ is a tree match, and 
    \item[(ii)]
      $\roots \neq \emptyset$, i.e., the root splitting is non-empty and 
      $\pi_{fr}$ is a forest match but not a tree match. 
  \end{description}
  For (i): since $(\Q_{fr}, \emptyset) \in \frK{\Q}$, there is some 
  $v \in \terms{\Q_{fr}}$ such that $C = \con{\Q_{fr}, v}$ and $\Q_i = 
  C(v)$. By Lemma~\ref{lem:conceptsandqueries} and, since $\I \models 
  \Q_{fr}$, there is an element $d \in \dom$ such that $d \in \Int{C}$. 
  Hence $\I \models^\pi C(v)$ with $\pi : v \mapsto d$ as required. 
  
  For (ii): since \roots is the root splitting induced by $\pi_{fr}$, 
  for each $t \in \roots$ there is some $a_t \in \indA$ such that 
  $\pi_{fr}(t) = (a_t, \varepsilon)$. We now define the mapping $\tau 
  \colon \roots \to \indA$ as follows: for each $t \in \roots$, $\tau(t) 
  = a_t$ iff $\pi_{fr}(t) = (a_t, \varepsilon)$. By definition of 
  $\grounding{\Q_{fr}, \roots, \tau}$, $\Q_i = \grounding{\Q_{fr}, 
  \roots, \tau} \in \groundings{\Q}$. Since $\I \models^{\pi_{fr}} 
  \Q_{fr}$, $\I \models \sq{\Q_{fr}, t}$ for each $t \in \roots$. Since 
  $\Q_{fr}$ is forest-shaped, each $\sq{\Q_{fr}, t}$ is tree-shaped. 
  Then, by Lemma~\ref{lem:conceptsandqueries}, $\I \models \Q_i'$, where 
  $\Q_i'$ is the query obtained from $\Q_{fr}$ by replacing each 
  sub-query $\sq{\Q_{fr}, t}$ with $C(t)$ for $C$ the query concept of 
  $\sq{\Q_{fr}, t}$ w.r.t.\ $t$. By definition of $\tau$ from the forest 
  match $\pi_{fr}$, it is clear that $\I \models \grounding{\Q_{fr}, 
  \roots, \tau}$ as required. 
\end{proof}

\begin{lemmaapp}[\ref{lem:counting}]
  Let \Q be a Boolean conjunctive query, \KBDef a \SHIQ knowledge base, 
  $|\Q| := n$ and $|\KB| := m$. Then there is a polynomial $p$ such that
  \begin{enumerate}
    \item
      $\card{\co{\Q}} \leq 2^{p(n)}$ and, for each $\Q' \in \co{\Q}$, 
      $|\Q'| \leq p(n)$, 
    \item
      $\card{\srK{\Q}} \leq 2^{p(n) \cdot \log p(m)}$, and, for 
      each $\Q' \in \srK{\Q}$, $|\Q'| \leq p(n)$, 
    \item
      $\card{\lrK{\Q}} \leq 2^{p(n) \cdot \log p(m)}$, and, for 
      each $\Q' \in \lrK{\Q}$, $|\Q'| \leq p(n)$, 
    \item
      $\card{\frK{\Q}} \leq 2^{p(n) \cdot \log p(m)}$, and, for 
      each $\Q' \in \frK{\Q}$, $|\Q'| \leq p(n)$, 
    \item
      $\card{\trees{\Q}} \leq 2^{p(n) \cdot \log p(m)}$, and, for 
      each $\Q' \in \trees{\Q}$, $|\Q'| \leq p(n)$, and 
    \item
      $\card{\groundings{\Q}} \leq 2^{p(n) \cdot \log p(m)}$, and, 
      for each $\Q' \in \groundings{\Q}$, $|\Q'| \leq p(n)$.
  \end{enumerate}
\end{lemmaapp}

\begin{proof}[Proof of Lemma~\ref{lem:counting}]
  $\mbox{ }$
  \begin{enumerate}
    \item
      The set $\co{\Q}$ contains those queries obtained from \Q by 
      adding at most $n$ equality atoms to \Q. The number of collapsings 
      corresponds, therefore, to building all equivalence classes over 
      the terms in \Q by \sapprox. Hence, the cardinality of the set 
      \co{\Q} is at most exponential in $n$. Since we add at most one 
      equality atom for each pair of terms, the size of a query $\Q' \in 
      \co{\Q}$ is at most $n + n^2$, and $|\Q'|$ is, therefore, 
      polynomial in $n$. 
    \item 
      For each of the at most $n$ role atoms, we can choose to do 
      nothing, replace the atom with two atoms, or with three atoms. 
      For every replacement, we can choose to introduce a new variable 
      or re-use one of the existing variables. If we introduce a new 
      variable every time, the new query contains at most $3n$ terms. 
      Since \KB can contain at most $m$ non-simple roles that are a 
      sub-role of a role used in role atoms of \Q, we have at most $m$ 
      roles to choose from when replacing a role atom. Overall, this 
      gives us at most $1 + m(3n) + m(3n)(3n)$ choices for each of the 
      at most $n$ role atoms in a query and, therefore, the number of 
      split rewritings for each query $\Q' \in \co{\Q}$ is polynomial in 
      $m$ and exponential in $n$. In combination with the results from 
      (1), this also shows that the overall number of split rewritings 
      is polynomial in $m$ and exponential in $n$.

      Since we add at most two new role atoms for each of the existing 
      role atoms, the size of a query $\Q' \in \srK{\Q}$ is linear in 
      $n$.
    \item
      There are at most $n$ role atoms of the form $r(t, t)$ in a query 
      $\Q' \in \srK{\Q}$ that could give rise to a loop rewriting, at 
      most $m$ non-simple sub-roles of $r$ in \KB that can be used in 
      the loop rewriting, and we can introduce at most one new variable 
      for each role atom $r(t, t)$. Therefore, for each query in 
      $\srK{\Q}$, the number of loop rewritings is again polynomial in 
      $m$ and exponential in $n$. Combined with the results from (2), 
      this bound also holds for the cardinality of \lrK{\Q}. 
      
      In a loop rewriting, one role atom is replaced with two role 
      atoms, hence, the size of a query $\Q' \in \lrK{\Q}$ at most 
      doubles.
    \item
      We can use similar arguments as above in order to derive a bound 
      that is exponential in $n$ and polynomial in $m$ for the number of 
      forest rewritings in \frK{\Q}. 
      
      Since the number of role atoms that we can introduce in a forest 
      rewriting is polynomial in $n$, the size of each query $\Q' \in 
      \frK{\Q}$ is at most quadratic in $n$. 
    \item
      The cardinality of the set \trees{\Q} is clearly also polynomial 
      in $m$ and exponential in $n$ since each query in \frK{\Q} can 
      contribute at most one query to the set \trees{\Q}. It is not hard 
      to see that the size of a query $\Q' \in \trees{\Q}$ is polynomial 
      in $n$.
    \item 
      By (1)-(4) above, the number of terms in a root splitting is 
      polynomial in $n$ and there are at most $m$ individual names 
      occurring in \AB that can be used for the mapping $\tau$ 
      from terms to individual names. Hence the number of 
      different ground mappings $\tau$ is at most polynomial in $m$ and 
      exponential in $n$. The number of ground queries that a single 
      tuple $(\Q_{fr}, \roots) \in \frK{\Q}$ can contribute is, 
      therefore, also at most polynomial in $m$ and exponential in $n$. 
      Together with the bound on the number of forest rewritings from 
      (4), this shows that the cardinality of \groundings{\Q} is 
      polynomial in $m$ and exponential in $n$. Again it is not hard to 
      see that the size of each query $\Q' \in \groundings{\Q}$ is 
      polynomial in $n$. 
  \end{enumerate}
\end{proof}

\begin{lemmaapp}[\ref{lem:correctness}]
  Let \KB be a \SHIQ knowledge base and $\Q$ a union of connected 
  Boolean conjunctive queries. The algorithm from 
  Definition~\ref{def:algodeterministic} answers ``\KB entails \Q'' iff 
  $\KB \models \Q$ under the unique name assumption.
\end{lemmaapp}

\begin{proof}[Proof of Lemma~\ref{lem:correctness}]
  For the ``only if''-direction: let $\Q = \Q_1 \vee \ldots \vee 
  \Q_\ell$. We show the contrapositive and assume that $\KB 
  \not\models \Q$. We can assume that \KB is consistent 
  since an inconsistent knowledge base trivially entails every query. 
  Let \I be a model of \KB such that $\I \not\models \Q$. We show that 
  \I is also a model of some extended knowledge base $\KB_\Q = (\TB \cup 
  \TB_\Q, \RB, \AB \cup \AB_\Q)$. We first show that \I is a model of 
  $\TB_\Q$. To this end, let $\top \sqsubseteq \neg C$ in $\TB_\Q$. Then 
  $C(v) \in T$ and $C = \con{\Q_{fr}, v}$ for some pair $(\Q_{fr}, 
  \emptyset) \in \frK{\Q_1} \cup \ldots \cup \frK{\Q_\ell}$ and $v \in 
  \vars{\Q_{fr}}$. Let $i$ be such that $(\Q_{fr}, \emptyset) \in 
  \frK{\Q_i}$. Now $\Int{C} \neq \emptyset$ implies, by 
  Lemma~\ref{lem:conceptsandqueries}, that $\I \models \Q_{fr}$ and, by 
  Lemma~\ref{lem:forestrewriting}, $\I \models \Q_i$ and, hence, $\I 
  \models \Q$, contradicting our assumption. Thus $\I \models \top 
  \sqsubseteq \neg C$ and, thus, $\I \models \TB_\Q$. 
  
  Next, we define an extended ABox $\AB_\Q$ such that, for each $\Q' \in 
  G$,
  \begin{itemize}
    \item 
      if $C(a) \in \Q'$ and $\Int{a} \in \Int{\neg C}$, then $\neg C(a) 
      \in \AB_\Q$;
    \item 
      if $r(a, b) \in \Q'$ and $(\Int{a}, \Int{b}) \notin \Int{r}$, 
      then $\neg r(a, b) \in \AB_\Q$.
  \end{itemize}
  Now assume that we can have a query $\Q' = \grounding{\Q_{fr}, \roots, 
  \tau} \in \groundings{\Q_1} \cup \ldots \cup \groundings{\Q_\ell}$ 
  such that there is no atom $at \in \Q'$ with $\neg at \in \AB_\Q$. 
  Then trivially $\I \models \Q'$. Let $i$ be such that $(\Q_{fr}, 
  \roots) \in \frK{\Q_i}$. By Theorem~\ref{thm:union}, $\I \models \Q_i$ 
  and thus $\I \models \Q$, which is a contradiction. Hence $\KB_\Q$ is 
  an extended knowledge base and $\I \models \KB_\Q$ as required. 
  
%
  
  For the ``if''-direction, we assume that $\KB \models \Q$, but the 
  algorithm answers ``\KB does not entail \Q''. Hence there is an 
  extended knowledge base $\KB_\Q = (\TB \cup \TB_\Q, \RB, \AB \cup 
  \AB_\Q)$ that is consistent, i.e., there is a model \I such that $\I 
  \models \KB_\Q$. Since $\KB_\Q$ is an extension of \KB, $\I \models 
  \KB$. Moreover, we have that $\I \models \TB_\Q$ and hence, for each 
  $d \in \dom$, $d \in \Int{\neg C}$ for each $C(v) \in \trees{\Q_1} 
  \cup \ldots \cup \trees{\Q_\ell}$. By 
  Lemma~\ref{lem:conceptsandqueries}, we then have that $\I \not\models 
  \Q'$ for each $\Q' \in \trees{\Q_1} \cup \ldots \cup \trees{\Q_\ell}$ 
  and, by Lemma~\ref{lem:forestrewriting}, $\I \not\models \Q_i$ for 
  each $i$ with $1 \leq i \leq \ell$. 
  
  By definition of extended knowledge bases, $\AB_\Q$ contains an 
  assertion $\neg at$ for at least one atom $at$ in each query $\Q' = 
  \grounding{\Q_{fr}, \roots, \tau}$ from $\groundings{\Q_1} \cup \ldots 
  \cup \groundings{\Q_\ell}$. Hence $\I \not\models \Q'$ for each $\Q' 
  \in \groundings{\Q_1} \cup \ldots \cup \groundings{\Q_\ell}$. Then, by 
  Theorem~\ref{thm:union}, $\I \not\models \Q$, which contradicts our 
  assumption.
\end{proof}

\begin{lemmaapp}[\ref{lem:uparrow}]
  Let \RB be a role hierarchy, and $r_1, \ldots, r_n$ roles. For 
  every interpretation \I such that $\I \models \RB$, it holds that 
  $\Int{(\up{r_1 \sqcap \ldots \sqcap r_n})} = \Int{(r_1 \sqcap \ldots 
  \sqcap r_n)}$.
\end{lemmaapp}

\begin{proof}[Proof of Lemma~\ref{lem:uparrow}]
  The proof is a straightforward extension of Lemma~6.19 by 
  \citeA{Tobi01a}. By definition, $\up{r_1 \sqcap \ldots \sqcap r_n} = 
  \up{r_1} \sqcap \ldots \sqcap \up{r_n}$ and, by definition of the 
  semantics of role conjunctions, we have that $\Int{(\up{r_1} \sqcap 
  \ldots \sqcap \up{r_n})} = \Int{\up{r_1}} \cap \ldots \cap 
  \Int{\up{r_n}}$. If $s \sssR r$, then $\{s' \mid r \sssR s'\} 
  \subseteq \{s' \mid s \sssR s'\}$ and hence $\Int{\up{s}} \subseteq 
  \Int{\up{r}}$. If $\I \models \RB$, then $\Int{r} \subseteq \Int{s}$ 
  for every $s$ with $r \sssR s$. Hence, $\Int{\up{r}} = \Int{r}$ and 
  $\Int{(\up{r_1 \sqcap \ldots \sqcap r_n})} = \Int{(\up{r_1} \sqcap 
  \ldots \sqcap \up{r_n})} = \Int{\up{r_1}} \cap \ldots \cap 
  \Int{\up{r_n}} = \Int{r_1} \cap \ldots \cap \Int{r_n} = \Int{(r_1 
  \sqcap \ldots \sqcap r_n)}$ as required.
\end{proof}


\begin{lemmaapp}[\ref{lem:shiqrconsistency}]
  Given a \SHIQR knowledge base \KBDef where $m := |\KB|$ and the size 
  of the longest role conjunction is $n$, we can decide consistency of 
  \KB in deterministic time $2^{p(m) 2^{p(n)}}$ with $p$ a polynomial.
\end{lemmaapp}

\begin{proof}[Proof of Lemma~\ref{lem:shiqrconsistency}]
  We first translate \KB into an \ALCQIb knowledge base $\tr{\KB} = 
  (\tr{\TB}, \tr{\AB})$. Since the longest role conjunction is of size 
  $n$, the cardinality of each set $\mn{tc}(R, \RB)$ for a role 
  conjunction $R$ is bounded by $m^n$. Hence, the TBox $\tr{\TB}$ can 
  contain exponentially many axioms in $n$. It is not hard to check 
  that the size of each axiom is polynomial in $m$. Since deciding 
  whether an \ALCQIb KB is consistent is an \ExpTime-complete problem 
  (even with binary coding of numbers) \cite[Theorem 4.42]{Tobi01a}, 
  the   consistency of $\tr{\KB}$ can be checked in time $2^{p(m) 
  2^{p(n)}}$. 
\end{proof}

\begin{lemmaapp}[\ref{lem:combinedcomplexity}]
  Let \KBDef be a \SHIQ knowledge base with $m := |\KB|$ and \Q a union 
  of connected Boolean conjunctive queries with $n := |\Q|$. The 
  algorithm given in Definition~\ref{def:algodeterministic} decides 
  whether $\KB \models \Q$ under the unique name assumption in 
  deterministic time in $2^{p(m) 2^{p(n)}}$. 
\end{lemmaapp}

\begin{proof}[Proof of Lemma~\ref{lem:combinedcomplexity}]
  We first show that there is some polynomial $p$ such that we have to 
  check at most $2^{p(m) 2^{p(n)}}$ extended knowledge bases for 
  consistency and then that each consistency check can be done in time 
  $2^{p(m) 2^{p(n)}}$, which gives an upper bound of $2^{p(m) 2^{p(n)}}$ 
  on the time needed for deciding whether $\KB \models \Q$. 
  
  Let $\Q := \Q_1 \vee \ldots \vee \Q_\ell$. Clearly, we can use $n$ as 
  a bound for $\ell$, i.e., $\ell \leq n$. Moreover, the size of each 
  query $\Q_i$ with $1 \leq i \leq \ell$ is bounded by $n$. Together 
  with Lemma~\ref{lem:counting}, we get that $\card{T}$ and $\card{G}$ 
  are bounded by $2^{p(n) \cdot \log p(m)}$ for some polynomial $p$ and 
  it is clear that the sets can be computed in this time bound as well. 
  The size of each query $\Q' \in G$ w.r.t.\ an ABox \AB is polynomial 
  in $n$ and, when constructing $\AB_\Q$, we can add a subset of 
  (negated) atoms from each $\Q' \in G$ to $\AB_\Q$. Hence, there are at 
  most $2^{p(m) 2^{p(n)}}$ extended ABoxes $\AB_\Q$ and, therefore, 
  $2^{p(m) 2^{p(n)}}$ extended knowledge bases that have to be tested 
  for consistency. 
  
  Due to Lemma~\ref{lem:counting} (\ref{it:trees}), the size of each 
  query $\Q' \in T$ is polynomial in $n$. Computing a query concept 
  $C_{\Q'}$ of $\Q'$ w.r.t.\ some variable $v \in \vars{\Q'}$ can be 
  done in time polynomial in $n$. Thus the TBox $\TB_\Q$ can be computed 
  in time $2^{p(n) \cdot \log p(m)}$. The size of an extended ABox 
  $\AB_\Q$ is maximal if we add, for each of the $2^{p(n) \cdot \log 
  p(m)}$ ground queries in $G$, all atoms in their negated form. Since, 
  by Lemma~\ref{lem:counting} (\ref{it:groundings}), the size of these 
  queries is polynomial in $n$, the size of each extended ABox $\AB_\Q$ 
  is bounded by $2^{p(n) \cdot \log p(m)}$ and it is clear that we can 
  compute an extended ABox in this time bound as well. Hence, the size 
  of each extended KB $\KB_\Q = (\TB \cup \TB_\Q, \RB, \AB \cup \AB_\Q)$ 
  is bounded by $2^{p(n) \cdot \log p(m)}$. Since role conjunctions 
  occur only in $\TB_\Q$ or $\AB_\Q$, and the size of each concept in 
  $\TB_\Q$ and $\AB_\Q$ is polynomial in $n$, the length of the 
  longest role conjunction is also polynomial in $n$. 

  When translating an extended knowledge base into an \ALCQIb knowledge 
  base, the number of axioms resulting from each concept $C$ that occurs 
  in $\TB_\Q$ or $\AB_\Q$ can be exponential in $n$. Thus, the size of 
  each extended knowledge base is bounded by $2^{p(n) \cdot \log p(m)}$.

  Since deciding whether an \ALCQIb knowledge base is consistent is an 
  \ExpTime-complete problem (even with binary coding of numbers) 
  \cite[Theorem 4.42]{Tobi01a}, it can be checked in time $2^{p(m) 
  2^{p(n)}}$ if \KB is consistent or not.

  Since we have to check at most $2^{p(m) 2^{p(n)}}$ knowledge 
  bases for consistency, and each check can be done in time $2^{p(m) 
  2^{p(n)}}$, we obtain the desired upper bound of $2^{p(m) 2^{p(n)}}$ 
  for deciding whether $\KB \models \Q$. 
\end{proof}

\begin{lemmaapp}[\ref{lem:UNA}]
  Let \KBDef be a \SHIQ knowledge base and \Q a union of Boolean 
  conjunctive queries. $\KB \not\models \Q$ without making the unique 
  name assumption iff there is an \AB-partition $\KB^\Pmc = (\TB, \RB, 
  \AB^\Pmc)$ and $\Q^\Pmc$ w.r.t.\ \KB and \Q such that $\KB^\Pmc 
  \not\models \Q^\Pmc$ under the unique name assumption. 
\end{lemmaapp}

\begin{proof}[Proof of Lemma~\ref{lem:UNA}]
  For the ``only if''-direction: Since $\KB \not\models \Q$, there is a 
  model \I of \KB such that $\I \not\models \Q$. Let $f \colon \indA \to 
  \indA$ be a total function such that, for each set of individual names 
  $\{a_1, \ldots, a_n\}$ for which $\Int{a_1} = \Int{a_i}$ for $1 
  \leq i \leq n$, $f(a_i) = a_1$. Let $\AB^\Pmc$ and $\Q^\Pmc$ 
  be obtained from \AB and \Q by replacing each individual name $a$ in 
  \AB and \Q with $f(a)$. Clearly, $\KB^\Pmc = (\TB, \RB, \AB^\Pmc)$ 
  and $\Q^\Pmc$ are an \AB-partition w.r.t.\ \KB and \Q. Let $\I^\Pmc 
  = (\dom, \cdot^{\I^\Pmc})$ be an interpretation that is obtained by 
  restricting $\cdot^\I$ to individual names in $\indA[\AB^\Pmc]$. It 
  is easy to see that $\I^\Pmc \models \KB^\Pmc$ and that the unique 
  name assumption holds in $\I^\Pmc$. We now show that $\I^\Pmc 
  \not\models \Q^\Pmc$. Assume, to the contrary of what is to be 
  shown, that $\I^\Pmc \models^{\pi'} \Q^\Pmc$ for some match $\pi'$. 
  We define a mapping $\pi \colon \terms{\Q} \to \dom$ from $\pi'$ 
  such $\pi(a) = \pi'(f(a))$ for each individual name $a \in 
  \inds{\Q}$ and $\pi(v) = \pi'(v)$ for each variable $v \in 
  \vars{\Q}$. It is easy to see that $\I \models^\pi \Q$, which is a 
  contradiction.
  
  For the ``if''-direction: 
  Let $\I^\Pmc = (\dom, \cdot^{\I^\Pmc})$ be such that $\I^\Pmc \models 
  \KB^\Pmc$ under UNA and $\I^\Pmc \not\models \Q^\Pmc$ and let $f 
  \colon \indA \to \indA[\AB^\Pmc]$ be a total function such that $f(a)$ 
  is the individual that replaced $a$ in $\AB^\Pmc$ and $\Q^\Pmc$. Let 
  \I = \inter be an interpretation that extends $\I^\Pmc$ such that 
  $\Int{a} = \Int[\I^\Pmc]{f(a)}$. We show that $\I \models \KB$ and 
  that $\I \not\models \Q$. It is clear that $\I \models \TB$. Let 
  $C(a)$ be an assertion in \AB such that $a$ was replaced with 
  $a^\Pmc$ in $\AB^\Pmc$. Since $\I^\Pmc \models C(a^\Pmc)$ and 
  $\Int{a} = \Int[\I^\Pmc]{f(a)} = \Int[\I^\Pmc]{a^\Pmc} \in 
  \Int[\I^\Pmc]{C}$, $\I \models C(a)$. We can use a similar argument 
  for (possibly negated) role assertions. Let $a \ndoteq b$ be an 
  assertion in $\AB$ such that $a$ was replaced with $a^\Pmc$ and $b$ 
  with $b^\Pmc$ in $\AB^\Pmc$, i.e., $f(a) = a^\Pmc$ and $f(b) = 
  b^\Pmc$. Since $\I^\Pmc \models a^\Pmc \ndoteq b^\Pmc$, $\Int{a} = 
  \Int[\I^\Pmc]{f(a)} = \Int[\I^\Pmc]{a^\Pmc} \neq \Int[\I^\Pmc]{b^\Pmc} 
  = \Int[\I^\Pmc]{f(b)} = \Int{b}$ and $\I \models a \ndoteq b$ as 
  required. Therefore, we have that $\I \models \KB$ as required. 
  
  Assume that $\I \models^\pi \Q$ for a match $\pi$. Let $\pi^\Pmc 
  \colon \terms{\Q^\Pmc} \to \dom$ be a mapping such that $\pi^\Pmc(v) = 
  \pi(v)$ for $v \in \vars{\Q^\Pmc}$ and $\pi^\Pmc(a^\Pmc) = \pi(a)$ for 
  $a^\Pmc \in \inds{\Q^\Pmc}$ and some $a$ such that $a^\Pmc = f(a)$. 
  Let $C(a^\Pmc) \in \Q^\Pmc$ be such that $C(a) \in \Q$ and $a$ was 
  replaced with $a^\Pmc$, i.e., $f(a) = a^\Pmc$. By assumption, $\pi(a) 
  \in \Int{C}$, but then $\pi(a) = \Int{a} = \Int[\I^\Pmc]{f(a)} = 
  \Int[\I^\Pmc]{a^\Pmc} = \pi^\Pmc(a^\Pmc) \in \Int[\I^\Pmc]{C}$ and 
  $\I^\Pmc \models C(a^\Pmc)$. Similar arguments can be used to show 
  entailment for role and equality atoms, which yields the desired 
  contradiction. 
\end{proof}

\begin{theoremapp}[\ref{thm:datacomplexity}]
  Let \KBDef be a \SHIQ knowledge base with $m := |\KB|$ and $\Q := \Q_1 
  \vee \ldots \vee \Q_\ell$ a union of Boolean conjunctive queries with 
  $n := |\Q|$. The algorithm given in 
  Definition~\ref{def:algonondeterministic} decides in non-deterministic 
  time $p(m_a)$ whether $\KB \not\models \Q$ for $m_a := |\AB|$ and $p$ 
  a polynomial. 
\end{theoremapp}

\begin{proof}[Proof of Theorem~\ref{thm:datacomplexity}]
  Clearly, the size of an ABox $\AB^\Pmc$ in an \AB-partition is bounded 
  by $m_a$. As established in Lemma~\ref{thm:combinedcomplexity}, the 
  maximal size of an extended ABox $\AB^\Pmc_\Q$ is polynomial in $m_a$. 
  Hence, $|\AB^\Pmc \cup \AB^\Pmc_\Q| \leq p(m_a)$ for some polynomial 
  $p$. Due to Lemma~\ref{lem:counting} and since the size of \Q, \TB, 
  and \RB is fixed by assumption, the sets $\trees[\KB^\Pmc]{\Q_i}$ and 
  $\groundings[\KB^\Pmc]{\Q_i}$ for each $i$ such that $1 \leq i \leq 
  \ell$ can be computed in time polynomial in $m_a$. From 
  Lemma~\ref{lem:combinedcomplexity}, we know that the translation of an 
  extended knowledge base into an \ALCQIb knowledge base is polynomial 
  in $m_a$ and a close inspection of the algorithm by \citeA{Tobi01a} 
  for deciding consistency of an \ALCQIb knowledge base shows that 
  its runtime is also polynomial in $m_a$.
\end{proof}

\bibliographystyle{theapa}
\bibliography{long-string,bglimm}

\begin{thebibliography}{}

\bibitem[\protect\BCAY{Baader, Calvanese, McGuinness, Nardi,\ \BBA\
  Patel-Schneider}{Baader et~al.}{2003}]{dlhb}
Baader, F., Calvanese, D., McGuinness, D.~L., Nardi, D., \BBA\ Patel-Schneider,
  P.~F.\BEDS. \BBOP2003\BBCP.
\newblock {\Bem The Description Logic Handbook}.
\newblock Cambridge University Press.

\bibitem[\protect\BCAY{Bechhofer, van Harmelen, Hendler, Horrocks, McGuinness,
  Patel-Schneider,\ \BBA\ Stein}{Bechhofer et~al.}{2004}]{BHHH04a}
Bechhofer, S., van Harmelen, F., Hendler, J., Horrocks, I., McGuinness, D.~L.,
  Patel-Schneider, P.~F., \BBA\ Stein, L.~A. \BBOP2004\BBCP.
\newblock \BBOQ {OWL} web ontology language reference\BBCQ\
\newblock \BTR, World Wide Web Consortium.
\newblock {\footnotesize
  \url{http://www.w3.org/TR/2004/REC-owl-ref-20040210/}}.

\bibitem[\protect\BCAY{Calvanese, De~Giacomo, Lembo, Lenzerini,\ \BBA\
  Rosati}{Calvanese et~al.}{2006}]{CGLL06a}
Calvanese, D., De~Giacomo, G., Lembo, D., Lenzerini, M., \BBA\ Rosati, R.
  \BBOP2006\BBCP.
\newblock \BBOQ Data complexity of query answering in description logics\BBCQ\
\newblock In Doherty, P., Mylopoulos, J., \BBA\ Welty, C.~A.\BEDS, {\Bem
  Proceedings of the 10th International Conference on Principles of Knowledge
  Representation and Reasoning ({KR} 2006)}, \BPGS\ 260--270. {AAAI} Press/The
  {MIT} Press.

\bibitem[\protect\BCAY{Calvanese, De~Giacomo, Lembo, Lenzerini,\ \BBA\
  Rosati}{Calvanese et~al.}{2007}]{CDLL07a}
Calvanese, D., De~Giacomo, G., Lembo, D., Lenzerini, M., \BBA\ Rosati, R.
  \BBOP2007\BBCP.
\newblock \BBOQ Tractable reasoning and efficient query answering in
  description logics: The dl-lite family\BBCQ\
\newblock {\Bem Journal of Automated Reasoning}, {\Bem 39\/}(3), 385--429.

\bibitem[\protect\BCAY{Calvanese, De~Giacomo,\ \BBA\ Lenzerini}{Calvanese
  et~al.}{1998a}]{CaDL98a}
Calvanese, D., De~Giacomo, G., \BBA\ Lenzerini, M. \BBOP1998a\BBCP.
\newblock \BBOQ On the decidability of query containment under
  constraints\BBCQ\
\newblock In {\Bem Proceedings of the 17th ACM SIGACT-SIGMOD-SIGART Symposium
  on Principles of Database Systems ({PODS} 1998)}, \BPGS\ 149--158. {ACM}
  Press and Addison Wesley.

\bibitem[\protect\BCAY{Calvanese, De~Giacomo, Lenzerini, Nardi,\ \BBA\
  Rosati}{Calvanese et~al.}{1998b}]{CDLN98a}
Calvanese, D., De~Giacomo, G., Lenzerini, M., Nardi, D., \BBA\ Rosati, R.
  \BBOP1998b\BBCP.
\newblock \BBOQ Description logic framework for information integration\BBCQ\
\newblock In {\Bem Proceedings of the 6th International Conference on
  Principles of Knowledge Representation and Reasoning ({KR} 1998)}.

\bibitem[\protect\BCAY{Calvanese, Eiter,\ \BBA\ Ortiz}{Calvanese
  et~al.}{2007}]{CaEO07a}
Calvanese, D., Eiter, T., \BBA\ Ortiz, M. \BBOP2007\BBCP.
\newblock \BBOQ Answering regular path queries in expressive description
  logics: An automata-theoretic approach\BBCQ\
\newblock In {\Bem Proceedings of the 22th National Conference on Artificial
  Intelligence ({AAAI} 2007)}.

\bibitem[\protect\BCAY{Chekuri\ \BBA\ Rajaraman}{Chekuri\ \BBA\
  Rajaraman}{1997}]{CaRa97a}
Chekuri, C.\BBACOMMA\  \BBA\ Rajaraman, A. \BBOP1997\BBCP.
\newblock \BBOQ Conjunctive query containment revisited\BBCQ\
\newblock In {\Bem Proceedings of the 6th International Conference on Database
  Theory ({ICDT} 1997)}, \BPGS\ 56--70, London, UK. Springer-Verlag.

\bibitem[\protect\BCAY{Glimm, Horrocks,\ \BBA\ Sattler}{Glimm
  et~al.}{2006}]{GlHS06a}
Glimm, B., Horrocks, I., \BBA\ Sattler, U. \BBOP2006\BBCP.
\newblock \BBOQ Conjunctive query answering for description logics with
  transitive roles\BBCQ\
\newblock In {\Bem Proceedings of the 19th International Workshop on
  Description Logics ({DL} 2006)}.
\newblock {\footnotesize
  \url{http://www.cs.man.ac.uk/~glimmbx/download/GlHS06a.pdf}}.

\bibitem[\protect\BCAY{Gr{\"a}del}{Gr{\"a}del}{2001}]{Grad01a}
Gr{\"a}del, E. \BBOP2001\BBCP.
\newblock \BBOQ Why are modal logics so robustly decidable?\BBCQ\
\newblock In Paun, G., Rozenberg, G., \BBA\ Salomaa, A.\BEDS, {\Bem Current
  Trends in Theoretical Computer Science, Entering the 21th Century},
  \lowercase{\BVOL}~2, \BPGS\ 393--408. World Scientific.

\bibitem[\protect\BCAY{Grahne}{Grahne}{1991}]{Grah91a}
Grahne, G. \BBOP1991\BBCP.
\newblock {\Bem Problem of Incomplete Information in Relational Databases}.
\newblock Springer-Verlag.

\bibitem[\protect\BCAY{Horrocks, Patel-Schneider,\ \BBA\ van Harmelen}{Horrocks
  et~al.}{2003}]{HoPH03a}
Horrocks, I., Patel-Schneider, P.~F., \BBA\ van Harmelen, F. \BBOP2003\BBCP.
\newblock \BBOQ From {SHIQ} and {RDF} to {OWL}: The making of a web ontology
  language\BBCQ\
\newblock {\Bem Journal of Web Semantics}, {\Bem 1\/}(1), 7--26.

\bibitem[\protect\BCAY{Horrocks, Sattler, Tessaris,\ \BBA\ Tobies}{Horrocks
  et~al.}{1999}]{HSTT99a}
Horrocks, I., Sattler, U., Tessaris, S., \BBA\ Tobies, S. \BBOP1999\BBCP.
\newblock \BBOQ Query containment using a {DLR} {ABox}\BBCQ\
\newblock Ltcs-report\ LTCS-99-15, LuFG Theoretical Computer Science, RWTH
  Aachen, Germany.
\newblock Available online at {\footnotesize
  \url{http://www-lti.informatik.rwth-aachen.de/Forschung/Reports.html}}.

\bibitem[\protect\BCAY{Horrocks, Sattler,\ \BBA\ Tobies}{Horrocks
  et~al.}{2000}]{HoST00a}
Horrocks, I., Sattler, U., \BBA\ Tobies, S. \BBOP2000\BBCP.
\newblock \BBOQ {Reasoning with Individuals for the Description Logic
  $\mathcal{SHIQ}$}\BBCQ\
\newblock In McAllester, D.\BED, {\Bem Proceedings of the 17th International
  Conference on Automated Deduction ({CADE} 2000)}, \lowercase{\BNUM}\ 1831 in
  Lecture Notes in Artificial Intelligence, \BPGS\ 482--496. Springer-Verlag.

\bibitem[\protect\BCAY{Horrocks\ \BBA\ Tessaris}{Horrocks\ \BBA\
  Tessaris}{2000}]{HoTe00a}
Horrocks, I.\BBACOMMA\  \BBA\ Tessaris, S. \BBOP2000\BBCP.
\newblock \BBOQ A conjunctive query language for description logic aboxes\BBCQ\
\newblock In {\Bem Proceedings of the 17th National Conference on Artificial
  Intelligence ({AAAI} 2000)}, \BPGS\ 399--404.

\bibitem[\protect\BCAY{Hustadt, Motik,\ \BBA\ Sattler}{Hustadt
  et~al.}{2005}]{HuMS05a}
Hustadt, U., Motik, B., \BBA\ Sattler, U. \BBOP2005\BBCP.
\newblock \BBOQ Data complexity of reasoning in very expressive description
  logics\BBCQ\
\newblock In {\Bem Proceedings of the International Joint Conference on
  Artificial Intelligence ({IJCAI} 2005)}, \BPGS\ 466--471.

\bibitem[\protect\BCAY{Levy\ \BBA\ Rousset}{Levy\ \BBA\
  Rousset}{1998}]{LeRo98a}
Levy, A.~Y.\BBACOMMA\  \BBA\ Rousset, M.-C. \BBOP1998\BBCP.
\newblock \BBOQ Combining horn rules and description logics in {CARIN}\BBCQ\
\newblock {\Bem Artificial Intelligence}, {\Bem 104\/}(1--2), 165--209.

\bibitem[\protect\BCAY{Lutz}{Lutz}{2007}]{Lutz07a}
Lutz, C. \BBOP2007\BBCP.
\newblock \BBOQ Inverse roles make conjunctive queries hard\BBCQ\
\newblock In {\Bem Proceedings of the 20th International Workshop on
  Description Logics ({DL} 2007)}.

\bibitem[\protect\BCAY{McGuinness\ \BBA\ Wright}{McGuinness\ \BBA\
  Wright}{1998}]{McWr98a}
McGuinness, D.~L.\BBACOMMA\  \BBA\ Wright, J.~R. \BBOP1998\BBCP.
\newblock \BBOQ An industrial strength description logic-based configuration
  platform\BBCQ\
\newblock {\Bem IEEE Intelligent Systems}, {\Bem 13\/}(4).

\bibitem[\protect\BCAY{Motik, Sattler,\ \BBA\ Studer}{Motik
  et~al.}{2004}]{MoSS04a}
Motik, B., Sattler, U., \BBA\ Studer, R. \BBOP2004\BBCP.
\newblock \BBOQ Query answering for {OWL-DL} with rules\BBCQ\
\newblock In {\Bem Proceedings of the 3rd International Semantic Web Conference
  ({ISWC} 2004)}, Hiroshima, Japan.

\bibitem[\protect\BCAY{Ortiz, Calvanese,\ \BBA\ Eiter}{Ortiz
  et~al.}{2006a}]{OrCE06b}
Ortiz, M., Calvanese, D., \BBA\ Eiter, T. \BBOP2006a\BBCP.
\newblock \BBOQ Data complexity of answering unions of conjunctive queries in
  $\mathcal{SHIQ}$\BBCQ\
\newblock In {\Bem Proceedings of the 19th International Workshop on
  Description Logics ({DL} 2006)}.

\bibitem[\protect\BCAY{Ortiz, Calvanese,\ \BBA\ Eiter}{Ortiz
  et~al.}{2006b}]{OrCE06a}
Ortiz, M.~M., Calvanese, D., \BBA\ Eiter, T. \BBOP2006b\BBCP.
\newblock \BBOQ Characterizing data complexity for conjunctive query answering
  in expressive description logics\BBCQ\
\newblock In {\Bem Proceedings of the 21th National Conference on Artificial
  Intelligence ({AAAI} 2006)}.

\bibitem[\protect\BCAY{Rosati}{Rosati}{2006a}]{Rosa06a}
Rosati, R. \BBOP2006a\BBCP.
\newblock \BBOQ {DL+log}: Tight integration of description logics and
  disjunctive datalog\BBCQ\
\newblock In {\Bem Proceedings of the Tenth International Conference on
  Principles of Knowledge Representation and Reasoning ({KR} 2006)}, \BPGS\
  68--78.

\bibitem[\protect\BCAY{Rosati}{Rosati}{2006b}]{Rosa06b}
Rosati, R. \BBOP2006b\BBCP.
\newblock \BBOQ On the ddecidability and finite controllability of query
  processing in databases with incomplete information\BBCQ\
\newblock In {\Bem Proceedings of the 25th ACM SIGACT SIGMOD Symposium on
  Principles of Database Systems (PODS-06)}, \BPGS\ 356--365. {ACM} Press and
  Addison Wesley.

\bibitem[\protect\BCAY{Rosati}{Rosati}{2007a}]{Rosa07b}
Rosati, R. \BBOP2007a\BBCP.
\newblock \BBOQ The limits of querying ontologies\BBCQ\
\newblock In {\Bem Proceedings of the Eleventh International Conference on
  Database Theory (ICDT~2007)}, \lowercase{\BVOL}\ 4353 of {\Bem Lecture Notes
  in Computer Science}, \BPGS\ 164--178. Springer-Verlag.

\bibitem[\protect\BCAY{Rosati}{Rosati}{2007b}]{Rosa07a}
Rosati, R. \BBOP2007b\BBCP.
\newblock \BBOQ On conjunctive query answering in {EL}\BBCQ\
\newblock In {\Bem Proceedings of the 2007 Description Logic Workshop (DL
  2007)}. {CEUR} Workshop Proceedings.

\bibitem[\protect\BCAY{Schaerf}{Schaerf}{1993}]{Scha93a}
Schaerf, A. \BBOP1993\BBCP.
\newblock \BBOQ On the complexity of the instance checking problem in concept
  languages with existential quantification\BBCQ\
\newblock {\Bem Journal of Intelligent Information Systems}, {\Bem 2\/}(3),
  265--278.

\bibitem[\protect\BCAY{Sirin\ \BBA\ Parsia}{Sirin\ \BBA\
  Parsia}{2006}]{SiPa06a}
Sirin, E.\BBACOMMA\  \BBA\ Parsia, B. \BBOP2006\BBCP.
\newblock \BBOQ Optimizations for answering conjunctive abox queries\BBCQ\
\newblock In {\Bem Proceedings of the 19th International Workshop on
  Description Logics ({DL} 2006)}.

\bibitem[\protect\BCAY{Sirin, Parsia, Cuenca~Grau, Kalyanpur,\ \BBA\
  Katz}{Sirin et~al.}{2006}]{SPGK06a}
Sirin, E., Parsia, B., Cuenca~Grau, B., Kalyanpur, A., \BBA\ Katz, Y.
  \BBOP2006\BBCP.
\newblock \BBOQ Pellet: A practical {OWL-DL} reasoner\BBCQ\
\newblock Accepted for the Journal of Web Semantics, Available online at
  {\footnotesize \url{http://www.mindswap.org/papers/PelletJWS.pdf}}.

\bibitem[\protect\BCAY{Tessaris}{Tessaris}{2001}]{Tess01a}
Tessaris, S. \BBOP2001\BBCP.
\newblock {\Bem Questions and answers: reasoning and querying in Description
  Logic}.
\newblock {PhD} thesis, University of Manchester.

\bibitem[\protect\BCAY{Tobies}{Tobies}{2001}]{Tobi01a}
Tobies, S. \BBOP2001\BBCP.
\newblock {\Bem Complexity Results and Practical Algorithms for Logics in
  Knowledge Representation}.
\newblock {PhD} thesis, RWTH Aachen.

\bibitem[\protect\BCAY{Tsarkov\ \BBA\ Horrocks}{Tsarkov\ \BBA\
  Horrocks}{2006}]{TsHo06a}
Tsarkov, D.\BBACOMMA\  \BBA\ Horrocks, I. \BBOP2006\BBCP.
\newblock \BBOQ {FaCT}++ description logic reasoner: System description\BBCQ\
\newblock In Furbach, U.\BBACOMMA\  \BBA\ Shankar, N.\BEDS, {\Bem Proceedings
  of the Third International Joint Conference on Automated Reasoning ({IJCAR}
  2006)}, \lowercase{\BVOL}\ 4130 of {\Bem Lecture Notes in Computer Science},
  \BPGS\ 292 -- 297. Springer-Verlag.

\bibitem[\protect\BCAY{van~der Meyden}{van~der Meyden}{1998}]{Meyd98a}
van~der Meyden, R. \BBOP1998\BBCP.
\newblock \BBOQ Logical approaches to incomplete information: A survey\BBCQ\
\newblock In {\Bem Logics for Databases and Information Systems}, \BPGS\
  307--356. Kluwer Academic Publishers.

\bibitem[\protect\BCAY{Vardi}{Vardi}{1997}]{Vard97a}
Vardi, M.~Y. \BBOP1997\BBCP.
\newblock \BBOQ Why is modal logic so robustly decidable?\BBCQ\
\newblock In {\Bem Descriptive Complexity and Finite Models: Proceedings of a
  DIMACS Workshop}, \lowercase{\BVOL}~31 of {\Bem DIMACS: Series in Discrete
  Mathematics and Theoretical Computer Science}, \BPGS\ 149--184. American
  Mathematical Society.

\bibitem[\protect\BCAY{Wessel\ \BBA\ M{\"o}ller}{Wessel\ \BBA\
  M{\"o}ller}{2005}]{WeMo05a}
Wessel, M.\BBACOMMA\  \BBA\ M{\"o}ller, R. \BBOP2005\BBCP.
\newblock \BBOQ A high performance semantic web query answering engine\BBCQ\
\newblock In {\Bem Proceedings of the 18th International Workshop on
  Description Logics}.

\bibitem[\protect\BCAY{Wolstencroft, Brass, Horrocks, Lord, Sattler, Turi,\
  \BBA\ Stevens}{Wolstencroft et~al.}{2005}]{WBHL05a}
Wolstencroft, K., Brass, A., Horrocks, I., Lord, P., Sattler, U., Turi, D.,
  \BBA\ Stevens, R. \BBOP2005\BBCP.
\newblock \BBOQ {A Little Semantic Web Goes a Long Way in Biology}\BBCQ\
\newblock In {\Bem Proceedings of the 2005 International Semantic Web
  Conference ({ISWC} 2005)}.

\end{thebibliography}

\end{document}